\documentclass{article}

\usepackage[final]{neurips_2024}

\usepackage{graphicx}

\usepackage[utf8]{inputenc} %
\usepackage[T1]{fontenc}    %

\usepackage{float}
\usepackage{wrapfig}

\usepackage[dvipsnames,table]{xcolor}

\usepackage[colorlinks,linktoc=all]{hyperref}
\usepackage[all]{hypcap}
\hypersetup{citecolor=Mahogany}
\hypersetup{linkcolor=RoyalBlue}
\hypersetup{urlcolor =RoyalBlue}

\usepackage{url}            %
\usepackage{booktabs}       %
\usepackage[detect-all]{siunitx}
\usepackage{amsfonts}       %
\usepackage{nicefrac}       %
\usepackage{microtype}      %
\usepackage[ruled,linesnumbered,noend]{algorithm2e}

\SetCommentSty{mycommfont}

\usepackage{amssymb,amstext,amsmath,amsthm,amsfonts,pifont,array}

\usepackage{nicematrix}

\usepackage[capitalize,nameinlink]{cleveref}
\crefname{section}{\S}{\S\S}
\Crefname{section}{\S}{\S\S}
\creflabelformat{equation}{#2\textup{#1}#3}

\definecolor{mylightgray}{gray}{0.95}
\definecolor{mylightgray2}{gray}{0.85}

\usepackage{pifont}
\newcommand{\cmark}{\textcolor{green!60!black}{\ding{51}}\xspace}
\newcommand{\xmark}{\textcolor{red!60!black}{\ding{55}}\xspace}

\usepackage{enumitem}

\usepackage[acronym,nowarn]{glossaries}
\setacronymstyle{long-sc-short}
\glsdisablehyper
\makeglossaries
\usepackage{multirow}
\usepackage{tabularx}
\usepackage{mleftright}
\usepackage{xr-hyper}
\usepackage{adjustbox}
\usepackage{footmisc}

\usepackage{listings}

\definecolor{codegreen}{rgb}{0,0.6,0}
\definecolor{codegray}{rgb}{0.5,0.5,0.5}
\definecolor{codepurple}{rgb}{0.58,0,0.82}
\definecolor{backcolour}{rgb}{0.95,0.95,0.92}
\lstdefinestyle{mystyle}{
    backgroundcolor=\color{white},
    numberstyle=\tiny\color{codegray},
    stringstyle=\color{codepurple},
    basicstyle=\fontsize{7.5pt}{7.5pt}\ttfamily\selectfont,
    commentstyle=\fontsize{7.5pt}{7.5pt}\color{codegreen},
    keywordstyle=\fontsize{7.5pt}{7.5pt}\color{magenta},
    breakatwhitespace=false,
    breaklines=true,
    captionpos=b,
    keepspaces=true,
    numbers=left,
    numbersep=5pt,
    showspaces=false,
    showstringspaces=false,
    showtabs=false,
    tabsize=2
}
\usepackage{epstopdf,caption,tablefootnote}
\usepackage{tikz,tikzscale,subcaption}
 \usetikzlibrary{arrows.meta,calc,decorations.markings,math,arrows.meta,shapes,arrows}
 \tikzset{%
             base/.style = {rectangle, draw=black,
                            minimum width=4cm, minimum height=1cm,
                            text centered, font=\sffamily},
   binary/.style = {base, minimum width=1cm},
        startstop/.style = {base, fill=red!30, minimum width=2cm},
     activityRuns/.style = {base, fill=green!30},
          process/.style = {base, minimum width=2cm, fill=gray!15,
                            font=\ttfamily},
          sum/.style      = {draw, circle, node distance = 1.5cm},
 }
 \usetikzlibrary{positioning}
 \usetikzlibrary{shapes.misc}
 \usetikzlibrary{automata,arrows,positioning,calc}

\usepackage{pgfplots}
\pgfplotsset{compat=1.6}
\usepgfplotslibrary{groupplots}
\tikzstyle{every picture}+=[font=\sffamily]
\tikzstyle{optimized} = [circle,fill=white,draw=black, dashed,inner sep=1pt, minimum size=20pt, font=\fontsize{10}{10}\selectfont, node distance=1]
\pgfkeys{/pgf/number format/.cd,1000 sep={}}
\pgfplotsset{
	tick label style = {font=\sffamily},
	every axis label/.append style={font=\sffamily},
	typeset ticklabels with strut,
}
\pgfplotsset{every axis/.append style={
			every x tick label/.append style={font=\fontsize{6pt}{6pt}\sffamily, yshift=.5ex,},
			every y tick label/.append style={font=\fontsize{6pt}{6pt}\sffamily, xshift=.5ex},
			every y label/.append style={xshift=10ex, font=\sffamily},
			every x label/.append style={yshift=3ex, font=\sffamily},
			every title/.append style={font=\sffamily}
		},
}
\pgfplotsset{
  xticklabel={$\mathsf{\pgfmathprintnumber{\tick}}$},
  yticklabel={$\mathsf{\pgfmathprintnumber{\tick}}$},
}
\pgfplotsset{every axis title/.append style={yshift=-1ex}}
\newlength\figureheight
\newlength\figurewidth
\usepgfplotslibrary{external}
\usepackage{todonotes}

\definecolor{color_binarynet}{HTML}{bcbd22}
\definecolor{color_binaryconnect}{HTML}{17becf}
\definecolor{color_xnornet}{HTML}{2ca02c}
\definecolor{color_birealnet}{HTML}{e377c2}
\definecolor{color_real2binary}{HTML}{8c564b}
\definecolor{color_reactnet}{HTML}{9467bd}
\definecolor{color_meliusnet}{HTML}{7f7f7f}
\definecolor{color_bnext}{HTML}{ff7f0e}
\definecolor{color_pokebnn}{HTML}{1f77b4}
\definecolor{color_booldl}{HTML}{d62728}

\usepackage{tcolorbox}
\definecolor{lightblue}{HTML}{ff7f2a}
\definecolor{lighterblue}{HTML}{ffe6d5}%
\newtcolorbox{mybox}{colback=mylightgray,colframe=mylightgray,top=1.2pt,bottom=1.2pt,right=1.8pt,left=1.8pt}
\newtcolorbox{myboxcontrib}{colback=mylightgray2,colframe=mylightgray2,top=1.2pt,bottom=1.2pt,right=1.8pt,left=1.8pt}

\definecolor{asparagus}{rgb}{0.53, 0.66, 0.42}
\definecolor{b_df_c}{rgb}{0, 0.69, 0.31}
\definecolor{b_block_c}{rgb}{0.57, 0.81, 0.31}
\definecolor{fp_df_c}{rgb}{0.78, 0, 0.04}
\definecolor{fp_block_c}{rgb}{0.78, 0, 0.04}
\definecolor{int_df_c}{rgb}{0, 0.2, 0.8}

\usepackage{setspace}

\theoremstyle{plain}
\newtheorem{theorem}{Theorem}[section]
\newtheorem{proposition}[theorem]{Proposition}
\newtheorem{lemma}[theorem]{Lemma}

\newtheorem{definition}[theorem]{\textit{Definition}}

\newtheorem{assumption}{A.}

\theoremstyle{remark}
\newtheorem{remark}[theorem]{\textit{Remark}}%
\newtheorem{example}[theorem]{\textit{Example}}%
\newtheorem{notation}[theorem]{\textit{Notation}}%

\newcommand{\fref}[1]{Fig.~\ref{#1}}

\newcommand{\eref}[1]{(\ref{#1})}

\newcommand{\thmref}[1]{Theorem~\ref{#1}}

\newcommand{\True}{\mathrm{T}}
\newcommand{\False}{\mathrm{F}}
\newcommand{\xor}{\mathbf{xor}}
\newcommand{\xnor}{\mathbf{xnor}}

\newcommand{\maj}{\mathbf{maj}}

\newcommand{\Conv}{\mathrm{Conv}}
\newcommand{\full}{\mathrm{full}}
\newcommand{\logic}{\mathrm{logic}}
\newcommand{\MP}{\mathrm{MP}}

\def\Bb{\mathbb{B}}

\def\Db{\mathbb{D}}
\def\Ec{\mathcal{E}}
\def\Fc{\mathcal{F}}

\def\Lb{\mathbb{L}}
\def\Lm{\mathrm{L}}
\def\Loss{\mathit{Loss}}
\def\Mb{\mathbb{M}}
\def\Nb{\mathbb{N}}
\def\Nc{\mathcal{N}}

\def\Rb{\mathbb{R}}

\def\Zb{\mathbb{Z}}
\def\Loss{\mathrm{Loss}}
\def\xb{\mathbf{x}}

\def\wrt{w.r.t. }

\newcommand{\pren}[1]{\mleft(#1\mright)}

\usepackage{xargs}

\DeclareMathOperator*{\deq}{\overset{{\rm def}}{=}}

\newcommand{\bvar}[1]{\delta#1} %
\newcommand{\one}[1]{\mathbf{1}\mleft(#1\mright)}
\newcommand{\E}[1]{\mathbb{E}\mleft[#1\mright]}
\newcommand{\Et}[1]{\mathbb{E}_t\mleft[#1\mright]}

\def\sign{\operatorname{sign}}
\def\var{\operatorname{Var}}
\def\tanh{\operatorname{tanh}}
\def\proj{\operatorname{p}}
\def\emb{\operatorname{e}}

\makeatletter

\def\env@sqcases{%
	\let\@ifnextchar\new@ifnextchar
	\left\lbrack
	\def\arraystretch{1.2}%
	\array{@{}l@{\quad}l@{}}%
}
\makeatother

\newcommand \good[1]{\textcolor{blue}{#1$^\dag$}}
\newcommand \bad[1]{\textcolor{orange}{#1$^*$}}

\newacronym{KD}{kd}{knowledge distillation}
\newacronym{BNN}{bnn}{binarized neural network}
\newacronym{STE}{ste}{straight-through-estimator}
\newacronym{SE}{se}{squeeze-and-excitation}
\newacronym{CNN}{cnn}{convolutional neural network}
\newacronym{FP}{fp}{full-precision}
\newacronym{SOTA}{sota}{state-of-the-art}
\newacronym{LLM}{llm}{large language model}
\newacronym{OP}{op}{compute operation}
\newacronym{NN}{nn}{neural network}
\newacronym{BN}{bn}{batch normalization}
\newacronym{FC}{fc}{fully-connected}
\newacronym{ASPP}{aspp}{atrous pyramid pooling}
\newacronym{LGN}{lgn}{logic gate network}

\newcommand{\ours}{\textsc{b}{\fontsize{8.5pt}{8.5pt}\selectfont\raisebox{\dimexpr1.43ex-\height}{$\oplus$}}\textsc{ld}\xspace}
\newcommand{\oursfig}{\textsc{b}{\fontsize{6.5pt}{6.5pt}\selectfont\raisebox{\dimexpr1.25ex-\height}{$\oplus$}}\textsc{ld}\xspace}
\newcommand{\ie}{i.e.}

\title{BOLD: Boolean Logic Deep Learning}
\author{
	Van Minh Nguyen\thanks{Van Minh developed the mathematical principle, designed and implemented the Boolean deep learning framework first in C++ and later in PyTorch, validated the concept on the MNIST and CIFAR-10 benchmarks, and led all aspects of the project. Cristian evaluated the design using computer vision benchmarks. Aymen analyzed and evaluated the computational complexity. Louis proposed the incorporation of plasticity effects in the Boolean optimizer and developed the convergence analysis. Ba Hien contributed to the design and evaluation of Boolean attention mechanisms and enhanced the quality of the paper's presentation.} \And Cristian Ocampo \And Aymen Askri \And Louis Leconte \And Ba-Hien Tran \\
	\AND \textnormal{Mathematical and Algorithmic Sciences Laboratory},\\ Huawei Paris Research Center, France \\ \texttt{vanminh.nguyen@huawei.com}
}

\begin{document}
\addtocontents{toc}{\protect\setcounter{tocdepth}{0}}

\maketitle

\begin{abstract}
  Computational intensiveness of deep learning has motivated low-precision arithmetic designs. %
However, the current quantized/binarized training approaches are limited by: %
(1) significant performance loss due to arbitrary approximations of the latent weight gradient through its discretization/binarization function, and (2) training computational intensiveness due to the reliance on full-precision latent weights. 
This paper proposes a novel mathematical principle by introducing the notion of Boolean variation such that neurons made of Boolean weights and/or activations can be trained ---for the first time--- natively in Boolean domain instead of latent-weight gradient descent and real arithmetic.
We explore its convergence, conduct extensively experimental benchmarking, and provide consistent complexity evaluation by considering chip architecture, memory hierarchy, dataflow, and arithmetic precision.
Our approach achieves baseline full-precision accuracy in ImageNet classification and surpasses state-of-the-art results in semantic segmentation, with notable performance in image super-resolution, and natural language understanding with transformer-based models.
Moreover, it significantly reduces energy consumption during both training and inference.

\end{abstract}

\section{Introduction}\label{sec:Intro}

\vspace{-2ex}

Deep learning \citep{lecun2015deep} has become the \emph{de facto solution} to a wide range of tasks.
However, running deep learning models for \emph{inference} demands significant computational resources, yet it is just the tip of the iceberg.
\emph{Training} deep models is even much more intense. 
\if False  %
	since it involves iterative processes of storing multiple temporal variables, buffers for gradient computation and parameter optimization.
	This intensive process is then repeated for hyperparameter tuning, lasting weeks or months on specialized equipment, resulting in a significant carbon footprint and computational resource demand \citep{Strubell2019}, especially critical given the current ecological emergency and the trend of \glspl{LLM} \citep{Brown2020, touvron2023llama,openai2023gpt,team2024gemma}.
	Additionally, accuracy suffers when pre-trained models used for edge inference fail to adapt to evolving data.
	Thus, addressing the computational intensity of deep model training is crucial technologically to bring it closer to data generation points, alleviating privacy concerns and enabling new online/on-device learning capabilities.
\fi
The extensive literature on this issue can be summarized into four approaches addressing different sources of complexity. 
These include: \textit{(1)} model compression, pruning \citep{Han2015, Cheng2020, Yang2017b} and network design \citep{Sze2017,Howard2017,Tan2019} for large model dimensions;
\textit{(2)} arithmetic approximation \citep{Lav2020,Sze2017,Chen2020a} for intensive multiplication;
\textit{(3)} quantization techniques like post-training \citep{Xiao2023,Gholami2022}, quantization-aware training \citep{Gupta2015, Zhang2018, Jin2021}, and quantized training to reduce precision \citep{Chen2020,Sun2020};
and \textit{(4)} hardware design \citep{Chen2016,Sebastian2020,Williams2009,Verma2019,Grimm2022,Yu2016} to overcome %
computing bottleneck by moving computation closer to or in memory.

Aside from hardware and dataflow design, deep learning designs have primarily focused on the number of \glspl{OP}, such as \textsc{flops} or \textsc{bops}, as a complexity measure \citep{GarciaMartin2019,Qin2020} rather than the consumed energy or memory, and particularly in inference tasks.
However, it has been demonstrated that \glspl{OP} alone are inadequate and even detrimental as a measure of system complexity.
Instead, energy consumption provides a more consistent and efficient metric for computing hardware \citep{Sze2017,Sze2020a,Yang2017a,Strubell2019}.
Data movement, especially, dominates energy consumption and is closely linked to system architecture, memory hierarchy, and dataflow \citep{kwon2019understanding, sim2019energy,yang2020interstellar,Chen2016a}.
Therefore, efforts aimed solely at reducing \glspl{OP} are inefficient.

\glsreset{BNN}
\glsreset{FP}
Quantization-aware training, notably \glspl{BNN} \citep{Courbariaux2015,Courbariaux2016}, have garnered significant investigation \citep[see, e.g.][and references therein]{Qin2020,Gholami2022}.
\glspl{BNN} typically binarize weights and activations, forming principal computation blocks in binary.
They learn binary weights, $\mathbf{w}_{\texttt{bin}}$, through \emph{\gls{FP} latent weights}, $\mathbf{w}_{\texttt{fp}}$,
leading to no memory or computation savings during training.
For example, a binarized linear layer is operated as
$s = \alpha \cdot \mathbf{w}_{\texttt{bin}}^{\top} \mathbf{x}_{\texttt{bin}}$,
where $s$ is the output and $\alpha$ is a \gls{FP} scaling factor, $\mathbf{w}_{\texttt{bin}}=\operatorname{sign}(\mathbf{w}_{\texttt{fp}})$, and $\mathbf{x}_{\texttt{bin}}=\operatorname{sign}(\mathbf{x}_{\texttt{fp}})$
is the binarized inputs.  
The weights are updated via common gradient descent backpropagation,
i.e. $\mathbf{w}_{\texttt{bin}}=\operatorname{sign}(\mathbf{w}_{\texttt{fp}} - \eta \cdot \mathbf{g}_{\mathbf{w}_{\texttt{fp}}})$ with a learning rate $\eta$, and \gls{FP} gradient signal $\mathbf{g}_{\mathbf{w}_{\texttt{fp}}}$.
Gradient approximation of binarized variables often employs a differentiable proxy of the binarization function $\operatorname{sign}$, commonly the identity proxy.
Various approaches treat \gls{BNN} training as a constrained optimization problem \citep{bai2018proxquant, ajanthan2019proximal, ajanthan2021mirror, lin2020rotated}, exploring methods to derive binary weights from real-valued latent ones.
\glspl{BNN} commonly suffers notable accuracy drops due to reduced network capacity and the use of proxy \gls{FP} optimizers \citep{Li2017} instead of operating directly in binary domain  \citep{Qin2020, Nie2022, Guo2022}. 
Recent works mitigate this by incorporating multiple \gls{FP} components in the network, retaining only a few binary dataflows \citep{Liu2020}.
Thus, while binarization aids in reducing inference complexity, it increases network training complexity and memory usage.

In contrast to binarizing \gls{FP} models like \glspl{BNN}, designing native binary models not relying on \gls{FP} latent weight has been explored. 
For example, Expectation Backpropagation \citep{Soudry2014}, although operating on full-precision training, was proposed for this purpose.
Statistical physics-inspired \citep{Baldassi2009,Baldassi2015} and Belief Propagation \citep{Baldassi2015a} algorithms utilize integer latent weights, mainly applied to single perceptrons, with unclear applicability to deep models.
Evolutionary algorithms \citep{Morse2016,Ito2010} are also an alternative but encounter performance and scalability challenges.

\begin{mybox}
	\textbf{Summary:}
	No scalable and efficient algorithm currently exists for \emph{natively} training deep models in binary.
	The challenge of significantly reducing the training complexity while maintaining high performance of deep learning models remains open.
	
\end{mybox}

\begin{myboxcontrib}
\paragraph{Contributions.}
For the aforementioned challenge, we propose a novel framework --- \emph{\textbf{Bo}olean \textbf{L}ogic \textbf{D}eep Learning} (\ours) --- which relies on Boolean notions to define models and training:
\end{myboxcontrib}

{
\setlist[itemize]{leftmargin=5mm}
\begin{itemize}
	\item We introduce the notion of variation to the Boolean logic and develop a new mathematical framework of function variation (see \cref{sec:Foundation}). One of the noticeable properties is that Boolean variation has the chain rule (see \cref{thm:MainRules}) similar to the continuous gradient. 

	\vspace{-0.5ex}
		
	\item Based on the proposed framework, we develop a novel Boolean backpropagation and optimization method allowing for a deep model to support native Boolean components operated solely with Boolean logic and trained directly in Boolean domain, eliminating the need for gradient descent and \gls{FP} latent weights (see \cref{sec:Backprop}).
	This drastically cuts down memory footprint and energy consumption during \emph{both training and inference} (see, e.g., \cref{fig:cifar10_vgg}).

	\vspace{-0.5ex}

	\item We provide a theoretical analysis of the convergence of our training algorithm (see \cref{thm:Convegence}). %

	\vspace{-0.5ex}
	
	\item We conduct an extensive experimental campaign using modern network architectures such as \glspl{CNN} and Transformers \citep{Vaswani2017} 
	on a wide range of challenging tasks including image classification, segmentation, super-resolution and natural language understanding (see \cref{sec:experiments}).
	We rigorously evaluate analytically the complexity of \ours and \glspl{BNN}.
	We demonstrate the superior performance of our method  in terms of both accuracy and complexity compared to the state-of-the-art (see, e.g.,  \cref{tab:seg_cs_val}, \cref{table:imagenet_resntricks}, \cref{table:nlp}). %
\end{itemize}
}

\if False
The remaining of the paper is organized as follows.
Section \ref{sec:Sota} presents the related works and details the current problems of \glspl{BNN}.
Section \ref{sec:Method} introduces the proposed methodology.
The experiments and complexity are reported in Section \ref{sec:experiments}.
Final remarks are given in Section \ref{sec:Conclusions}.
\fi

\section{Are Current Binarized Neural Networks Really Efficient?}\label{sec:Sota}

\vspace{-2ex}

\glsreset{BNN}
Our work is closely related with the line of research on \glspl{BNN}.
The concept of \glspl{BNN} traces back to early efforts to reduce the complexity of deep learning models.
\textsc{binaryconnect} \citep{Courbariaux2015} is one of the pioneering works that introduced the idea of binarizing \gls{FP} weights during training, effectively reducing memory footprint and computational cost.
Similarly, \textsc{binarynet} \citep{Courbariaux2016}, \textsc{xnor-net} \citep{Rastegari2016} extended this approach to binarize both weights and activations, further enhancing the efficiency of neural network inference.
However, these early \glspl{BNN} struggled with maintaining accuracy comparable to their full-precision counterparts.
To address this issue, significant advances have been made on \glspl{BNN} \cite[see, e.g.,][and references therein]{Qin2020,qin23b}, which can be categorized into three main aspects.

\paragraph{\raisebox{.5pt}{\textcircled{\raisebox{-.9pt} {{\small{1}}}}} Binarization strategy.}
The binarization strategy aims to efficiently convert real-valued data such as activations and weights into binary form $\{-1, 1\}$.
The $\operatorname{sign}$ function is commonly used for this purpose, sometimes with additional constraints such as clipping bounds in \gls{SOTA} methods like \textsc{pokebnn} \citep{zhang2022a} or \textsc{bnext} \citep{Guo2022}. %
\textsc{reactnet} \citep{Liu2020} introduces $\operatorname{rsign}$ as a more general alternative to the $\operatorname{sign}$ function, addressing potential shifts in the distribution of activations and weights.
Another approach \citep{Tu2022} explores the use of other two real values instead of strict binary to enhance the representational capability of \glspl{BNN}.

\vspace{-0.5ex}

\paragraph{\raisebox{.5pt}{\textcircled{\raisebox{-.9pt} {{\small{2}}}}} Optimization and training strategy.}
\glsreset{STE}
\gls{BNN} optimization relies 
totally on latent-weight training,  
necessitating a differentiable proxy %
for the backpropagation. 
\if False
	A popular choice is \gls{STE} \cite{ste2012}, which enables training \glspl{BNN} using the same gradient descent methods as the ordinary full-precision models.
	However, this technique commonly leads to unstable training. 
	To mitigate this problem, various alternatives to \gls{STE}, such as piece-wise polynomials and hyper-parameterized $\operatorname{tanh}$, have been explored \citep{Nie2022}.
\fi
Moreover, latent-weight based training methods have to store binary and real parameters during training and often requires multiple sequential training stages, where one starts with training a \gls{FP} model and only later enables binarization \citep{Guo2022, zhang2022a, Xing2022}.
This further increases the training costs.
Furthermore, \gls{KD} has emerged as a method to narrow the performance gap by transferring knowledge from a full-precision teacher model to a binary model \citep{Liu2020, Xing2022, zhang2022a, Lee2022}.
While a single teacher model can sufficiently improve the accuracy of the student model, recent advancements like multi-\gls{KD} with multiple teachers, such as \textsc{bnext} \citep{Guo2022}, have achieved unprecedented performance.
However, the \gls{KD} approach often treats network binarization as an add-on to full-precision models.
In addition, this training approach relies on specialized teachers for specific tasks, limiting adaptability on new data.
Lastly, \cite{Wang2021d, Helwegen2019} proposed some heuristics and improvements to the classic \gls{BNN} latent-weight based optimizer.

\begin{wrapfigure}[17]{r}{0.28\textwidth}
        \vspace{-2ex}
        \tikzexternaldisable
        \centering
        \scriptsize
        \setlength{\figurewidth}{4.6cm}
        \setlength{\figureheight}{4.6cm}
        \begin{tikzpicture}

\definecolor{crimson2143940}{RGB}{214,39,40}
\definecolor{darkgray176}{RGB}{176,176,176}
\definecolor{darkturquoise23190207}{RGB}{23,190,207}
\definecolor{forestgreen4416044}{RGB}{44,160,44}
\definecolor{goldenrod18818934}{RGB}{188,189,34}
\definecolor{gray}{RGB}{128,128,128}

\begin{axis}[
height=\figureheight,
major tick length=1ex,
tick align=outside,
tick pos=left,
width=\figurewidth,
x grid style={darkgray176},
xlabel={\hspace{-6ex}Energy Consum. w.r.t. FP (\%) (\(\displaystyle \leftarrow\))},
xmin=-20, xmax=130,
xtick style={color=black},
y grid style={darkgray176},
ylabel={Accuracy (\%) (\(\displaystyle \rightarrow\))},
ymin=89.4, ymax=94.3,
ytick style={color=black},
yticklabel style={/pgf/number format/fixed, /pgf/number format/precision=3}
]
\addplot [semithick, gray, dashed]
table {%
-20 93.8
100 93.8
};
\addplot [semithick, gray, dashed]
table {%
100 89.4
100 93.8
};
\addplot [semithick, gray, dashed]
table {%
-20 92.37
3.71 92.37
};
\addplot [semithick, gray, dashed]
table {%
3.71 89.4
3.71 92.37
};
\addplot [semithick, gray, dashed]
table {%
-20 90.29
2.78 90.29
};
\addplot [semithick, gray, dashed]
table {%
2.78 89.4
2.78 90.29
};
\addplot [semithick, black, mark=*, mark size=2, mark options={solid}]
table {%
100 93.8
};
\addplot [semithick, darkturquoise23190207, mark=*, mark size=2, mark options={solid}]
table {%
48.49 90.1
};
\addplot [semithick, forestgreen4416044, mark=*, mark size=2, mark options={solid}]
table {%
45.68 89.83
};
\addplot [semithick, goldenrod18818934, mark=*, mark size=2, mark options={solid}]
table {%
43.61 89.85
};
\addplot [semithick, crimson2143940, mark=triangle*, mark size=2, mark options={solid}]
table {%
2.78 90.29
};
\addplot [semithick, crimson2143940, mark=*, mark size=2, mark options={solid}]
table {%
3.71 92.37
};
\draw (axis cs:100,93.8) node[
  scale=0.85,
  anchor=base west,
  text=black,
  rotate=0.0
]{\textsc{fp}};
\draw (axis cs:3.71,92.37) node[
  scale=0.85,
  anchor=base west,
  text=black,
  rotate=0.0
]{\oursfig \textsc{with batch-norm}};
\draw (axis cs:2.28,90.49) node[
  scale=0.85,
  anchor=base west,
  text=black,
  rotate=0.0
]{\oursfig \textsc{w/o batch-norm}};
\draw (axis cs:48.49,90.1) node[
  scale=0.85,
  anchor=base west,
  text=black,
  rotate=0.0
]{\textsc{binaryconnect}};
\draw (axis cs:45.68,89.83) node[
  scale=0.85,
  anchor=base west,
  text=black,
  rotate=0.0
]{\textsc{xnor-net}};
\draw (axis cs:43.61,89.55) node[
  scale=0.85,
  anchor=base west,
  text=black,
  rotate=0.0
]{\textsc{binarynet}};
\draw (axis cs:-20,93.85) node[
  scale=0.65,
  anchor=base west,
  text=black,
  rotate=0.0
]{93.80};
\draw (axis cs:-20,90.34) node[
  scale=0.65,
  anchor=base west,
  text=black,
  rotate=0.0
]{90.29};
\draw (axis cs:-20,92.42) node[
  scale=0.65,
  anchor=base west,
  text=black,
  rotate=0.0
]{92.37};
\draw (axis cs:2.78,89.45) node[
  scale=0.65,
  anchor=base west,
  text=black,
  rotate=0.0
]{2.78};
\end{axis}

\end{tikzpicture}
        \tikzexternalenable
        \vspace{-4ex}
        \definecolor{crimson2143940}{RGB}{214,39,40}
        \caption{\small{{\color{crimson2143940}Our method} against notable \glspl{BNN} on \textsc{cifar10} using  \textsc{vgg-small}. The energy is analytically evaluated considering a hypothetical V100 equivalence with native 1-bit support; cf \cref{sec:experiments} for details.}}
        \label{fig:cifar10_vgg}
\end{wrapfigure}

\begin{table}[t!]
    \centering
    \caption{A summary of \gls{SOTA} \glspl{BNN} compared to our method.
    The notation \xmark indicates the nonexistence of a specific requirement (column) within a given method (row).
    The colors denoting the methods shall be used consistently throughout the paper.
    }
    \resizebox{0.99\textwidth}{!}{
        \begin{tabular}{lccccccc}
            \toprule
            \multirow{2}{4em}{Method}                   & Bitwidth      & Specialized  & Mandatory     & Multi-stage or & Weight                 & \multirow{2}{4em}{Backprop} & Training   \\
                                                         & (Weight-Act.) & Architecture & FP Components & KD Training    & Updates                &                             & Arithmetic \\ [.5ex]
            \midrule
            \midrule
            \tikzexternaldisable ({\protect\tikz[baseline=-.65ex]\protect\draw[thick, color=color_binarynet, fill=color_binarynet, mark=*, mark size=2pt, line width=1.25pt] plot[] (-.0, 0)--(-0,0);}) \tikzexternalenable \textsc{binarynet} \citep{Courbariaux2016}   & 1-1           & \xmark       & \xmark        & \xmark         & FP latent-weights      & Gradient                    & FP         \\
            \tikzexternaldisable ({\protect\tikz[baseline=-.65ex]\protect\draw[thick, color=color_binaryconnect, fill=color_binaryconnect, mark=*, mark size=2pt, line width=1.25pt] plot[] (-.0, 0)--(-0,0);}) \tikzexternalenable \textsc{binaryconnect} \citep{Courbariaux2015}   & 1-32           & \xmark       & \xmark        & \xmark         & FP latent-weights      & Gradient                    & FP         \\
            \tikzexternaldisable ({\protect\tikz[baseline=-.65ex]\protect\draw[thick, color=color_xnornet, fill=color_xnornet, mark=*, mark size=2pt, line width=1.25pt] plot[] (-.0, 0)--(-0,0);}) \tikzexternalenable \textsc{xnor-net} \citep{Rastegari2016}      & 1-1           & \xmark       & \xmark        & \xmark         & FP latent-weights      & Gradient                    & FP         \\
            \tikzexternaldisable ({\protect\tikz[baseline=-.65ex]\protect\draw[thick, color=color_birealnet, fill=color_birealnet, mark=*, mark size=2pt, line width=1.25pt] plot[] (-.0, 0)--(-0,0);}) \tikzexternalenable\textsc{bi-realnet} \citep{Liu2018}          & 1-1           & \cmark       & \cmark        & \cmark         & FP latent-weights      & Gradient                    & FP         \\
            \tikzexternaldisable ({\protect\tikz[baseline=-.65ex]\protect\draw[thick, color=color_real2binary, fill=color_real2binary, mark=*, mark size=2pt, line width=1.25pt] plot[] (-.0, 0)--(-0,0);}) \tikzexternalenable \textsc{real2binary} \citep{Martinez2020}    & 1-1           & \cmark       & \cmark        & \cmark         & FP latent-weights      & Gradient                    & FP         \\
            \tikzexternaldisable ({\protect\tikz[baseline=-.65ex]\protect\draw[thick, color=color_reactnet, fill=color_reactnet, mark=*, mark size=2pt, line width=1.25pt] plot[] (-.0, 0)--(-0,0);}) \tikzexternalenable \textsc{reactnet} \cite{Liu2020}             & 1-1           & \cmark       & \cmark        & \cmark         & FP latent-weights      & Gradient                    & FP         \\
            \tikzexternaldisable ({\protect\tikz[baseline=-.65ex]\protect\draw[thick, color=color_meliusnet, fill=color_meliusnet, mark=*, mark size=2pt, line width=1.25pt] plot[] (-.0, 0)--(-0,0);}) \tikzexternalenable \textsc{melius-net} \citep{Bethge_2021_WACV} & 1-1           & \cmark       & \cmark        & \cmark         & FP latent-weights      & Gradient                    & FP         \\
            \tikzexternaldisable ({\protect\tikz[baseline=-.65ex]\protect\draw[thick, color=color_bnext, fill=color_bnext, mark=*, mark size=2pt, line width=1.25pt] plot[] (-.0, 0)--(-0,0);}) \tikzexternalenable \textsc{bnext} \cite{Guo2022}                & 1-1           & \cmark       & \cmark        & \cmark         & FP latent-weights      & Gradient                    & FP         \\
            \tikzexternaldisable ({\protect\tikz[baseline=-.65ex]\protect\draw[thick, color=color_pokebnn, fill=color_pokebnn, mark=*, mark size=2pt, line width=1.25pt] plot[] (-.0, 0)--(-0,0);}) \tikzexternalenable \textsc{pokebnn} \cite{zhang2022a}           & 1-1           & \cmark       & \cmark        & \cmark         & FP latent-weights      & Gradient                    & FP         \\
            \midrule
            \tikzexternaldisable ({\protect\tikz[baseline=-.65ex]\protect\draw[thick, color=color_booldl, fill=color_booldl, mark=*, mark size=2pt, line width=1.25pt] plot[] (-.0, 0)--(-0,0);}) \tikzexternalenable \ours [\textbf{Ours}]             & 1-1           & \xmark       & \xmark        & \xmark         & Native Boolean weights & Logic                       & Logic      \\ [.5ex]
            \bottomrule
        \end{tabular}
    }
    \label{tab:method_summary}
    \vspace{-3ex}
\end{table}

\vspace{-0.5ex}

\paragraph{\raisebox{.5pt}{\textcircled{\raisebox{-.9pt} {{\small{3}}}}} Architecture design.} Architecture design in \glspl{BNN} commonly utilizes \textsc{resnet} \citep{Liu2018, Liu2020, Bethge_2021_WACV, Guo2022} and \textsc{mobilenet} \citep{Liu2020, Guo2022} layouts.
These methodologies often rely on heavily modified basic blocks, including additional shortcuts, automatic channel scaling via \gls{SE} \citep{zhang2022a}, block duplication with concatenation in the channel domain \citep{Liu2020, Guo2022}.
Recent approaches introduce initial modules to enhance input adaptation for binary dataflow \citep{Xing2022} or replace convolutions with lighter point-wise convolutions \citep{Liu2022c}.

Another related line of research involves \glspl{LGN} \citep{Petersen2022}. In \glspl{LGN}, each neuron functions as a binary logic gate and, consequently, has only two inputs. Unlike traditional neural networks, \glspl{LGN} do not utilize weights; instead, they are parameterized by selecting a specific logic gate for each neuron, which can be learned. Compared to the standard neural networks or our proposed method, \glspl{LGN} are sparse because each neuron receives only two inputs rather than multiple inputs. Recently, \citep{Bacellar2024} expanded on \glspl{LGN} by incorporating flexible and differentiable lookup tables. While these advancements show promises, adapting them to modern neural network architectures such as \gls{CNN} or Transformers is challenging. Furthermore, these approaches have not been validated on large-scale datasets like \textsc{imagenet} or on tasks that require high precision, such as image segmentation or super-resolution, as demonstrated in our work.

\vspace{-0.5ex}

\paragraph{Summary.}
\cref{tab:method_summary} shows key characteristics of \gls{SOTA} \gls{BNN} methods.
These methods will be considered in our experiments.
Notice that all these techniques indeed have to involve operations on \gls{FP} latent weights during training, whereas our proposed method works directly on native Boolean weights.
In addition, most of \gls{BNN} methods incorporate \gls{FP} data and modules as mandatory components.
As a result, existing \glspl{BNN} consume much more training energy compared to our \ours method. An example is shown in \cref{fig:cifar10_vgg}, where we consider the \textsc{vgg-small} architecture \citep{Courbariaux2015,Simonyan2014} on \textsc{cifar10} dataset.
In \cref{sec:experiments} we will consider much larger datasets and networks on more challenging tasks.
We can see that our method achieves $36\times$ and more than $15\times$ energy reduction compared to the \gls{FP} baseline and \textsc{binarynet}, respectively, while yielding better accuracy than \glspl{BNN}.
Furthermore, \glspl{BNN} are commonly tied to specialized network architecture and have to employ costly multi-stage or \gls{KD} training.
Meanwhile, our Boolean framework is completely orthogonal to these \gls{BNN} methods.
It is generic and applicable for a wide range of network architectures, and its training procedure purely relies on Boolean logic from scratch.
Nevertheless, we stress that it is not obligatory to use all Boolean components in our proposed framework as it is flexible and can be extended to architectures comprised of a mix of Boolean and \gls{FP} modules.
This feature further improves the superior performance of our method as can be seen in \cref{fig:cifar10_vgg}, where we integrate \gls{BN} \citep{ioffe15} into our Boolean model, and we will demonstrate extensively in our experiments.
 
\section{Proposed Method}\label{sec:Method}

\subsection{Neuron Design}

\paragraph{Boolean Neuron.} For the sake of simplicity, we consider a linear layer for presenting the design. 
Let $w_0$, $(w_1, \ldots, w_m)$, and $(x_1, \ldots, x_m)$ be the bias, weights, and inputs of a neuron of input size $m \geq 1$. 
In \emph{the core use case} of our interest, these variables are all Boolean numbers. Let $\Lm$ be a logic gate such as \textsc{and}, \textsc{or}, \textsc{xor}, \textsc{xnor}.
The neuron's pre-activation output is given as follows:
\begin{equation}\label{eq:Preactivation}
	s = w_0 + \sum_{i=1}^m \Lm(w_i, x_i),
\end{equation}
where the summation is understood as the counting of \textsc{true}s. 

\paragraph{Mixed Boolean-Real Neuron.} To allow for flexible use and co-existence of this Boolean design with real-valued parts of a deep model, two cases of mixed-type data are considered including Boolean weights with real-valued inputs, and real-valued weights with Boolean inputs.
These two cases can be addressed by the following extension of Boolean logic to mixed-type data.
To this end, we first introduce the essential notations and definitions.
Specifically, we denote $\Bb := \{\True, \False\}$ equipped with the Boolean logic.
Here, $\True$ and $\False$ indicate \textsc{true} and \textsc{false}, respectively. 

\begin{definition}[Three-valued logic]\label{def:ThreeValueLogic}
	Define $\Mb \deq \Bb \cup \{0\}$ 
	with logic connectives defined according to those of Boolean logic as follows. 
	First, the negation is: $\neg \True = \False$, $\neg \False = \True$, and $\neg 0 = 0$. 
	Second, let $\Lm$ be a logic connective, denote by $\Lm_{\Mb}$ and $\Lm_{\Bb}$ when it is in $\Mb$ and in $\Bb$, respectively, then $\Lm_{\Mb}(a,b) = \Lm_{\Bb}(a,b)$ for $a, b \in \Bb$ and $\Lm_{\Mb}(a,b) = 0$ otherwise.
\end{definition}

\begin{notation}
	Denote by $\Lb$ a logic set (e.g., $\Bb$ or $\Mb$), $\Rb$ the real set, $\Zb$ the set of integers, $\Nb$ a numeric set (e.g., $\Rb$ or $\Zb$), and $\Db$ a certain set of $\Lb$ or $\Nb$.
\end{notation}

\begin{definition}\label{def:Real2Bool}
	For $x \in \Nb$, its logic value denoted by $x_{\logic}$ is given as $x_{\logic} = \True \Leftrightarrow x > 0$, $x_{\logic} = \False \Leftrightarrow x < 0$, and $x_{\logic} = 0 \Leftrightarrow x = 0$.
\end{definition}

\begin{definition}
	The magnitude of a variable $x$, denoted $|x|$, is defined as its usual absolute value if $x \in \Nb$. And for $x \in \Lb$: $|x| = 0$ if $x = 0$, and $|x| = 1$ otherwise.
\end{definition}
\begin{definition}[Mixed-type logic]\label{def:MixedLogic}
	For $\Lm$ a logic connective of $\Lb$ and variables $a$, $b$, operation $c = \Lm(a, b)$ is defined such that $|c| = |a||b|$ and $c_{\logic} = \Lm(a_{\logic}, b_{\logic})$. 
\end{definition}
Using \cref{def:MixedLogic}, neuron formulation \cref{eq:Preactivation} directly applies to the mixed Boolean-real neurons. 

\vspace{-2ex}

\paragraph{Forward Activation.} 
It is clear that there can only be one unique family of binary activation functions, which is the threshold function. 
Let $\tau$ be a scalar, which can be fixed or learned, the forward Boolean activation is given as: $y = \True$ if $s \geq \tau$ and $y = \False$ otherwise where $s$ is the preactivation.
The backpropagation throughout this activation will be described in \cref{app:scaling_activation}.

\subsection{Mathematical Foundation}\label{sec:Foundation}

In this section we describe the mathematical foundation for our method to train Boolean weights directly in the Boolean domain without relying on \gls{FP} latent weights. 
Due to the space limitation, essential notions necessary for presenting the main results are presented here while a comprehensive treatment is provided in \cref{appendix:booleanvariation}.

\begin{definition}\label{def:BoolOrder}
	Order relations `$<$' and `$>$' in $\Bb$ are defined as follows: $\False < \True$, and $\True > \False$.
\end{definition}

\begin{definition}\label{def:BoolVariation}
	For $a, b \in \Bb$, the variation from $a$ to $b$, denoted $\bvar(a \to b)$, is defined as: $\bvar(a \to b) \deq \True$ if $b > a$, $\deq 0$ if $b = a$, and $\deq \False$ if $b < a$.
\end{definition}

Throughout the paper, $\Fc(\mathbb{S},\mathbb{T})$ denotes the set of all functions from source $\mathbb{S}$ to image $\mathbb{T}$.

\begin{definition}\label{def:BoolFuncVar}
	For $f \in \Fc(\Bb, \Db)$, $\forall x \in \Bb$, write $\bvar f(x \to \neg x) := \bvar(f(x) \to f(\neg x))$. The variation of $f$ \wrt $x$, denoted $f'(x)$, is defined as: $f'(x) \deq \xnor(\bvar(x \to \neg x), \bvar f(x \to \neg x))$.
\end{definition}
\begin{remark}
	The usual notation of continuous derivative $f'$ is intentionally adopted here for Boolean variation for convenience and notation unification. Its underlying meaning, i.e., continuous derivative or Boolean variation, can be understood directly from the context where function $f$ is defined.  
\end{remark}
Intuitively, the variation of $f$ \wrt $x$ is $\True$ if $f$ varies in the same direction with $x$. 
\begin{example}\label{ex:XORVariation}
	Let $a \in \Bb$, $f(x) = \xor(x,a)$ for $x \in \Bb$, the variation of $f$ \wrt $x$ can be derived by establishing a truth table (see \cref{tab:VariationXOR} in \cref{appendix:Boolean}) from which we obtain $f'(x) = \neg a$.
\end{example}

\if False %
	\begin{proposition}\label{prop:B2BVariation}
		For $f, g \in \Fc(\Bb, \Bb)$, $\forall x, y \in \Bb$ the following properties hold:
		\begin{enumerate}
			\item $\bvar f(x \to y) = \xnor(\bvar(x \to y), f'(x)).$
			\item $(\neg f(x))' = \neg f'(x)$.
			\item $(g \circ f)'(x) = \xnor(g'(f(x)), f'(x))$.
		\end{enumerate}
	\end{proposition}
	
	\begin{proposition}\label{prop:B2NVariation}
		For $f \in \Fc(\Bb, \Nb)$, the following properties hold:
		\begin{enumerate}
			\item $x, y \in \Bb$: $\bvar f(x \to y) = \xnor(\bvar(x \to y), f'(x))$.
			\item $\alpha \in \Nb$: $(\alpha f)'(x) = \alpha f'(x)$.
			\item $g \in \Fc(\Bb, \Nb)$: $(f + g)'(x) = f'(x) + g'(x)$.
		\end{enumerate}
	\end{proposition}
	
	The proof is given in the supplementary. \cref{def:BoolFuncVar} is essentially different than Akers' one \cite{Akers1959} in that the later only expresses the sole variation of $f$ without relating it to the variable's variation. One of the consequences is that when knowing $f'$, our formulation allows to obtain $f$ from the variation of $x$ as given by \cref{prop:B2BVariation}-(i) and \cref{prop:B2NVariation}-(i) whereas Akers' formulation cannot.
\fi

For $f \in \Fc(\Zb, \Nb)$, its derivative, also known in terms of \emph{finite differences}, has been defined in the literature as $f'(x) = f(x+1) - f(x)$, see e.g. \citep{Jordan1950}. With the logic variation as introduced above, we can make this definition more generic as follows. 
\begin{definition}
	For $f \in \Fc(\Zb, \Db)$, the variation of $f$ \wrt $x \in \Zb$ is defined as $f'(x) \deq \bvar f(x \to x+1)$, 
	where $\bvar f$ is in the sense of the variation defined in $\Db$.
\end{definition}
\begin{theorem}%
	\label{thm:MainRules}
	The following properties hold:
	\begin{enumerate}
		\item For $f \in \Fc(\Bb, \Bb)$: $(\neg f)'(x) = \neg f'(x)$, $\forall x \in \Bb$.
		\item For $f \in \Fc(\Bb, \Nb)$, $\alpha \in \Nb$: $(\alpha f)'(x) = \alpha f'(x)$, $\forall x \in \Bb$.
		\item For $f, g \in \Fc(\Bb, \Nb)$: $(f + g)'(x) = f'(x) + g'(x)$, $\forall x \in \Bb$.
		\item For $\Bb \overset{f}{\to} \Bb \overset{g}{\to} \Db$: $(g \circ f)'(x) = \xnor(g'(f(x)), f'(x))$, $\forall x \in \Bb$.
		\item For $\Bb \overset{f}{\to} \Zb \overset{g}{\to} \Db$, $x \in \Bb$, if $|f'(x)| \leq 1$ and $g'(f(x)) =g'(f(x)-1)$, then:
		\begin{equation*}
			(g \circ f)'(x) = \xnor(g'(f(x)), f'(x)).
		\end{equation*}
	\end{enumerate}
\end{theorem}
The proof is provided in \cref{appendix:Boolean}. These results are extended to the multivariate case in a straightforward manner. For instance, for multivariate Boolean functions it is as follows. 
\begin{definition}\label{def:BoolFuncVar0}
	For $\xb = (x_1, \ldots, x_n) \in \Bb^n$, denote $\xb_{\neg i} := (x_1, \ldots, x_{i-1}, \neg x_i, x_{i+1}, \ldots, x_n)$ for $n \ge 1$ and $1 \leq i \leq n$. 
	For $f \in \Fc(\Bb^n, \Bb)$, the (partial) variation of $f$ \wrt $x_i$, denoted $f'_{i}(\xb)$ or $\bvar f(\xb)/\bvar x_i$, is defined as: $f'_{i}(\xb) \equiv \bvar f(\xb)/\bvar x_i \deq \xnor(\bvar(x_i \to \neg x_i), \bvar f(\xb \to \xb_{\neg i}))$.
\end{definition}
\begin{proposition}\label{prop:MultiVariate}
	Let $f \in \Fc(\Bb^n, \Bb)$, $n \geq 1$, and $g \in \Fc(\Bb, \Bb)$. For $1 \le i \le n$:
	\begin{equation}
		(g \circ f)'_i(\xb) = \xnor(g'(f(\xb)), f'_i(\xb)), \quad \forall \xb \in \Bb^n.
	\end{equation}
\end{proposition} 
\begin{example}\label{ex:XNORVariation}
	From \cref{ex:XORVariation}, we have $\bvar{\xor(x,a)}/\bvar{x} = \neg a$ for $a, x \in \Bb$. Using \cref{thm:MainRules}-(1) we have: $\bvar{\xnor(x,a)}/\bvar{x} = a$ since $\xnor(x,a) = \neg \xor(x,a)$.
\end{example}
\begin{example}\label{ex:PreActVariation}
	Apply \cref{thm:MainRules}-(3) to $s$ from \cref{eq:Preactivation}: $\bvar{s}/\bvar{w_i} = \bvar{\Lm(w_i, x_i)}/\bvar{w_i}$ and $\bvar{s}/\bvar{x_i} = \bvar{\Lm(w_i, x_i)}/\bvar{x_i}$. Then, for $\Lm = \xnor$ as an example, we have: $\bvar{s}/\bvar{w_i} = x_i$ and $\bvar{s}/\bvar{x_i} = w_i$.
\end{example}

\subsection{BackPropagation}\label{sec:Backprop}

With the notions introduced in \cref{sec:Foundation}, we can write signals involved in the backpropagation process as shown in \cref{fig:BpSignals}. Therein, layer $l$ is a Boolean layer of consideration. For the sake of presentation simplicity, layer $l$ is assumed a fully-connected layer, and:
\begin{equation}\label{eq:Forward}
	x_{k, j}^{l+1} = w_{0, j}^l + \sum_{i=1}^m \Lm\pren{x_{k, i}^l, w_{i, j}^l}, \quad 1 \leq j \leq n,
\end{equation}
where $\Lm$ is the utilized Boolean logic, $k$ denotes sample index in the batch, $m$ and $n$ are the usual layer input and output sizes. Layer $l$ is connected to layer $l+1$ that can be an activation layer, a batch normalization, an arithmetic layer, or any others. The nature of $\bvar{\Loss}/\bvar{x_{k,j}^{l+1}}$ depends on the property of layer $l+1$. It can be the usual gradient if layer $l+1$ is a real-valued input layer, or a Boolean variation if layer $l+1$ is a Boolean-input layer. 
Given $\bvar{\Loss}/\bvar{x_{k,j}^{l+1}}$, layer $l$ needs to optimize its Boolean weights and compute signal $\bvar{\Loss}/\bvar{x_{k,i}^{l}}$ for the upstream.  
Hereafter, we consider $\Lm = \xnor$ when showing concrete illustrations of the method.

\begin{figure}[!t]
	\centering
	\includegraphics[width=0.63\columnwidth]{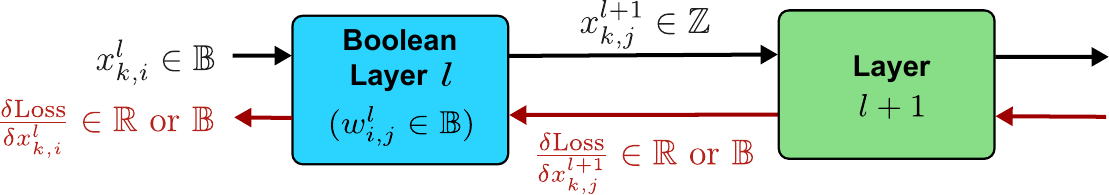}
	\caption{Illustration of {\color{BrickRed}backpropagation signals} with a {\color{Cerulean} Boolean linear layer}.
		Notice that the {\color{Green} subsequent layer} can be any \gls{FP}/Boolean layers or activation functions.
	}
	\label{fig:BpSignals}
	\vspace{-3ex}
\end{figure} 

\paragraph{Atomic Variation.} First, using \cref{thm:MainRules} and its extension to the multivariate case by \cref{prop:MultiVariate} in the same manner as shown in \cref{ex:PreActVariation}, we have:

\begin{equation}
	\frac{\bvar{x_{k,j}^{l+1}}}{\bvar{w_{i,j}^l}} = \frac{\bvar{ \Lm(x_{k,i}^l, w_{i, j}^l) }}{\bvar{w_{i,j}^l}} \overset{\Lm =\xnor}{=} x_{k,i}^l, \quad \frac{\bvar{x_{k,j}^{l+1}}}{\bvar{x_{k,i}^l}} = \frac{\bvar{ \Lm(x_{k,i}^l, w_{i, j}^l) }}{\bvar{x_{k,i}^l}} \overset{\Lm =\xnor}{=} w_{i, j}^l. 
\end{equation}

Using the chain rules given by \cref{thm:MainRules}--(4 \& 5), 
we have:

\begin{align}
	q_{i,j,k}^l := \frac{\bvar{\Loss}}{\bvar{w_{i,j}^{l}}}|_k & = \xnor(\frac{\bvar{\Loss}}{\bvar{x_{k,j}^{l+1}}}, \frac{\bvar{x_{k,j}^{l+1}}}{\bvar{w_{i,j}^l}}) \overset{\Lm =\xnor}{=}  \xnor(\frac{\bvar{\Loss}}{\bvar{x_{k,j}^{l+1}}}, x_{k,i}^l), \label{eq:AtomicWeightVariation}\\
	g_{k,i,j}^l := \frac{\bvar{\Loss}}{\bvar{x_{k,i}^{l}}}|_j & = \xnor(\frac{\bvar{\Loss}}{\bvar{x_{k,j}^{l+1}}}, \frac{\bvar{x_{k,j}^{l+1}}}{\bvar{x_{k,i}^l}}) \overset{\Lm =\xnor}{=} \xnor(\frac{\bvar{\Loss}}{\bvar{x_{k,j}^{l+1}}}, w_{i, j}^l). \label{eq:AtomicBprop}
\end{align}

\paragraph{Aggregation.} Atomic variation $q_{i,j,k}^l$ is aggregated over batch dimension $k$ while $g_{k,i,j}^l$ is aggregated over output dimension $j$. 
Let $\one{\cdot}$ be the indicator function. For $b \in \Bb$ and variable $x$, define: $\one{x = b} = 1$ if $x_{\logic} = b$ and $\one{x = b} = 0$ otherwise. 
Atomic variations are aggregated as:
\begin{align}
	q_{i,j}^l := \frac{\bvar{\Loss}}{\bvar{w_{i,j}^{l}}} & = \sum_k \one{q_{i,j,k}^l = \True}|q_{i,j,k}^l| - \sum_k \one{q_{i,j,k}^l = \False}|q_{i,j,k}^l|, \label{eq:AggrWeightVariation}\\
	g_{k,i}^l := \frac{\bvar{\Loss}}{\bvar{x_{k,i}^{l}}} & = \sum_j \one{g_{k,i,j}^l = \True}|g_{k,i,j}^l| - \sum_j \one{g_{k,i,j}^l = \False}|g_{k,i,j}^l|.\label{eq:AggrBprop}
\end{align}

\paragraph{Boolean Optimizer.} With $q_{i,j}^l$ obtained in \cref{eq:AggrWeightVariation}, the rule for optimizing $w_{i,j}^{l}$ subjected to making the loss decreased is simply given according to its definition as:
\begin{equation}\label{eq:OptimLogic}
	\boxed{w_{i,j}^l = \neg w_{i,j}^l  \textrm{ if } \xnor\pren{q_{i,j}^l, w_{i,j}^l} = \True.}
\end{equation}
\cref{eq:OptimLogic} is the core optimization logic based on which more sophisticated forms of optimizer can be developed in the same manner as different methods such as Adam, Adaptive Adam, etc. have been developed from the basic gradient descent principle. For instance, the following is an optimizer that accumulates $q_{i,j}^l$ over training iterations. Denote by $q_{i,j}^{l,t}$ the optimization signal at iteration $t$, and by $m_{i,j}^{l,t}$ its accumulator with $m_{i,j}^{l,0} := 0$ and: 
	\begin{equation}\label{eq:Accum}
		m_{i,j}^{l,t+1} = \beta^t m_{i,j}^{l,t} + \eta^t q_{i,j}^{l,t+1},
	\end{equation}
	where $\eta^t$ is an accumulation factor that can be tuned as a hyper-parameter, and $\beta^t$ is an auto-regularizing factor that expresses the system's state at time $t$. Its usage is linked to brain plasticity \citep{Fuchs2014} and Hebbian theory \citep{Hebb2005} 
	forcing weights to adapt to their neighborhood.
	For the chosen weight's neighborhood, for instance, neuron, layer, or network level, $\beta^t$ is given as:
	\begin{equation}
		\beta^t = \frac{\textrm{Number of unchanged weights at } t}{\textrm{Total number of weights}}.
	\end{equation}
	In the experiments presented later, $\beta^t$ is set to per-layer basis. %
	Finally, the learning process is as described in \cref{algo:BooleanTraining}.
	We encourage the readers to check the detailed implementations, practical considerations, and example codes of our proposed method, available in \cref{appendix:implementation} and \cref{appendix:train_regu}.

\begin{minipage}[][][t]{0.5\textwidth}
\paragraph{Convergence Analysis.} The following result describes how the iterative logic optimization based on \cref{eq:OptimLogic} minimizes a predefined loss $f$, under the standard non-convex assumption. The technical assumptions and the proof are given in \cref{sec:cvg}.
\begin{theorem}\label{thm:Convegence}
	Under the specified assumptions, Boolean optimization logic converges at:
	\begin{equation}
		\frac{1}{T}\sum_{t=0}^{T-1} \E{\left\|\nabla f\left(w_{t}\right)\right\|^{2}} \le  \frac{A^*}{T\eta} + B^* \eta + C^* \eta^2 + Lr_d,
	\end{equation}
	where $A^* = {2(f(w_0)-f_*)}$ with $f_*$ being uniform lower bound assumed exists, $B^*= 2L\sigma^2$, $C^*= 4L^2\sigma^2 \frac{\gamma}{(1-\gamma)^2}$ in which $L$ is the assumed Lipschitz constant, and $r_d = \nicefrac{d\kappa}{2}$.
\end{theorem}
The bound in \cref{thm:Convegence} contains four terms. The first is typical for a general non-convex target and expresses how initialization affects the convergence. The second and third terms depend on the fluctuation of the minibatch gradients.
There is an ``error bound'' of $2Ld\kappa$ independent of $T$. %
This error bound is the cost of using discrete weights as part of the optimization algorithm. Previous work with quantized models also includes such error bounds \cite{Li2017,Li2019b}.
\end{minipage}
\hfill
\begin{minipage}[][][t]{0.45\textwidth}
	\begin{algorithm}[H]\small
	\caption{Illustration with a FC layer.}
	\label{algo:BooleanTraining}
	\SetKwInOut{Input}{Input}
	\SetKwInOut{Output}{Output}
	\SetKwBlock{Loop}{Loop}{end}
	\SetKwInOut{Initialize}{Initialize}
	\SetKwFor{When}{When}{do}{end}
	\SetKwFunction{Wait}{Wait}
	\SetAlgoLined
	\SetNoFillComment
	\Input{
		Learning rate $\eta$, nb iterations $T$\;
	}
	\Initialize{	
		$m_{i,j}^{l,0} = 0$; $\beta^0 = 1$\;
	}
	\For{$t = 0,\dots, T-1$}{
		\tcc{\textbf{1. Forward}}
		Compute $x^{l+1,t}$ following \cref{eq:Forward}\; %
		\tcc{\textbf{2. Backward}}
		Receive $\frac{\bvar{\Loss}}{\bvar{x_{k,j}^{l+1,t}}}$ from downstream layer\;
		\tcc{\textbf{2.1  Backpropagation}}
		Compute and backpropagate $g^{l,t}$ of \cref{eq:AggrBprop}\; %
		\tcc{\textbf{2.2  Weight update process}}
		$N_{\textrm{tot}} := 0$, $N_{\textrm{unchanged}} := 0$\;
		\ForEach{$w_{i,j}^{l}$}{
			Compute $q_{i,j}^{l,t+1}$ following \cref{eq:AggrWeightVariation}\;  %
			Update $m_{i,j}^{l,t+1} = \beta^t m_{i,j}^{l,t} + \eta^t q_{i,j}^{l,t+1}$\;
			$N_{\textrm{tot}} \gets N_{\textrm{tot}} + 1$\;
			\eIf{$\xnor(m_{i,j}^{l,t+1}, w_{i,j}^{l,t}) = \True$}{ %
				$w_{i,j}^{l,t+1} = \neg w_{i,j}^{l,t}$ \hspace{-2.7ex} \tcc*[r]{invert}
				$m_{i,j}^{l,t+1} = 0$\;
			}{
				$w_{i,j}^{l,t+1} = w_{i,j}^{l,t}$  \tcc*[r]{keep}
				$N_{\textrm{unchanged}} \gets N_{\textrm{unchanged}} + 1$\;
			}%
		}
		Update $\eta^{t+1}$, $\beta^{t+1} = N_{\textrm{unchanged}}/N_{\textrm{tot}}$ \;
	}
\end{algorithm}

\end{minipage}

\paragraph{Regularization.} Exploding and vanishing gradients are two well-known issues when it comes to train deep neural networks. During Boolean Logic training, our preliminary experiments indicated that the backpropagation signal also experiences similar phenomenon, resulting in unstable training. 
Our idea is to scale the backpropagation signals so as to match their variance. Thanks to the Boolean structure, assumptions on the involved signals can be made so that the scaling factor can be analytically derived in closed-form without need of learning or batch statistics computation. Precisely, through linear layers of output size $m$, the backpropagation signal is scaled with $\sqrt{\nicefrac{2}{m}}$, and for convolutional layers of output size $c_{\textrm{out}}$, stride $v$ and kernel sizes $k_x, k_y$, the backpropagation signal is scaled with $2\sqrt{\nicefrac{2v}{c_{\textrm{out}} k_x k_y}}$ if maxpooling is applied, or $\sqrt{\nicefrac{2v}{c_{\textrm{out}} k_x k_y}}$ otherwise. The full detail is presented in \cref{appendix:train_regu}.

\section{Experiments}\label{sec:experiments}

Our \ours method achieves extreme compression by using both Boolean activations and weights.
We rigorously evaluate its performance on challenging precision-demanding tasks, including \emph{image classification} on \textsc{cifar10} \citep{Krizhevsky2009} and \textsc{imagenet} \citep{Krizhevsky2012},
as well as \emph{super-resolution} on five popular datasets. %
Furthermore, recognizing the necessity of deploying efficient lightweight models for edge computing, we delve into three fine-tuning scenarios, showcasing the adaptability of our approach.
Specifically, we investigate fine-tuning Boolean models for image classification on \textsc{cifar10} and \textsc{cifar100} using \textsc{vgg-small}. %
For \emph{segmentation tasks}, we study \textsc{deeplabv3} \citep{chen2017rethinking} fine-tuned on \textsc{cityscapes} \citep{Cordts_2016_CVPR} and \textsc{pascal voc 2012} \citep{pascal-voc-2012} datasets. The backbone for such a model is our Boolean \textsc{resnet18} \citep{He2016} network trained from scratch on \textsc{imagenet}.
Finally, we consider an evaluation in the domain of \emph{natural language understanding}, fine-tuning \textsc{bert} \citep{DevlinCLT19}, a transformer-based \citep{Vaswani2017} language model, on the \textsc{glue} benchmark \citep{Wang2019d}.

\paragraph{Experimental Setup.}
\glsreset{FP}
\glsreset{BN}
To construct our \ours models, we introduce Boolean weights and activations and substitute \gls{FP} arithmetic layers with Boolean equivalents.
Throughout all benchmarks, we maintain the general network design of the chosen \gls{FP} baseline, while excluding \gls{FP}-specific components like ReLU, PReLU activations, or \gls{BN} \citep{ioffe15}, unless specified otherwise.
Consistent with the common setup in existing literature \citep[see, e.g.,][]{Rastegari2016,Chmiel2021,Bethge_2021_WACV}, only the first and last layers remain in \gls{FP} and are optimized using an Adam optimizer \citep{Kingma2014}.
Comprehensive experiment details for reproducibility are provided in \Cref{exp_setup}.

\paragraph{Complexity Evaluation.}
It has been demonstrated that relying solely on \textsc{flops} and \textsc{bops} for complexity assessment is inadequate \citep{Sze2017,Sze2020a,Yang2017a,Strubell2019}.
In fact, these metrics fail to capture the actual load caused by propagating \gls{FP} data through the \glspl{BNN}. %
Instead, energy consumption serves as a crucial indicator of efficiency.
Given the absence of native Boolean accelerators, we estimate analytically energy consumption by analyzing the arithmetic operations, data movements within storage/processing units, and the energy cost of each operation.
This approach is implemented for the Nvidia GPU (Tesla V100) and Ascend \cite{Liao2021} architectures.
Further details are available in \cref{appendix:complexity}.

\subsection{Image Classification}\label{sec:classification}

Our \ours method is tested on two network configurations: \textit{small \& compact} and \textit{large \& deep}.
In the former scenario, we utilize the \textsc{vgg-small} \citep{Simonyan2014} baseline trained on \textsc{cifar10}. 
Evaluation of our Boolean architecture is conducted both without \gls{BN} \citep{Courbariaux2015}, and with \gls{BN} including activation from \citep{Liu2018}.
These designs achieve $90.29 \small{\pm 0.09}$\% (estimated over six repetitions) and $92.37 \small{\pm 0.01}$\% (estimated over five repetitions) accuracy, respectively (see \Cref{table:cifar}).
Notably, without \gls{BN}, our results align closely with \textsc{binaryconnect} \citep{Courbariaux2015}, which employs $32$-bit activations during both inference and training. 
Furthermore, \gls{BN} brings the accuracy within $1$ point of the \gls{FP} baseline. 
Additional results are provided in the supplementary material for \textsc{vgg-small} models ending with 1 \textsc{fc} layer.

Our method requires much less energy than the \gls{FP} baseline. In particular, it consumes less than 5\% of energy for our designs with and without \gls{BN} respectively.
These results highlight the remarkable energy efficiency of our \ours method in both inference and training, surpassing latent-weight based training methods \citep{Courbariaux2015, Rastegari2016, Hubara2017} reliant on \gls{FP} weights.
Notably, despite a slight increase in energy consumption, utilizing \gls{BN} yields superior accuracy.
Even with \gls{BN}, our approach maintains superior efficiency compared to alternative methods, further emphasizing the flexibility of our approach in training networks with a blend of Boolean and \gls{FP} components.

In the \textit{large \& deep} case, we consider the \textsc{resnet18} baseline trained from scratch on \textsc{imagenet}.
We compare our approach to methods employing the same baseline, larger architectures, and additional training strategies such as \gls{KD} with a \textsc{resnet34} teacher or \gls{FP}-based shortcuts \citep{Liu2018}.
Our method consistently achieves the highest accuracy across all categories, ranging from the standard model ($51.8$\% accuracy) to larger configurations ($70.0$\% accuracy), as shown in \cref{table:imagenet_resntricks}.
Additionally, our \ours method exhibits the smallest energy consumption in most categories, with a remarkable $24.45\%$ for our large architecture with and without \gls{KD}.
Notably, our method outperforms the \gls{FP} baseline when using $4\times$ filter enlargement (base 256), providing significant energy reduction ($24.45\%$).
Furthermore, it surpasses  the \gls{SOTA} \textsc{pokebnn} \citep{zhang2022a}, utilizing \textsc{resnet50} as a teacher.

For completeness, we also  implemented neural gradient quantization, utilizing \textsc{int4} quantization with a logarithmic round-to-nearest approach \citep{Chmiel2021} and statistics-aware weight binning \citep{choi2018bridging}.
Our experiments on \textsc{imagenet} confirm that 4-bit quantization is sufficient to achieve standard \gls{FP} performances, reaching $67.53$\% accuracy in $100$ epochs (further details provided in \Cref{nn_quant}).

\begin{table}[t!]
\centering
\begin{tabular}{cc}
\begin{minipage}{0.475\textwidth}
 \begin{minipage}{1\textwidth}
 	
\caption{Results with \textsc{vgg-small} on \textsc{cifar10}.
`Cons.' is the energy consumption w.r.t. the \gls{FP} baseline, evaluated on 1 training iteration.}
\resizebox{1\textwidth}{!}{
	\begin{tabular}{ l  r@{/}l  c@{.}l  r@{.}l  r@{.}l}
		\toprule
		\multirow{2}{*}{\textbf{Method}}     & \multicolumn{2}{c}{\multirow{2}{*}{\textbf{W/A}}} & \multicolumn{2}{c}{\multirow{2}{*}{\textbf{Acc.(\%)}}} & \multicolumn{2}{c}{\textbf{Cons.(\%)}} & \multicolumn{2}{c}{\textbf{Cons. (\%)}} \\ %
		                                     & \multicolumn{2}{c}{ }     &    \multicolumn{2}{c}{ }         & \multicolumn{2}{c}{\textbf{Ascend}}    & \multicolumn{2}{c}{\textbf{Tesla V100}}                       \\
		\midrule\midrule
		Full-precision \cite{Zhang2018}      & 32                               & 32                                    & 93                                     & 80                                      & 100 & 00 & 100 & 00 \\
		\tikzexternaldisable ({\protect\tikz[baseline=-.65ex]\protect\draw[thick, color=color_binaryconnect, fill=color_binaryconnect, mark=*, mark size=2pt, line width=1.25pt] plot[] (-.0, 0)--(-0,0);}) \tikzexternalenable \textsc{binaryconnect} \citep{Courbariaux2015} & 1                                & 32                                    & 90                                     & 10                                      & 38  & 59 & 48  & 49 \\
		\tikzexternaldisable ({\protect\tikz[baseline=-.65ex]\protect\draw[thick, color=color_xnornet, fill=color_xnornet, mark=*, mark size=2pt, line width=1.25pt] plot[] (-.0, 0)--(-0,0);}) \tikzexternalenable \textsc{xnor-net} \citep{Rastegari2016}        & 1                                & 1                                     & 89                                     & 83                                      & 34  & 21 & 45  & 68 \\
		\tikzexternaldisable ({\protect\tikz[baseline=-.65ex]\protect\draw[thick, color=color_binarynet, fill=color_binarynet, mark=*, mark size=2pt, line width=1.25pt] plot[] (-.0, 0)--(-0,0);}) \tikzexternalenable \textsc{binarynet} \citep{Courbariaux2016}             & 1                                & 1                                     & 89                                     & 85                                      & 32  & 60 & 43  & 61 \\
		\tikzexternaldisable ({\protect\tikz[baseline=-.65ex]\protect\draw[thick, color=color_booldl, fill=color_booldl, mark=*, mark size=2pt, line width=1.25pt] plot[] (-.0, 0)--(-0,0);}) \tikzexternalenable \ours w/o \textsc{BN} [\textbf{Ours}]                & 1                                & 1                                     & 90                                     & 29                                      & 3   & 64 & 2   & 78 \\
		\tikzexternaldisable ({\protect\tikz[baseline=-.65ex]\protect\draw[thick, color=color_booldl, fill=color_booldl, mark=*, mark size=2pt, line width=1.25pt] plot[] (-.0, 0)--(-0,0);}) \tikzexternalenable \ours with \textsc{BN} [\textbf{Ours}]               & 1                                & 1                                     & \textbf{92}                            & \textbf{37}                             & 4   & 87 & 3   & 71 \\ [.5ex] \bottomrule
	\end{tabular}}
\label{table:cifar}

 \end{minipage}

 \vspace{0.2cm}
 
 \begin{minipage}{1\textwidth}
 	\caption{Super-resolution results measured in PSNR (dB) ($\uparrow$), using the \textsc{edsr} baseline \citep{Lim2017}.}
\resizebox{\textwidth}{!}{
	\begin{tabular}{llccccc} \toprule
		\multirow{2}{*}{\textbf{Task}} & \multirow{2}{*}{\textbf{Method}}                                                                                                                                                                                                        & \textsc{set5} & \textsc{set14} & \textsc{bsd100} & \textsc{urban100} & \textsc{div2k} \\
		                               &                                                                                                                                                                                                                                         & \citep{bevilacqua2012} & \citep{zeyde2012} & \citep{huang2015} & \citep{martin2001} & \citep{Agustsson2017, Timofte2017} \\ \midrule\midrule
		\multirow{3}{*}{$\times 2$}    & \textsc{full edsr (fp)}                                                                                                                                                                                                                 & 38.11         & 33.92          & 32.32           & 32.93             & 35.03          \\
		                               & \textsc{small edsr (fp)}                                                                                                                                                                                                                & 38.01         & 33.63          & 32.19           & 31.60             & 34.67          \\
		                               & \tikzexternaldisable ({\protect\tikz[baseline=-.65ex]\protect\draw[thick, color=color_booldl, fill=color_booldl, mark=*, mark size=2pt, line width=1.25pt] plot[] (-.0, 0)--(-0,0);}) \tikzexternalenable \ours [\textbf{Ours}] & 37.42         & 33.00          & 31.75           & 30.26             & 33.82          \\ \bottomrule
		\multirow{3}{*}{$\times 3$}    & \textsc{full edsr (fp)}                                                                                                                                                                                                                 & 34.65         & 30.52          & 29.25           & $--$              & 31.26          \\
		                               & \textsc{small edsr (fp)}                                                                                                                                                                                                                & 34.37         & 30.24          & 29.10           & $--$              & 30.93          \\
		                               & \tikzexternaldisable ({\protect\tikz[baseline=-.65ex]\protect\draw[thick, color=color_booldl, fill=color_booldl, mark=*, mark size=2pt, line width=1.25pt] plot[] (-.0, 0)--(-0,0);}) \tikzexternalenable \ours [\textbf{Ours}] & 33.56         & 29.70          & 28.72           & $--$              & 30.22          \\ \bottomrule
		\multirow{3}{*}{$\times 4$}    & \textsc{full edsr (fp)}                                                                                                                                                                                                                 & 32.46         & 28.80          & 27.71           & 26.64             & 29.25          \\
		                               & \textsc{small edsr (fp)}                                                                                                                                                                                                                & 32.17         & 28.53          & 27.62           & 26.14             & 29.04          \\
		                               & \tikzexternaldisable ({\protect\tikz[baseline=-.65ex]\protect\draw[thick, color=color_booldl, fill=color_booldl, mark=*, mark size=2pt, line width=1.25pt] plot[] (-.0, 0)--(-0,0);}) \tikzexternalenable \ours [\textbf{Ours}] & 31.23         & 27.97          & 27.24           & 25.12             & 28.36          \\ \bottomrule
	\end{tabular}}
\label{table:super}

 \end{minipage}
 
 \vspace{0.2cm}
 
 \begin{minipage}{1\textwidth}
 	\centering
 	\caption{Image segmentation results.}
\resizebox{0.9\columnwidth}{!}{
	\begin{tabular}{ccc}
		\toprule
		\textbf{Dataset}                     & \textbf{Model}                                                                                                                                                                                                                          & \textbf{mIoU (\%) ($\uparrow$)} \\ \midrule\midrule
		\multirow{3}{*}{\textsc{cityscapes}} & \textsc{fp baseline}                                                                                                                                                                                                                             & 70.7               \\
		                                     & \textsc{binary dad-net} \cite{frickenstein2020binary}                                                                                                                                                                                            & 58.1               \\
		                                     & \tikzexternaldisable ({\protect\tikz[baseline=-.65ex]\protect\draw[thick, color=color_booldl, fill=color_booldl, mark=*, mark size=2pt, line width=1.25pt] plot[] (-.0, 0)--(-0,0);}) \tikzexternalenable \ours [\textbf{Ours}] & \textbf{67.4}      \\
		\midrule
		\multirow{2}{*}{\textsc{pascal voc 2012}}     & \textsc{fp baseline}                                                                                                                                                                                                                              & 72.1               \\
		                                     & \tikzexternaldisable ({\protect\tikz[baseline=-.65ex]\protect\draw[thick, color=color_booldl, fill=color_booldl, mark=*, mark size=2pt, line width=1.25pt] plot[] (-.0, 0)--(-0,0);}) \tikzexternalenable \ours [\textbf{Ours}] & 67.3               \\
		\bottomrule
	\end{tabular}}
\label{tab:seg_cs_val}

 \end{minipage}
 \end{minipage}
 &
 \begin{minipage}{0.475\textwidth}
 \begin{minipage}{1\textwidth}
 	
\caption{
Results with \textsc{resnet18} baseline on \textsc{imagenet}. `Base' is the mapping dimension of the first layer. Energy consumption is evaluated on 1 training iteration. `Cons.' is the energy consumption w.r.t. the \gls{FP} baseline.}

\resizebox{\textwidth}{!}{

\begin{tabular}{clccc} 
	\toprule
	\textbf{Training}   & \multirow{2}{*}{\textbf{Method}}	& \textbf{Acc.}   & \textbf{Cons. (\%}) 	& \textbf{Cons. (\%)}   \\
	\textbf{Modality}   &  				& \textbf{(\%)}   & \textbf{Ascend} 	& \textbf{Tesla V100}   \\
	\midrule\midrule
	\textsc{fp baseline}							& \textsc{resnet18} \cite{He2016}             &  69.7     	& 100.00 	& 100.00\\
	\midrule
	& \tikzexternaldisable ({\protect\tikz[baseline=-.65ex]\protect\draw[thick, color=color_binarynet, fill=color_binarynet, mark=*, mark size=2pt, line width=1.25pt] plot[] (-.0, 0)--(-0,0);}) \tikzexternalenable \textsc{binarynet} \citep{Courbariaux2016}         &  42.2       	& $--$     	& $--$\\
	& \tikzexternaldisable ({\protect\tikz[baseline=-.65ex]\protect\draw[thick, color=color_xnornet, fill=color_xnornet, mark=*, mark size=2pt, line width=1.25pt] plot[] (-.0, 0)--(-0,0);}) \tikzexternalenable \textsc{xnor-net} \citep{Rastegari2016}      &  51.2        & $--$     	& $--$\\
	& \tikzexternaldisable ({\protect\tikz[baseline=-.65ex]\protect\draw[thick, color=color_booldl, fill=color_booldl, mark=*, mark size=2pt, line width=1.25pt] plot[] (-.0, 0)--(-0,0);}) \tikzexternalenable \ours + \textsc{bn}  \small{(Base 64)}    &  51.8        & 8.77  	& 3.87 \\
	\midrule
	\textsc{fp shortcut} 						& \tikzexternaldisable ({\protect\tikz[baseline=-.65ex]\protect\draw[thick, color=color_birealnet, fill=color_birealnet, mark=*, mark size=2pt, line width=1.25pt] plot[] (-.0, 0)--(-0,0);}) \tikzexternalenable\textsc{bi-realnet:18} \citep{Liu2018}       &  56.4        & 46.60 	& 32.99\\
	\midrule
	\multirow{3}{4em}{\textsc{large models}}	& \tikzexternaldisable ({\protect\tikz[baseline=-.65ex]\protect\draw[thick, color=color_birealnet, fill=color_birealnet, mark=*, mark size=2pt, line width=1.25pt] plot[] (-.0, 0)--(-0,0);}) \tikzexternalenable\textsc{bi-realnet:34} \citep{Liu2018}       &  62.2        & 80.00 	& 58.24\\
	& \tikzexternaldisable ({\protect\tikz[baseline=-.65ex]\protect\draw[thick, color=color_birealnet, fill=color_birealnet, mark=*, mark size=2pt, line width=1.25pt] plot[] (-.0, 0)--(-0,0);}) \tikzexternalenable\textsc{bi-realnet:152} \cite{Liu2018}      &  64.5        & $--$     	& $--$\\
	&\tikzexternaldisable ({\protect\tikz[baseline=-.65ex]\protect\draw[thick, color=color_meliusnet, fill=color_meliusnet, mark=*, mark size=2pt, line width=1.25pt] plot[] (-.0, 0)--(-0,0);}) \tikzexternalenable \textsc{melius-net:29} \citep{Bethge_2021_WACV}    &  65.8        & $--$     	& $--$\\
	& \tikzexternaldisable ({\protect\tikz[baseline=-.65ex]\protect\draw[thick, color=color_booldl, fill=color_booldl, mark=*, mark size=2pt, line width=1.25pt] plot[] (-.0, 0)--(-0,0);}) \tikzexternalenable \ours  \small{(Base 256)}    &  \textbf{66.9}  & 38.82    & 24.45\\
	\midrule
	\multirow{4}{4em}{\textsc{kd: resnet34}}     & \tikzexternaldisable ({\protect\tikz[baseline=-.65ex]\protect\draw[thick, color=color_real2binary, fill=color_real2binary, mark=*, mark size=2pt, line width=1.25pt] plot[] (-.0, 0)--(-0,0);}) \tikzexternalenable \textsc{real2binary} \citep{Martinez2020}    &  65.4        & $--$     	& $--$ \\
	& \tikzexternaldisable ({\protect\tikz[baseline=-.65ex]\protect\draw[thick, color=color_reactnet, fill=color_reactnet, mark=*, mark size=2pt, line width=1.25pt] plot[] (-.0, 0)--(-0,0);}) \tikzexternalenable \textsc{reactnet-resnet18} \cite{Liu2020}   &  65.5        & 45.43 	& 77.89\\
	& \tikzexternaldisable ({\protect\tikz[baseline=-.65ex]\protect\draw[thick, color=color_bnext, fill=color_bnext, mark=*, mark size=2pt, line width=1.25pt] plot[] (-.0, 0)--(-0,0);}) \tikzexternalenable \textsc{bnext:18} \cite{Guo2022}             &  68.4        & 45.48 	& 37.51\\
	&\tikzexternaldisable ({\protect\tikz[baseline=-.65ex]\protect\draw[thick, color=color_booldl, fill=color_booldl, mark=*, mark size=2pt, line width=1.25pt] plot[] (-.0, 0)--(-0,0);}) \tikzexternalenable \ours + \textsc{bn}  \small{(Base 192)})            &  65.9        & 26.91 	& 16.91\\
	& \tikzexternaldisable ({\protect\tikz[baseline=-.65ex]\protect\draw[thick, color=color_booldl, fill=color_booldl, mark=*, mark size=2pt, line width=1.25pt] plot[] (-.0, 0)--(-0,0);}) \tikzexternalenable \ours  \small{(Base 256)}    &  \textbf{70.0} & 38.82    & 24.45\\
	\midrule
	\textsc{kd: resnet50}                        & \tikzexternaldisable ({\protect\tikz[baseline=-.65ex]\protect\draw[thick, color=color_pokebnn, fill=color_pokebnn, mark=*, mark size=2pt, line width=1.25pt] plot[] (-.0, 0)--(-0,0);}) \tikzexternalenable \textsc{pokebnn-resnet18} \cite{zhang2022a} &  65.2        & $--$    & $--$ \\
	\bottomrule
\end{tabular}}
\label{table:imagenet_resntricks}

 \end{minipage}
 
 \begin{minipage}{1\textwidth}
 	\caption{Results with \textsc{vgg-small} baseline fine-tuned on \textsc{cifar10} and \textsc{cifar100}. `FT' means `Fine-Tuning'.}
\resizebox{1\columnwidth}{!}{
	\begin{tabular}{ m{1.5em}  m{6em}  m{3.25em}  m{4.75em}  m{4.5em}  m{3em}}
		\toprule
		\textbf{Ref.}                                   & \textbf{Method}                                                                                                                                                                                                                        & \textbf{Model Init.} & \textbf{Train./FT Dataset} & \textbf{Bitwidth W/A/G} & \textbf{Acc. (\%)} \\
		\midrule\midrule
		\label{config:a}\hyperref[config:a]{\textsc{a}} & \textsc{fp baseline}                                                                                                                                                                                                                   & Random               & \textsc{cifar10}           & $32/32/32$              & 95.27              \\ %
		\label{config:b}\hyperref[config:b]{\textsc{b}} & \textsc{fp baseline} \tablefootnote{\label{note}\textsc{vgg-small} with the last \gls{FP} layer mapping to 100 classes.}
		                                                & Random                                                                                                                                                                                                                                 & \textsc{cifar100}    & $32/32/32$                 & 77.27                                        \\
		\hyperref[config:f]{\textsc{c}}                 & \tikzexternaldisable ({\protect\tikz[baseline=-.65ex]\protect\draw[thick, color=color_booldl, fill=color_booldl, mark=*, mark size=2pt, line width=1.25pt] plot[] (-.0, 0)--(-0,0);}) \tikzexternalenable \ours                & Random               & \textsc{cifar10}           & $1/1/16$                & 90.29              \\
		\hyperref[config:d]{\textsc{d}}                 & \tikzexternaldisable ({\protect\tikz[baseline=-.65ex]\protect\draw[thick, color=color_booldl, fill=color_booldl, mark=*, mark size=2pt, line width=1.25pt] plot[] (-.0, 0)--(-0,0);}) \tikzexternalenable \ours\footref{note}
		                                                & Random                                                                                                                                                                                                                                 & \textsc{cifar100}    & $1/1/16$                   & 68.43                                        \\ \hline
		\label{config:e}\hyperref[config:e]{\textsc{e}} & \textsc{fp baseline}\footref{note}
		                                                & \hyperref[config:a]{\textsc{a}}                                                                                                                                                                                                                                      & \textsc{cifar100}    & $32/32/32$                 & 76.74                                        \\ %
		\hyperref[config:f]{\textsc{f}}                 & \tikzexternaldisable ({\protect\tikz[baseline=-.65ex]\protect\draw[thick, color=color_booldl, fill=color_booldl, mark=*, mark size=2pt, line width=1.25pt] plot[] (-.0, 0)--(-0,0);}) \tikzexternalenable \ours \footref{note}
		                                                & \hyperref[config:c]{\textsc{c}}                                                                                                                                                                                                                                      & \textsc{cifar100}    & $1/1/16$                   & 68.37                                        \\ \hline
		\label{config:g}\hyperref[config:g]{\textsc{g}} & \textsc{fp baseline}                                                                                                                                                                                                                   & \hyperref[config:b]{\textsc{b}}                    & \textsc{cifar10}           & $32/32/32$              & 95.77              \\ %
		\hyperref[config:h]{\textsc{h}}                 & \tikzexternaldisable ({\protect\tikz[baseline=-.65ex]\protect\draw[thick, color=color_booldl, fill=color_booldl, mark=*, mark size=2pt, line width=1.25pt] plot[] (-.0, 0)--(-0,0);}) \tikzexternalenable \ours                & \hyperref[config:d]{\textsc{d}}                    & \textsc{cifar10}           & $1/1/16$                & 92.09              \\ \hline
	\end{tabular}}
\label{table:finetune}

 \end{minipage}
 \end{minipage}
\end{tabular}
\end{table}

\subsection{Image Super-resolution}

Next, we evaluate the efficacy of our \ours method to synthesize data.
We use a compact \textsc{edsr} network \cite{Lim2017} as our baseline, referred to as \textsc{small edsr}, comprising eight residual blocks.
Our \ours model employs Boolean residual blocks without \gls{BN}.
Results, presented in \cref{table:super}, based on the official implementation and benchmark\footnote{\url{https://github.com/sanghyun-son/EDSR-PyTorch}}, reveal remarkable similarity to the \gls{FP} reference at each scale.
Particularly noteworthy are the prominent results achieved on \textsc{set14} and \textsc{bsd100} datasets.
Our method consistently delivers high PSNR for high-resolution images, such as \textsc{div2k}, and even higher for low-resolution ones, like \textsc{set5}.
However, akin to \textsc{edsr}, our approach exhibits a moderate performance reduction at scale $4\times$.
These findings highlight the capability of our method to perform adequately on detail-demanding tasks while exhibiting considerable robustness across image resolutions.

\subsection{Adaptability on New Data}

\paragraph{Image classification fine-tuning.}

We aim to assess the adaptability of our method to similar problems but different datasets, a common scenario for edge inference tasks.
We employ the \textsc{vgg-small} architecture without \gls{BN} under two training configurations.
Firstly, the \ours model is trained from scratch with random initialization on \textsc{cifar10} (\label{config:c}\hyperref[config:c]{\textsc{ref. c}}) and \textsc{cifar100} (\label{config:d}\hyperref[config:d]{\textsc{ref. d}}).
Secondly, we fine-tune the trained networks on \textsc{cifar100} (\label{config:f}\hyperref[config:f]{\textsc{ref. f}}) and \textsc{cifar10} (\label{config:h}\hyperref[config:h]{\textsc{ref. h}}), respectively.
Notably, in \Cref{table:finetune}, fine-tuning our trained model on \textsc{cifar100} (\hyperref[config:f]{\textsc{ref. f}}) results in a model almost identical to the model trained entirely from scratch (\hyperref[config:d]{\textsc{ref. d}}).
Additionally, a noteworthy result is observed with our model (\hyperref[config:h]{\textsc{ref. h}}), which achieves higher accuracy than the model trained from scratch (\hyperref[config:c]{\textsc{ref. c}}).

\paragraph{Image segmentation fine-tuning.}

Next, we expand the scope of the aforementioned fine-tuning experiment to encompass a larger network and a different task.
The baseline is the \textsc{deeplabv3} network for semantic segmentation.
It consists of our Boolean \textsc{resnet18} (without \gls{BN}) as the backbone, followed by the Boolean \gls{ASPP} module \citep{chen2017rethinking} .
We refrain from utilizing auxiliary loss or knowledge distillation techniques, as these methods introduce additional computational burdens, which are contrary to our objective of efficient on-device training.
As demonstrated in \cref{tab:seg_cs_val}, our method achieves a notable $67.4\%$ mIoU on \textsc{cityscapes} (see \cref{fig:main_cityscapes} for prediction examples).
This result surpasses the \gls{SOTA}, \textsc{binary dad-net} \citep{frickenstein2020binary}, and approaches the performance of the \gls{FP} baseline.
Likewise, on \textsc{pascal voc 2012}, our methodology nears the performance of the \gls{FP} baseline.
Importantly, these improvements are attained without the intermediate use of \gls{FP} parameters during training, highlighting the efficiency and effectiveness of our approach.
This shows that our method not only preserves the inherent lightweight advantages of highly quantized neural networks but also significantly enhances performance in complex segmentation tasks.

\paragraph{BERT fine-tuning for NLU tasks.}

Finally, we consider fine-tuning \textsc{bert} \citep{DevlinCLT19}, a transformer-based model \citep{Vaswani2017}, on the \textsc{glue} benchmark \citep{Wang2019d}.
We follow the standard experimental protocol as in \citep{DevlinCLT19, Bai2021, Qin2022}.
Our model and the chosen baselines are employed with 1-bit bitwidth for both weights and activations.
Our Boolean \textsc{bert} model is inspired by \textsc{bit} \citep{Liu2022d} for binarizing activations and incorporating \gls{KD} during training, where the \gls{FP} teacher guides the student in a layer-wise manner. We follow the experimental setup of \textsc{bit}, including using the same method for binarizing activations and backpropagation for softmax and attention in the \textsc{bert} model.
As shown in \cref{table:nlp}, all methods suffer from performance drop compared to the \gls{FP} model as extreme binarization of transformer-based model is not trivial.
Nevertheless, our method yields results comparable to \textsc{bit} \citep{Liu2022d}, the \gls{SOTA} method on this task, outperforming \textsc{binarybert} \citep{Bai2021} and \textsc{bibert} \citep{Qin2022} on average.
This is remarkable as our method natively uses Boolean weights during the training, whereas the baselines heavily rely on \gls{FP} latent weights.
These findings indicate potential for energy-efficient \glspl{LLM} using our method for both training and inference.

\begin{minipage}[][][t]{0.48\textwidth}
	\begin{figure}[H]
		\centering
		\includegraphics[width=\textwidth]{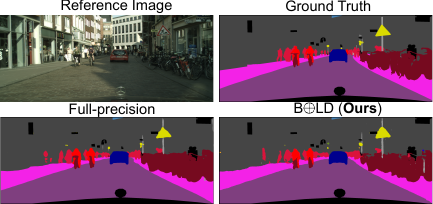}
		\caption{An example of \textsc{cityscapes}. \label{fig:main_cityscapes}}
	\end{figure}
\end{minipage}
\hfill
\begin{minipage}[][][t]{0.5\textwidth}
		\captionof{table}{\textsc{bert} models results. $^\dagger$Source code \citep{Liu2022d}.  \label{table:nlp}}
\resizebox{\columnwidth}{!}{
     \setlength\tabcolsep{1.3pt}
     \begin{tabular}{cccccccccc}  \toprule
          \multirow{2}{*}{\textbf{Method}}	&   \multicolumn{9}{c}{\textbf{GLUE Benchmark}  (Accuracy, $\uparrow$)} \\ \cmidrule{2-10}
                                                                                                                                                                                                                                            & \textsc{mnli} & \textsc{qqp} & \textsc{qnli} & \textsc{sst-2} & \textsc{cola} & \textsc{sst-b} & \textsc{mrpc} & \textsc{rte} & \textbf{Avg.} \\ \midrule\midrule
          \textsc{fp bert}	                                                                                                                                                                                                     & 84.9          & 91.4         & 92.1          & 93.2           & 59.7          & 90.1           & 86.3          & 72.2         & 83.9          \\
          \midrule
          \textsc{binarybert}                                                                                                                                                                                                      & 35.6          & 66.2         & 51.5          & 53.2           & 0.0           & 6.1            & 68.3          & 52.7         & 41.0          \\
          \textsc{bibert}                                                                                                                                                                                                           & 66.1          & 84.8         & 72.6          & 88.7           & 25.4          & 33.6           & 72.5          & 57.4         & 63.2          \\
          \textsc{bit} 	                                                                                                                                                                                                           & 77.1          & 82.9         & 85.7          & 87.7           & 25.1          & 71.1           & 79.7          & 58.8         & 71.0          \\
          \textsc{bit} (Reprod.$^\dagger$)  %
& 76.8          & 87.2         & 85.6          & 87.5           & 24.1          & 70.5           & 78.9          & 58.8         & 69.7          \\
          \tikzexternaldisable ({\protect\tikz[baseline=-.65ex]\protect\draw[thick, color=color_booldl, fill=color_booldl, mark=*, mark size=2pt, line width=1.25pt] plot[] (-.0, 0)--(-0,0);}) \tikzexternalenable \ours & 75.6          & 85.9         & 84.1          & 88.7           & 27.1          & 68.7           & 78.4          & 58.8         & 70.9          \\ \bottomrule
     \end{tabular}}

\end{minipage}

\section{Conclusions}\label{sec:Conclusions}

We introduced the notion of Boolean variation and developed a first framework of its calculus.
This novel mathematical principle enabled the development of Boolean logic backpropagation and Boolean optimization replacing gradient backpropagation and gradient descent for binary deep learning.
Deep models can be built with native Boolean weights and/or Boolean activations, and trained in Boolean natively by this principled exact Boolean optimization.
That brings a key advantage to the existing popular quantized/binarized training approach that suffers from critical bottlenecks -- \textit{(i)} performance loss due to an arbitrary approximation of the latent weight gradient through its discretization/binarization function, \textit{(ii)} training computational intensiveness due to full-precision latent weights. 
We have extensively explored its capabilities, highlighting:
\textit{(i)} both training and inference are now possible in binary;
\textit{(iv)} deep training complexity can be drastically reduced to unprecedented levels.
\textit{(iii)} Boolean models can handle finer tasks beyond classification, contrary to common belief;
\textit{(ii)} in some applications, suitablely enlarging Boolean model can recover \gls{FP} performance while still gaining significant complexity reduction.

\paragraph{Limitations.}
Due to current computing accelerators, such as GPUs, primarily designed for real arithmetic, our method could not be assessed on native Boolean accelerator.
Nevertheless, its considerable potential may inspire the development of new logic circuits and architectures utilizing Boolean logic processing. 
It also remains as an open question the approximation capacity of Boolean neural networks. A mathematical result equivent to the existing universal approximation theory of real-valued neural networks would provide a solid guarantee.

\clearpage
\newpage
\section*{Acknowledgments}

This work has been made possible through the invaluable support of various departments and colleagues within Huawei. Van Minh Nguyen would like to extend his sincere thanks to everyone involved, with special appreciation to Jean-Claude Belfiore and the former students who have participated in this project over the years: Valentin Abadie, Tom Huix, Tianyu Li, and Youssef Chaabouni.

{
\small
\bibliographystyle{abbrv}
\bibliography{IEEEabrv,bibliography.bib}
}

\clearpage
\newpage
\appendix

{\bf \Large{Appendix}}

\begin{spacing}{0.2}
    \addtocontents{toc}{\protect\setcounter{tocdepth}{2}} %
    \renewcommand{\contentsname}{\textbf{Table of Contents}\vskip3pt\hrule}
    \tableofcontents
    \vskip5pt\hrule
\end{spacing}

\section{Details of the Boolean Variation Method}
\label{appendix:booleanvariation}

\subsection{Boolean Variation Calculus}\label{appendix:Boolean}
This section provides supplementary details of the Boolean Variation method, a sole development can be found in \citep{Nguyen2024}. 
With reference to \cref{ex:XORVariation}, \cref{tab:VariationXOR} gives the variation truth table of $f(x) = \xor(a,x)$, for $a, x \in \Bb$.

{\renewcommand{\arraystretch}{1.0}
		\begin{table}[H]
			\centering
			\caption{Variation truth table of $f(x) = \xor(a,x)$, $a, x \in \Bb$.}
			\label{tab:VariationXOR}
			\vspace{1ex}
			\begin{tabular}{cccccccc}
				\toprule %
				\multirow{2}{*}{$a$} 
				& \multirow{2}{*}{$x$} 
				& \multirow{2}{*}{$\neg x$} 
				& \multirow{2}{*}{$\bvar(x \to \neg x)$} 
				& \multirow{2}{*}{$f(a,x)$}
				& \multirow{2}{*}{$f(a,\neg x)$}
				& \multirow{2}{*}{$\bvar f(x \to \neg x)$} 
				& \multirow{2}{*}{$f'(x)$}   \\
				& & & & & \\
				\midrule %
				$\True$  &  $\True$  &  $\False$  & $\False$  &  $\False$ 	& $\True$	& $\True$  & $\False$  \\
				
				$\True$  &  $\False$ &  $\True$   & $\True$   &  $\True$	& $\False$ 	& $\False$ & $\False$  \\
				
				$\False$ &  $\True$  &  $\False$  & $\False$  &  $\True$	& $\False$ 	& $\False$ & $\True$   \\
				
				$\False$ &  $\False$ &  $\True$   & $\True$   &  $\False$ 	& $\True$ 	& $\True$  & $\True$   \\
				\bottomrule
			\end{tabular}
		\end{table}
	}

We now proceed to the proof of \cref{thm:MainRules}. 

\begin{definition}[Type conversion]\label{def:Conversion}
	Define:
	\begin{align}
		\proj \colon & \Nb \to \Lb \nonumber\\
		& x \mapsto \proj(x) = \begin{cases}
			\True, & \textrm{if } x > 0,\\
			0, & \textrm{if } x = 0, \\
			\False, & \textrm{if } x < 0.
		\end{cases}\label{eq:Projector}
	\end{align}
	\begin{align}
		\emb \colon & \Lb \to \Nb \nonumber\\
		& a \mapsto \emb(a) = \begin{cases}
			+1, & \textrm{if } a = \True,\\
			0, & \textrm{if } a = 0, \\
			-1, & \textrm{if } a = \False.
		\end{cases}\label{eq:Embedding}
	\end{align}
\end{definition}
$\proj$ projects a numeric type in logic, and $\emb$ embeds a logic type in numeric. The following properties are straightforward:

\begin{proposition}\label{prop:Conversion}
	The following properties hold:
	\begin{enumerate}
		\item  $\forall x, y \in \Nb$: $\proj(xy) = \xnor(\proj(x), \proj(y))$.
		\item  $\forall a, b \in \Lb$: $\emb(\xnor(a,b)) = \emb(a)\emb(b)$.
		\item  $\forall x, y \in \Nb$: $x = y \Leftrightarrow |x| = |y| \textrm{ and } \proj(x) = \proj(y)$.
	\end{enumerate}
\end{proposition}
In particular, property \cref{prop:Conversion}(2) implies that by the embedding map $\emb(\cdot)$, we have:
\begin{align}
	(\{\True, \False\}, \xor) & \cong (\{\pm 1\}, -\times),\\
	(\{\True, \False\}, \xnor) & \cong (\{\pm 1\}, \times),
\end{align}
where $\cong$ and $\times$ stand for isomorphic relation, and the real multiplication, resp. A consequence is that by $\emb(\cdot)$, a computing sequence of pointwise XOR/XNOR, counting, and majority vote is equivalent to a sequence of pointwise multiplications and accumulation performed on the embedded data. This property will be used in \cref{sec:cvg,appendix:train_regu} for studying Boolean method using some results from \glspl{BNN} literature and real analysis.

\begin{proposition}\label{prop:XNORAlgebra}
	The following properties hold:
	\begin{enumerate}
		\item $a \in \Lb$, $x \in \Nb$: $\xnor(a,x) = \emb(a)x$.
		\item $x, y \in \Nb$: $\xnor(x,y) = xy$. %
		\item $x \in \{\Lb, \Nb\}$, $y, z \in \Nb$: $\xnor(x,y+z) = \xnor(x,y) + \xnor(x,z)$. %
		\item $x \in \{\Lb, \Nb\}$, $y, \lambda \in \Nb$: $\xnor(x, \lambda y) = \lambda \xnor(x,y)$.
		\item $x \in \{\Lb, \Nb\}$, $y \in \Nb$: $\xor(x,y) = -\xnor(x,y)$.
	\end{enumerate}
\end{proposition}
\begin{proof} 
	The proof follows definitions \ref{def:MixedLogic} and \ref{def:Conversion}.
	\begin{itemize}
		\item Following \cref{def:ThreeValueLogic} we have $\forall t \in \Mb$, $\xnor(\True, t) = t$, $\xnor(\False, t) = \neg t$, and $\xnor(0, t) = 0$. Put $v = \xnor(a, x)$. We have $|v| = |x|$ and $\proj(v) = \xnor(a, \proj(x))$. Hence, $a = 0 \Rightarrow \proj(v) = 0 \Rightarrow v = 0$; $a = \True \Rightarrow \proj(v) = \proj(x) \Rightarrow v = x$; $a = \False \Rightarrow \proj(v) = \neg \proj(x) \Rightarrow v = -x$. Hence (1).
		\item The result is trivial if $x=0$ or $y=0$. For $x, y \neq 0$, put $v = \xnor(x,y)$, we have $|v| = |x||y|$ and $\proj(v) = \xnor(\proj(x), \proj(y))$. According to \cref{def:Conversion}, if $\sign(x)=\sign(y)$, we have $\proj(v) = \True \Rightarrow v = |x||y| = xy$. Otherwise, i.e., $\sign(x)=-\sign(y)$, $\proj(v) = \False \Rightarrow v = -|x||y| = xy$. Hence (2).
		\item (3) and (4) follow (1) for $x \in \Lb$ and follow (2) for $x \in \Nb$.
		\item For (5), write $u = \xor(x,y)$ and $v = \xnor(x,y)$, we have $|u| = |v|$ and $\proj(u) = \xor(\proj(x), \proj(y)) = \neg \xnor(\proj(x), \proj(y)) = \neg \proj(v)$. Thus, $\sign(u) = -\sign(v) \Rightarrow u = -v$. \qedhere
	\end{itemize}
\end{proof}

\begin{proposition}\label{prop:B2BVariation}
	For $f, g \in \Fc(\Bb, \Bb)$, $\forall x, y \in \Bb$ the following properties hold:
	\begin{enumerate}
		\item $\bvar f(x \to y) = \xnor(\bvar(x \to y), f'(x)).$
		\item $(\neg f)'(x) = \neg f'(x)$.
		\item $(g \circ f)'(x) = \xnor(g'(f(x)), f'(x))$.
	\end{enumerate}
\end{proposition}
\begin{proof} The proof is by definition:
	\begin{enumerate}
		\item $\forall x, y \in \Bb$, there are two cases. If $y = x$, then the result is trivial. Otherwise, i.e., $y = \neg x$, by definition we have:
		\begin{align*}
			f'(x) & = \xnor(\bvar(x \to \neg x), \bvar f(x \to \neg x)) \\
			\Leftrightarrow \quad \bvar f(x \to \neg x) & = \xnor(\bvar(x \to \neg x), f'(x)).
		\end{align*}
		Hence the result.
		\item $\forall x, y \in \Bb$, it is easy to verify by truth table that $\bvar(\neg f(x \to y)) = \neg \bvar{f(x \to y)}$. Hence, by definition,
		\begin{align*}
			(\neg f)'(x) & = \xnor(\bvar(x \to \neg x), \bvar(\neg f(x \to \neg x))) \\
			& = \xnor(\bvar(x \to \neg x), \neg \bvar f(x \to \neg x)) \\
			& = \neg \xnor(\bvar(x \to \neg x), \bvar f(x \to \neg x)) \\
			& = \neg f'(x).
		\end{align*}
		\item Using definition, property (i), and associativity of $\xnor$, $\forall x \in \Bb$ we have:
		\begin{align*}
			(g\circ f)'(x) & = \xnor(\bvar(x \to \neg x), \bvar g(f(x) \to f(\neg x))) \\
			& = \xnor\pren{\bvar(x \to \neg x), \xnor\pren{\bvar f(x \to \neg x), g'\pren{f(x)} }} \\
			& = \xnor\pren{g'(f(x)), \xnor\pren{\bvar(x \to \neg x), \bvar f(x \to \neg x) } } \\
			& = \xnor(g'(f(x)), f'(x) ).
		\end{align*} \qedhere
	\end{enumerate}
\end{proof}

\begin{proposition}\label{prop:B2NVariation}
	For $f \in \Fc(\Bb, \Nb)$, the following properties hold:
	\begin{enumerate}
		\item $x, y \in \Bb$: $\bvar f(x \to y) = \xnor(\bvar(x \to y), f'(x))$.
		\item $x \in \Bb$, $\alpha \in \Nb$: $(\alpha f)'(x) = \alpha f'(x)$.
		\item $x \in \Bb$, $g \in \Fc(\Bb, \Nb)$: $(f + g)'(x) = f'(x) + g'(x)$.
	\end{enumerate}
\end{proposition}
\begin{proof}
	The proof is as follows:
	\begin{enumerate}
	\item For $x, y \in \Bb$. Firstly, the result is trivial if $y = x$. For $y \neq x$, i.e., $y = \neg x$, by definition:
		\begin{equation*}
			f'(x) = \xnor(\bvar(x \to \neg x), \bvar f(x \to \neg x)).
		\end{equation*}
		Hence, $|\bvar f(x \to \neg x)| = |f'(x)|$ since $|\bvar(x \to \neg x)| = 1$, and
		\begin{align*}
			\proj(f'(x)) & = \xnor(\bvar(x \to \neg x), \proj(\bvar f(x \to \neg x)))\\
			\Leftrightarrow \quad \proj(\bvar f(x \to \neg x)) & = \xnor(\bvar(x \to \neg x), \proj(f'(x))),
		\end{align*}
		where $\proj(\cdot)$ is the logic projector \cref{eq:Projector}. Thus, $\bvar f(x \to \neg x) = \xnor(\bvar(x \to \neg x), f'(x))$. Hence the result.
	\item Firstly $\forall x, y \in \Bb$, we have
		\begin{equation*}
			\bvar(\alpha f(x \to y)) = \alpha f(y) - \alpha f(x) = \alpha \bvar{f(x \to y)}.
		\end{equation*}
		Hence, by definition,
		\begin{align*}
			(\alpha f)'(x) & = \xnor(\bvar(x \to \neg x), \bvar(\alpha f(x \to \neg x))) \\
			& = \xnor(\bvar(x \to \neg x), \alpha\bvar{f(x \to \neg x)}) \\
			& = \alpha \, \xnor(\bvar(x \to \neg x), \bvar{f(x \to \neg x)}), \textrm{ due to \cref{prop:XNORAlgebra}(4)}\\
			& = \alpha f'(x).
		\end{align*}
	\item For $f, g \in \Fc(\Bb, \Nb)$,
		\begin{align*}
			(f+g)'(x) & = \xnor(\bvar(x \to \neg x), \bvar(f+g)(x \to \neg x))\\
			& = \xnor(\bvar(x \to \neg x), \bvar f(x \to \neg x) + \bvar g(x \to \neg x)) \\
			& \overset{(*)}{=} \xnor(\bvar(x \to \neg x), \bvar f(x \to \neg x)) + \xnor(\bvar(x \to \neg x), \bvar g(x \to \neg x)),\\
			& = f'(x) + g'(x),
		\end{align*}
		where $(*)$ is due to \cref{prop:XNORAlgebra}(3). \qedhere
	\end{enumerate}
\end{proof}

\begin{proposition}[Composition rules]\label{prop:Composition}
	The following properties hold:
	\begin{enumerate}
		\item For $\Bb \overset{f}{\to} \Bb \overset{g}{\to} \Db$: $(g \circ f)'(x) = \xnor(g'(f(x)), f'(x))$, $\forall x \in \Bb$.
		\item For $\Bb \overset{f}{\to} \Zb \overset{g}{\to} \Db$, $x \in \Bb$, if $|f'(x)| \leq 1$ and $g'(f(x)) =g'(f(x)-1)$, then:
		\begin{equation*}
			(g \circ f)'(x) = \xnor(g'(f(x)), f'(x)).
		\end{equation*}
	\end{enumerate}
\end{proposition}
\begin{proof}
	The proof is as follows.
	\begin{enumerate}
		\item The case of $\Bb \overset{f}{\to} \Bb \overset{g}{\to} \Bb$ is obtained from \cref{prop:B2BVariation}(3). For $\Bb \overset{f}{\to} \Bb \overset{g}{\to} \Nb$, by using \cref{prop:B2NVariation}(1), the proof is similar to that of \cref{prop:B2BVariation}(3).
		\item By definition, we have
		\begin{equation}\label{eq:proofComposition2_1}
			(g \circ f)'(x) = \xnor(\bvar(x \to \neg x), \bvar g(f(x) \to f(\neg x))). 
		\end{equation} 
		Using property (1) of \cref{prop:B2NVariation}, we have:
		\begin{align}
			f(\neg x) & = f(x) + \bvar f(x \to \neg x) \nonumber\\
			& = f(x) + \xnor(\bvar(x \to \neg x), f'(x)). \label{eq:proofComposition2_2}
		\end{align}
		Applying \cref{eq:proofComposition2_2} back to \cref{eq:proofComposition2_1}, the result is trivial if $f'(x) = 0$. The remaining case is $|f'(x)| = 1$ for which we have $\xnor(\bvar(x \to \neg x), f'(x)) = \pm 1$. First, for $\xnor(\bvar(x \to \neg x), f'(x)) = 1$, we have:
		\begin{align}
			\bvar g(f(x) \to f(\neg x)) &= \bvar g(f(x) \to f(x) + 1) \nonumber\\
			& = g'(f(x)) \nonumber\\
			& = \xnor(g'(f(x)), 1) \nonumber\\
			& = \xnor(g'(f(x)), \xnor(\bvar(x \to \neg x), f'(x)) ) \label{eq:proofComposition2_3}.
		\end{align}
		Substitute \cref{eq:proofComposition2_3} back to \cref{eq:proofComposition2_1}, we obtain:
		\begin{align*}
			(g \circ f)'(x) & = \xnor(\bvar(x \to \neg x), \bvar g(f(x) \to f(\neg x))) \\ 
			& = \xnor(\bvar(x \to \neg x), \xnor(g'(f(x)), \xnor(\bvar(x \to \neg x), f'(x)) ) ) \\
			& = \xnor(g'(f(x)), f'(x)),
		\end{align*} 
		where that last equality is by the associativity of $\xnor$ and that $\xnor(x, x) = \True$ for $x \in \Bb$. 
		Similarly, for $\xnor(\bvar(x \to \neg x), f'(x)) = -1$, we have:
		\begin{align}
			\bvar g(f(x) \to f(\neg x)) &= \bvar g(f(x) \to f(x) - 1) \nonumber\\
			& = - g'(f(x)-1) \nonumber\\
			& = \xnor(g'(f(x)-1), -1) \nonumber\\
			& = \xnor(g'(f(x)-1), \xnor(\bvar(x \to \neg x), f'(x)) ) \label{eq:proofComposition2_4}.
		\end{align}
		Substitute \cref{eq:proofComposition2_4} back to \cref{eq:proofComposition2_1} and use the assumption that $g'(f(x)) = g'(f(x)-1)$, we have:
		\begin{align*}
			(g \circ f)'(x) & = \xnor(\bvar(x \to \neg x), \bvar g(f(x) \to f(\neg x))) \\ 
			& = \xnor(\bvar(x \to \neg x), \xnor(g'(f(x)-1), \xnor(\bvar(x \to \neg x), f'(x)) ) ) \\
			& = \xnor(g'(f(x)), f'(x)).
		\end{align*} 
		Hence the result. \qedhere
	\end{enumerate}
\end{proof}

The results of \cref{thm:MainRules} are from \cref{prop:B2BVariation,prop:B2NVariation,prop:Composition}.

\subsection{Convergence Proof}\label{sec:cvg}
To the best of our knowledge, in the quantized neural network literature and in particular \gls{BNN}, one can only prove the convergence up to an irreducible error floor \cite{Li2017}. This idea has been extended to SVRG \cite{de2018high}, and recently to SGLD in \cite{Zhang2022b}, which is also up to an error limit.

In this section we provide complexity bounds for Boolean Logic in a smooth non-convex environment. We introduce an abstraction to model its optimization process and prove its convergence.

\subsubsection{Continuous Abstraction}\label{sec:analysis-prel}
Boolean optimizer is discrete, proving its convergence directly is a hard problem. The idea is to find a continuous equivalence so that some proof techniques existing from the \gls{BNN} and quantized neural networks literature can be employed.
We also abstract the logic optimization rules as compressor $Q_0(), Q_1()$, and define gradient accumulator $a_t$ as $a_{t+1} = a_t + \varphi(q_t)$. When $\eta$ is constant, we recover the definition \eref{eq:Accum} and obtain $m_t= \eta a_t$. Our analysis is based on the following standard non-convex assumptions on $f$:
\begin{assumption} %
	\label{assum:uniflowerbound}
	Uniform Lower Bound: There exists $f_* \in \mathbb{R}$ s.t. $f(w) \geq f_*$, $\forall w \in \mathbb{R}^d$.
\end{assumption}
\begin{assumption} %
	\label{assum:smoothgradients}
	Smooth Derivatives: The gradient $\nabla f(w)$ is $L$-Lipschitz continuous for some $L>0$, i.e., $\forall w, \forall v \in \mathbb{R}^d$:
	$
		\left\|\nabla f(w)-\nabla f(v)\right\| \leq L\|w-v\| .
	$
\end{assumption}
\begin{assumption} %
	\label{assum:boundedvariance}
	Bounded Variance: The variance of the stochastic gradients is bounded by some $\sigma^2>0$, i.e., $\forall w \in \mathbb{R}^d$:
	$
		\E{\Tilde{\nabla} f(w)} = \nabla f(w)
	$ and $
		\E{\|\Tilde{\nabla} f(w)\|^2} \leq \sigma^2 .
	$
\end{assumption}
\begin{assumption}
	\label{assum:alwaysflip}
	Compressor: There exists $\gamma<1$ s.t. $\forall w, \forall v \in \mathbb{R}^d$, $\Vert Q_1(v,w)-v \Vert^2 \leq \gamma \Vert v \Vert^2$.
\end{assumption}
\begin{assumption}
	\label{assum:boundedaccumulator}
	Bounded Accumulator: There exists $\kappa \in \mathbb{R}^*_+$ s.t. $\forall t$ and $\forall i \in [d]$, we have $|a_t|_i \le \kappa$.
\end{assumption}
\begin{assumption}
	\label{assum:stoflip}
	Stochastic Flipping Rule: For all $w \in \mathbb{R}^d$, we have $\E{Q_0(w) | w}=w$.
\end{assumption}

In existing frameworks, quantity $\Tilde{\nabla} f(\cdot)$ denotes the stochastic gradient computed on a random mini-batch of data. Boolean framework does not have the notion of gradient, it however has an optimization signal (i.e., $\varphi$(q) or its accumulator $m_t$ \eref{eq:Accum}) that plays the same role as $\Tilde{\nabla} f(\cdot)$. Therefore, these two notions, i.e., continuous gradient and Boolean optimization signal, can be encompassed into a generalized notion. That is the root to the following continuous relaxation in which $\Tilde{\nabla} f(\cdot)$ standards for the optimization signal computed on a random mini-batch of data. 

Within this continuous relaxation framework, \cref{assum:stoflip} expresses our assumption that the flipping rule is stochastic and unbiased. Note that this assumption is standard in the literature related to (stochastic) quantization, see e.g., \citep{alistarh2017qsgd,Polyak1987,Wangni2018}.

For reference, the original Boolean optimizer as formulated in \cref{sec:Method} is summarized in \Cref{alg:BoolOptim} in which $\texttt{flip}(w_t, m_{t+1})$ flips weight and $\texttt{reset}(w_t, m_{t+1})$ resets its accumulator when the flip condition is triggered.
\begin{algorithm}[h]
	\caption{Boolean optimizer}\label{alg:BoolOptim}
	$m_{t+1} \gets \beta_{t} m_{t} + \eta \varphi(q_t)$ \;
	$w_{t+1} \gets \texttt{flip}(w_t, m_{t+1})$\;
	$m_{t+1} \gets \texttt{reset}(w_t, m_{t+1})$\;
\end{algorithm}

\begin{algorithm}[h]
	\caption{Equivalent formulation of Boolean optimizer}\label{alg:EquivOptim}
	\KwData{$Q_0, Q_1$ quantizer}
	$m_t \gets \eta \Tilde{\nabla} f(w_t) + e_t $\;
	$\Delta_t \gets Q_1(m_t, w_t)$\;
	$w_{t+1} \gets Q_0(w_t - \Delta_t)$\;
	$e_{t+1} \gets m_t - \Delta_t$\;
\end{algorithm}

\Cref{alg:EquivOptim} describes an equivalent formulation of Boolean optimizer. Therein, $Q_0$, $Q_1$ are quantizers which are specified in the following. Note that EF-SIGNSGD (SIGNSGD with Error-Feedback) algorithm from \cite{errorfeedback} is a particular case of this formulation with $Q_0()=\operatorname{Identity}()$ and $Q_1()=\sign()$.
For Boolean Logic abstraction, they are given by:

\begin{equation}
\label{eq:quantizers}
	\begin{cases}
		Q_1(m_t, w_t) = w_t (\textrm{ReLu}(w_t m_t-1)+\frac{1}{2} \sign(w_t m_t-1) +\frac{1}{2}), \\
		Q_0(w_t) = \sign(w_t).
	\end{cases}
\end{equation}

The combination of $Q_1$ and $Q_0$ is crucial to take into account the reset property of the accumulator $m_t$. Indeed in practice, $\Delta_t:=Q_1({m_t}, w_t)$ is always equal to $0$ except when $| m_t | > 1$ and $\sign(m_t)=\sign(w_t)$ (i.e., when the flipping rule is applied). As $w_t$ has only values in $\{\pm 1\}$, $Q_0$ acts as identity function, except when $\Delta_t$ is non-zero (i.e., when the flipping rule is applied). With the choices \eqref{eq:quantizers}, we can identify $\texttt{flip}(w_t, m_{t}) = Q_0(w_t - Q_1(m_t, w_t))$. We do not have closed-form formula for $\texttt{reset}(w_t, m_{t+1})$ from \Cref{alg:BoolOptim}, but the residual errors $e_t$ play this role. Indeed, $e_{t+1} = m_t$ except when $\Delta_t$ is non-zero (i.e., when the flipping rule is applied and $e_{t+1}$ is equal to $0$).

The main difficulty in the analysis comes from the parameters quantization $Q_0()$. Indeed, we can follow the derivations in Appendix B.3 from \cite{errorfeedback} to bound the error term $\E{\Vert e_t \Vert^2}$, but we also have additional terms coming from the quantity:
\begin{equation}\label{eq:ht}
	h_t = Q_0(w_t - Q_1(m_t, w_t)) - (w_t - Q_1(m_t, w_t)).
\end{equation}
As a consequence, assumptions \ref{assum:uniflowerbound} to \ref{assum:stoflip} enable us to obtain $\E{h_t} = 0$ and to bound the variance of $h_t$.

\begin{remark}
	Assumptions \ref{assum:uniflowerbound} to \ref{assum:boundedvariance} are standard. Assumptions \ref{assum:alwaysflip} to \ref{assum:stoflip} are non-classic but dedicated to Boolean Logic strategy. \Cref{assum:alwaysflip} is equivalent to assuming Boolean Logic optimization presents at least one flip at every iteration $t$. \Cref{assum:alwaysflip} is classic in the literature of compressed SGD \cite{errorfeedback, alistarh2017qsgd, philippenko2020bidirectional}. Moreover, \Cref{assum:boundedaccumulator} and \Cref{assum:stoflip} are not restrictive, but algorithmic choices. For example, rounding ($Q_0$ function) can be stochastic based on the value of the accumulator $m_t$. Similar to STE clipping strategy, the accumulator can be clipped to some pre-defined value $\kappa$ before applying the flipping rule to verify \Cref{assum:boundedaccumulator}.
\end{remark}

\begin{remark}
	Our proof assumes that the step size $\eta$ is constant over iterations. But in practice, we gently decrease the value of $\eta$ at some time steps. Our proof can  be adapted to this setting by defining a gradient accumulator $a_t$ such that $a_{t+1}=a_t+  \varphi(q_t)$. When $\eta$ is constant we recover the definition \eref{eq:Accum} and we obtain $m_t= \eta a_t$. In the proposed algorithm, gradients are computed on binary weight $w_t$ and accumulated in $a_t$. Then, one applies the flipping rule on the quantity $\Tilde{w}_t = \eta a_t$ ($\Tilde{w}_t=m_t$ when $\eta$ is constant), and one (may) reset the accumulator $a_t$.
\end{remark}

We start by stating a key lemma which shows that the residual errors $e_t$ maintained in \Cref{alg:EquivOptim} do not accumulate too much.

\begin{lemma}
    \label{lem:boundederror}
    Under \Cref{assum:boundedvariance} and \Cref{assum:alwaysflip}, the error can be bounded as $\E{\Vert e_t \Vert^2} \leq \frac{2\gamma}{(1-\gamma)^2}\eta^2\sigma^2$.
\end{lemma}

\begin{proof}
	We start by using the definition of the error sequence:
	\begin{align*}
		\Vert e_{t+1} \Vert^2 = \Vert Q_1(m_t, w_t)- m_t \Vert^2.
	\end{align*}
	Next we make use of \Cref{assum:alwaysflip}:
	\begin{align*}
		\Vert e_{t+1} \Vert^2 \leq \gamma \Vert m_t \Vert^2.
	\end{align*}
	We develop the accumulator update:
	\begin{align*}
		\Vert e_{t+1} \Vert^2 \leq \gamma \Vert e_t + \eta \Tilde{\nabla} f(w_t) \Vert^2.
	\end{align*}
	We thus have a recurrence relation on the bound of $e_t$. Using Young’s inequality, we have that for any $\beta>0$,
	\begin{align*}
		\Vert e_{t+1} \Vert^2 \leq \gamma (1+\beta) \Vert e_t \Vert^2 + \gamma (1+\frac{1}{\beta}) \eta^2 \Vert \Tilde{\nabla} f(w_t) \Vert^2.
	\end{align*}
	Rolling the recursion over and using \Cref{assum:boundedvariance} we obtain:
	\begin{align*}
		\E{\Vert e_{t+1} \Vert^2} \leq & \gamma (1+\beta) \E{\Vert e_t \Vert^2} + \gamma (1+\frac{1}{\beta}) \eta^2 \E{\Vert \Tilde{\nabla} f(w_t) \Vert^2} \\
		\leq                            & \gamma (1+\beta) \E{\Vert e_t \Vert^2} + \gamma (1+\frac{1}{\beta}) \eta^2 \sigma^2                                      \\
		\leq                            & \sum_r^t (\gamma (1+\beta))^r \gamma (1+\frac{1}{\beta}) \eta^2 \sigma^2                                                  \\
		\leq                            & \frac{\gamma(1+\frac{1}{\beta})}{1-\gamma(1+\beta)} \eta^2 \sigma^2.
	\end{align*}
	Take $\beta = \frac{1-\gamma}{2\gamma}$ and plug it in the above bounds gives:
	\begin{align*}
		\E{\Vert e_{t+1} \Vert^2} \leq \frac{2\gamma}{(1-\gamma)^2}\eta^2\sigma^2.
	\end{align*}
\end{proof}

Then, the next Lemma allows us to bound the averaged norm-squared of the distance between the Boolean weight and $w_t - Q_1(m_t, w_t)$. We make use of the previously defined quantity $h_t$ \eref{eq:ht} and have:
\begin{lemma}
    \label{lem:boundedht}
    Under assumptions \Cref{assum:boundedaccumulator} and \Cref{assum:stoflip}: $\E{\Vert h_t \Vert^2} \leq \eta d \kappa$.
\end{lemma}

\begin{proof}
	Let consider a coordinate $i \in [d]$. $Q_0|_i$ as $-1$ or $+1$ for value with some probability $p_{i,t}$. For the ease of presentation, we will drop the subscript $i$. Denote $u_t := w_t - Q_1(m_t, w_t)$. Hence, $h_t$ can take value $(1-u_t)$ with some probability $p_t$ and $(-1-u_t)$ with probability $1-p_t$. Assumption \Cref{assum:stoflip} yields $2p_t-1=u_t$. Therefore, we can compute the variance of $h_t$ as follows:
	\begin{align*}
		\E{\Vert h_t \Vert^2} & = \E{\sum_i^d 1 + (w_t - Q_1(m_t, w_t))^2 -2Q_0(w_t - Q_1(m_t, w_t)(w_t - Q_1(m_t, w_t)} \\
		& = \sum_i^d ((1-u_t)^2p_t + (-1-u_t)^2(1-p_t)) \\
		& = \sum_i^d (1-u_t^2).
	\end{align*}
	The definition of $u_t$ leads to
	\begin{align*}
		1-u_t^2 & = 1-(1+Q_1(m_t, w_t)^2-2w_t Q_1(m_t, w_t)) \\
		& =Q_1(m_t, w_t)(2w_t-Q_1(m_t, w_t)).
	\end{align*}
	When $m_t < 1$, we directly have $Q_1(m_t, w_t)(2w_t-Q_1(m_t, w_t)) = 0 \leq \eta \kappa$. When $m_t \geq 1$, we apply the definition of $Q_1$ to obtain:
	\begin{align*}
		Q_1(m_t, w_t)(2w_t-Q_1(m_t, w_t)) & \leq m_t(2-m_t) \\
		& \leq \eta \kappa.
	\end{align*}
	Therefore, we can apply this result to every coordinate, and conclude that:
	\begin{equation}
		\E{\Vert h_t \Vert^2} \leq \eta d \kappa.
	\end{equation}
\end{proof}

\subsubsection{Proof of \cref{thm:Convegence}}

We now can proceed to the proof of \thmref{thm:Convegence}.

\begin{proof}
	Consider the virtual sequence $x_t=w_t-e_t$. We have:
	\begin{align*}
		x_{t+1} & = Q_0(w_t-\Delta_t) - (m_t-\Delta_t) \\
		& = (Q_0(w_t-\Delta_t) + \Delta_t - e_t) - \eta \Tilde{\nabla} f(w_t).
	\end{align*}
	Considering the expectation with respect to the random variable $Q_0$ and the gradient noise, we have:
	\begin{align*}
		\E{x_{t+1}|w_t} = x_t - \eta \nabla f (w_t).
	\end{align*}
	We consider $\Et{\cdot}$ the expectation with respect to every random process know up to time $t$. We apply the $L$-smoothness assumption \Cref{assum:smoothgradients}, and assumptions \Cref{assum:boundedvariance}, \Cref{assum:stoflip} to obtain:
	\begin{align*}
		\Et{f(x_{t+1})-f(x_t)} & \leq -\eta \langle \nabla f (x_t), \nabla f (w_t) \rangle + \frac{L}{2} \Et{\Vert (Q_0(w_t-\Delta_t) + \Delta_t) - \eta \Tilde{\nabla} f(w_t) - w_t \Vert^2}.
	\end{align*}
	We now reuse $h_t$ from \eref{eq:ht} and simplify the above:
	\begin{align*}
		\Et{f(x_{t+1})-f(x_t)} & \leq -\eta \langle \nabla f (x_t), \nabla f (w_t) \rangle + \frac{L}{2} \Et{\Vert h_t - \eta \Tilde{\nabla} f(w_t) \Vert^2} \\
		& \leq -\eta \langle \nabla f (x_t)-\nabla f (w_t) + \nabla f (w_t), \nabla f (w_t) \rangle + \frac{L}{2} \Et{\Vert h_t - \eta \Tilde{\nabla} f(w_t) \Vert^2}.
	\end{align*}
	Using Young’s inequality, we have that for any $\beta > 0$,
	\begin{align*}
		\Et{f(x_{t+1})-f(x_t)} \leq & -\eta \langle \nabla f (x_t)-\nabla f (w_t) + \nabla f (w_t), \nabla f (w_t) \rangle \\
		& + \frac{L}{2}(1+\beta) \Et{\Vert h_t\Vert^2} + \frac{L}{2}\eta^2(1+\frac{1}{\beta})\sigma^2.
	\end{align*}
	Making use again of smoothness and Young's inequality we have:
	\begin{align*}
		\Et{f(x_{t+1})-f(x_t)} \leq & -\eta \Vert \nabla f (w_t)\Vert^2 -\eta \langle \nabla f (x_t)-\nabla f (w_t), \nabla f (w_t) \rangle \\
		& + \frac{L}{2}(1+\beta) \Et{\Vert h_t\Vert^2} + \frac{L}{2}\eta^2(1+\frac{1}{\beta})\sigma^2 \\
		\leq & -\eta \Vert \nabla f (w_t)\Vert^2 + \frac{\eta \rho}{2} \Vert \nabla f (w_t) \Vert^2 + \frac{\eta}{2\rho} \Vert \nabla f(x_t) - \nabla f(w_t) \Vert^2 \\
		& + \frac{L}{2}(1+\beta) \Et{\Vert h_t\Vert^2} + \frac{L}{2}\eta^2(1+\frac{1}{\beta})\sigma^2 \\
		\leq & -\eta \Vert \nabla f (w_t)\Vert^2 + \frac{\eta \rho}{2} \Vert \nabla f (w_t) \Vert^2 + \frac{\eta L^2}{2\rho} \underbrace{\Vert x_t - w_t \Vert^2}_{\Vert e_t \Vert^2} \\
		& + \frac{L}{2}(1+\beta) \Et{\Vert h_t\Vert^2} + \frac{L}{2}\eta^2(1+\frac{1}{\beta})\sigma^2.
	\end{align*}
	Under the law of total expectation, we make use of \Cref{lem:boundederror} and \Cref{lem:boundedht} to obtain:
	\begin{align*}
		\E{f(x_{t+1})} - \E{f(x_t)} \leq & -\eta(1-\frac{\rho}{2}) \E{\Vert \nabla f (w_t)\Vert^2} + \frac{\eta L^2}{2\rho} \frac{4\gamma}{(1-\gamma)^2}\eta^2\sigma^2 \\
		& + \frac{L}{8}\eta (1+\beta) d\kappa + \frac{L}{2}\eta^2(1+\frac{1}{\beta})\sigma^2.
	\end{align*}
	Rearranging the terms and averaging over $t$ gives for $\rho < 2$ (we can choose for instance $\rho = \beta = 1$):
	\begin{align*}
		\frac{1}{T+1} \sum_{t=0}^T \Vert \nabla f(w_t) \Vert^2 \leq \frac{2(f(w_0)-f_*)}{\eta (T+1)} + 2L\sigma^2 \eta + 4L^2\sigma^2 \frac{\gamma}{(1-\gamma)^2} \eta^2 + \frac{L}{2}d\kappa.
	\end{align*}
\end{proof}

The bound in \Cref{thm:Convegence} contains 4 terms. The first term is standard for a general non-convex target and expresses how initialization affects convergence. The second and third terms depend on the fluctuation of the minibatch gradients. Another important aspect of the rate determined by \Cref{thm:Convegence} is its dependence on the quantization error. Note that there is an ``error bound'' of $\frac{L}{2}d\kappa$ that remains independent of the number of update iterations. The error bound is the cost of using discrete weights as part of the optimization algorithm. Previous work with quantized models also includes error bounds \cite{Li2017,Li2019b}.

\newpage

\section{Code Sample of Core Implementation}
\label{appendix:implementation}

In this section we provide example codes in Python of a Boolean linear layer that employs $\xor$ logic kernel.
This implementation is in particular based on PyTorch \cite{Paszke2019}.
The class of Boolean linear layer is defined in \cref{alg:code_xorlinear}; and its backpropagation mechanism, overwritten from \texttt{autograd}, is shown in \cref{alg:code_xorfunction}.
Here, we consider both cases of the received backpropagation signal $Z$ as described in \cref{fig:BpSignals}, which are either Boolean  (see \cref{alg:code_bpropbool}) or real-valued (see \cref{alg:code_bpropreal}).
An example code of the Boolean optimizer is provided in \cref{alg:code_optim}.

Notice that our custom \texttt{XORLinear} layer can be flexibly combined with any \gls{FP} PyTorch modules to define a model.
The parameters of this layer are optimized by the \texttt{BooleanOptimizer}, whereas those of \gls{FP} layers are optimized by a common \gls{FP} optimizer like Adam \cite{Kingma2014}. 

\begin{algorithm}[hb]
    \centering
    \caption{Python code of \textsc{xor} linear layer}\label{alg:code_xorlinear}
\begin{lstlisting}[language=Python]
import torch

from torch import Tensor, nn, autograd
from typing import Any, List, Optional, Callable


class XORLinear(nn.Linear):

    def __init__(self, in_features: int, out_features: int, bool_bprop: bool, **kwargs):
        super(XORLinear, self).__init__(in_features, out_features, **kwargs)        
        self.bool_bprop = bool_bprop
    
    def reset_parameters(self):
        self.weight = nn.Parameter(torch.randint(0, 2, self.weight.shape))
        
        if self.bias is not None:
            self.bias = nn.Parameter(torch.randint(0, 2, (self.out_features,)))
    
    def forward(self, X):
        return XORFunction.apply(X, self.weight, self.bias, self.bool_bprop)
\end{lstlisting}
\end{algorithm}

\begin{algorithm}[hb]
    \centering
    \caption{Python code of the backpropagation logic of \textsc{xor} linear layer}\label{alg:code_xorfunction}
\begin{lstlisting}[language=Python]
class XORFunction(autograd.Function):

    @staticmethod
    def forward(ctx, X, W, B, bool_bprop: bool):
        ctx.save_for_backward(X,W,B)
        ctx.bool_bprop = bool_bprop      
          
        # Elementwise XOR logic
        S = torch.logical_xor(X[:,None,:], W[None,:,:])   
             
        # Sum over the input dimension
        S = S.sum(dim=2) + B
        
        # 0-centered for use with BatchNorm when preferred
        S = S - W.shape[1]/2
        
        return S
    
    @staticmethod
    def backward(ctx, Z):
        if ctx.bool_bprop:
            G_X, G_W, G_B = backward_bool(ctx, Z)
        else:
            G_X, G_W, G_B = backward_real(ctx, Z)

        return G_X, G_W, G_B, None
\end{lstlisting}
\end{algorithm}

\newpage

\begin{algorithm}[hb]
    \centering
    \caption{Backpropagation logic with Boolean received backpropagation}\label{alg:code_bpropbool}
\begin{lstlisting}[language=Python]
def backward_bool(ctx, Z):
    """
    Variation of input:
        - delta(xor(x,w))/delta(x) = neg w
        - delta(Loss)/delta(x) = xnor(z,neg w) = xor(z,w)
    Variation of weights:
        - delta(xor(x,w))/delta(w) = neg x
        - delta(Loss)/delta(x) = xnor(z,neg x) = xor(z,x)
    Variation of bias:
        - bias = xnor(bias,True) ==> Variation of bias is driven in 
          the same basis as that of weight with xnor logic and input True.
    Aggregation:
        - Count the number of TRUEs = sum over the Boolean data
        - Aggr = TRUEs - FALSEs = TRUEs - (TOT - TRUEs) = 2TRUES - TOT
          where TOT is the size of the aggregated dimension
    """
    X, W, B = ctx.saved_tensors
    
    # Boolean variation of input
    G_X = torch.logical_xor(Z[:,:,None], W[None,:,:])
    
    # Aggregate over the out_features dimension
    G_X = 2 * G_X.sum(dim=1) - W.shape[0]   
    
    # Boolean variation of weights
    G_W = torch.logical_xor(Z[:,:,None], X[:,None,:])
    
    # Aggregate over the batch dimension
    G_W = 2 * G_W.sum(dim=0) - X.shape[0]
    
    # Boolean variation of bias
    if B is not None:
        # Aggregate over the batch dimension
        G_B = 2 * Z.sum(dim=0) - Z.shape[0]

    # Return
    return G_X, G_W, G_B
\end{lstlisting}
\end{algorithm}

\begin{algorithm}[hb]
    \centering
    \caption{Backpropagation logic with real received backpropagation}\label{alg:code_bpropreal}
\begin{lstlisting}[language=Python]
def backward_real(ctx, Z):
    X, W, B = ctx.saved_tensors

    """
    Boolean variation of input processed using torch avoiding loop:
        -> xor(Z: Real, W: Boolean) = -Z * emb(W)
        -> emb(W): T->1, F->-1 => emb(W) = 2W-1
        => delta(Loss)/delta(X) = Z*(1-2W) """
    G_X = Z.mm(1-2*W)
    
    """ 
    Boolean variation of weights processed using torch avoiding loop:
        -> xor(Z: Real, X: Boolean) = -Z * emb(X)
        -> emb(X): T->1, F->-1 => emb(X) = 2X-1
        => delta(Loss)/delta(W) = Z^T * (1-2X) """
    G_W = Z.t().mm(1-2*X)
    
    """ Boolean variation of bias """
    if B is not None:
        G_B = Z.sum(dim=0)
    
    # Return
    return G_X, G_W, G_B
\end{lstlisting}
\end{algorithm}

\newpage

\begin{algorithm}[H]
    \centering
    \caption{Python code of Boolean optimizer}\label{alg:code_optim}
\begin{lstlisting}[language=Python]
class BooleanOptimizer(torch.optim.Optimizer):
    
    def __init__(self, params, lr: float):
        super(BooleanOptimizer, self).__init__(params, dict(lr=lr))
        for param_group in self.param_groups:
            param_group['accums'] = [torch.zeros_like(p.data) for p in param_group['params']]
            param_group['ratios'] = [0 for p in param_group['params']]
        self._nb_flips = 0

    @property
    def nb_flips(self):
        n = self._nb_flips
        self._nb_flips = 0
        return n

    def step(self):
        for param_group in self.param_groups:
            for idx, p in enumerate(param_group['params']):
                self.update(p, param_group, idx)

    def update(self, param: Tensor, param_group: dict, idx: int):
        accum = param_group['ratios'][idx] * param_group['accums'][idx] + param_group['lr'] * param.grad.data
        param_group['accums'][idx] = accum
        param_to_flip = accum * (2*param.data-1) >= 1
        param.data[param_to_flip] = torch.logical_not(param.data[param_to_flip])
        param_group['accums'][idx][param_to_flip] = 0.
        param_group['ratios'][idx] = 1 - param_to_flip.float().mean()
        self._nb_flips += float(param_to_flip.float().sum())
\end{lstlisting}
\end{algorithm}

\section{Training Regularization} \label{appendix:train_regu}

\subsection{Assumptions and Notations}

We use capital letters to denote tensors. So, $W^l$, $X^l$, $S^l$, and $Z^l$ denote weight, input, pre-activation, and received backpropagation tensors of a layer $l$. We also use \cref{prop:Conversion} and the discussion therein for studying Boolean logic in its embedding binary domain. It follows that in this section $\Bb$ denotes the binary set $\{\mp 1\}$.

Let $\Nc(\mu, \sigma^2)$ denote the Gaussian distribution of mean $\mu$ and variance $\sigma^2$, and $\mathbf{B}(p)$ denote Bernoulli distribution on $\Bb$ with $p \in [0, 1]$. Note that for $X \sim \mathbf{B}(p)$, if $\E{X} = 0$, then $\E{X^2} = 1$.

We assume that the backpropagation signal to each layer follows a Gaussian distribution. In addition, according to experimental evidence, cf. \fref{fig:Mean2Std}, we assume that their mean $\mu$ can be neglected w.r.t. their variance $\sigma^2$. 

\begin{figure}[!bh]
	\centering
	\includegraphics[width=0.75\textwidth]{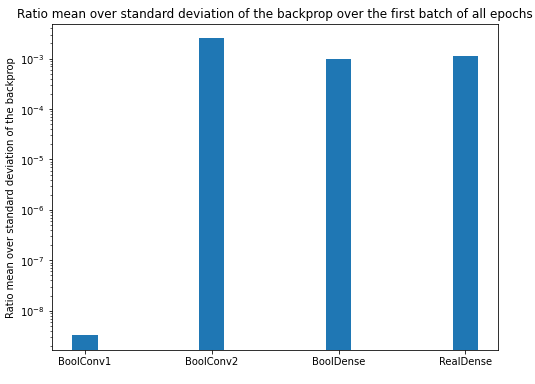}
	\caption{Empirical ratio of the mean to standard deviation of the backpropagation signal, experimented with CNN composed of BoolConv - BoolConv - BoolDense - RealDense layers and MNIST dataset.}
	\label{fig:Mean2Std}
\end{figure}

It follows that:
\begin{equation}
    Z^l \sim \Nc(\mu,\sigma^2), \quad  \textrm{with } \quad \mu \ll \sigma.
\end{equation}

We also assume that signals are mutually independent, i.e., computations are going to be made on random vectors, matrices and tensors, and it is assumed that the coefficients of these vectors, matrices and tensors are mutually independent. We assume that forward signals, backpropagation signals, and weights (and biases) of the layers involved are independent.

Consider a Boolean linear layer of input size $n$ and output size $m$, denote:
\begin{itemize}
    \item Boolean weights: $W^l \in \Bb^{m \times n}$;
    \item Boolean inputs: $X^l \in \Bb^n$;
    \item Pre-activation and Outputs: $S^l \in [|-n, n |]^m$, $X^{l+1} \in \Bb^m$;
    \item Downstream backpropagated signal: $Z^l \in \mathbb{R}^m$;
    \item Upstream backpropagated signal: $Z^{l-1} \in \mathbb{R}^n$.
\end{itemize}
In the forward pass, we have: $S^l = \xor(W^l, X^l)$, and $X^{l+1}=\maj(S^l)$.

With the following notations:

\begin{itemize}
    \item $W^l = (W^l_{ij})_{i<m,j<n} \sim (\mathbf{B}(p_{ij}))_{i<m,j<n}$ with $p_{ij} \in [0,1]$;
    \item $X^l = (X^l_{i})_{i<n} \sim (\mathbf{B}(q_{i}))_{i<n}$ with $q_{i} \in [0,1]$;
    \item $Z^l = (Z^l_{i})_{i<m} \sim (\Nc(\mu_i, \sigma_i^2))_{i<m}$;
    \item $\tilde Z$ stands for the truncated (with derivative of tanh) version of $Z$ for sake of simplicity.
\end{itemize}

And assuming that $\forall i, \sigma_i=\sigma, \mu_i=\mu$ and $\mu \ll \sigma$, we can derive scaling factors for linear layers in the next paragraph.

\begin{remark}
    The scaling factor inside a convolutional layer behaves in a similar fashion except that the scalar dot product is replace by a \textbf{full} convolution with the 180-rotated kernel matrix.
\end{remark}

\subsection{Backpropagation Scaling}\label{app:scaling_backprop}

We now compute the variance of the upstream backpropagated signal $Z^{l-1}$ with respect to the number of neurons and the variance of the downstream backpropagated signal:

\begin{align}
    \forall j<n, \var(Z^{l-1}_j) & = \sum_i^m \var(W^l_{ij} \tilde Z^l_i), \quad \text{($W$, $Z$ are self mutually independent)}\\
                                 & = \sum_i^m \E{{W^l_{ij}}^2 \tilde {Z^l_i}^2} - \E{W^l_{ij} \tilde{Z}^l_i}^2 \\
                                 & = \sum_i^m \E{{W^l_{ij}}^2} \E{\tilde {Z^l_i}^2} - \E{W^l_{ij}}^2 \E{\tilde Z^l_i}^2, \quad \text{($W$, $Z$ are independent)} \\
                                 & = \sum_i^m \E{\tilde {Z^l_i}^2} - (2p_{ij}-1)^2 \E{\tilde Z^l_i}^2 \\
                                 & \simeq m \E{\tilde {Z^l}^2}, \qquad\qquad\qquad\qquad\qquad\qquad\qquad \text{($\mu \ll \sigma)$} \\
                                 & = m \E{{Z^l}^2} \E{\frac{\partial \tanh}{\partial u}(u= \alpha W^l \cdot x^l)^2}, \qquad \text{(independence assump.)} \label{eq:variance}
\end{align}

Let us focus on the term $\E{\frac{\partial \tanh}{\partial u}(u)^2}$, where $u \in [|-m, m|]$ for m even. The probability $\mathbb{P}(u=l)$ is computed thanks to enumeration. The event ``$u=l$'' indicates that we have $k+l$ scalar values at level ``$+1$'' and $k$ at level ``$-1$'' such that $2k+l=m$. Hence,

\begin{equation}
    \mathbb{P}(u=l) = \binom{m}{\frac{m-l}{2}} \frac{1}{2}^\frac{m-l}{2} \frac{1}{2}^{m-\frac{m-l}{2}}.
\end{equation}
Thus,

\begin{align}
    \E{\frac{\partial \tanh}{\partial u}(u)^2} & = \sum_{u=-m}^m \frac{\partial \tanh}{\partial u}(u)^2 p(u)                                  \\
                                               & = 2 \sum_{u=0}^m \frac{\partial \tanh}{\partial u}(u)^2 p(u), \qquad \text{(with symmetry)}   \\
                                               & = \frac{1}{2^{m-1}} \sum_{u=0, u \; even}^m \binom{m}{\frac{m-l}{2}} (1- \tanh^2(\alpha u)).
\end{align}

The latter can be easily pre-computed for a given value of output layer size $m$, see Figure\ref{fig:tanhfactor}.

\begin{figure}
    \centering
    \includegraphics[width=0.6\textwidth]{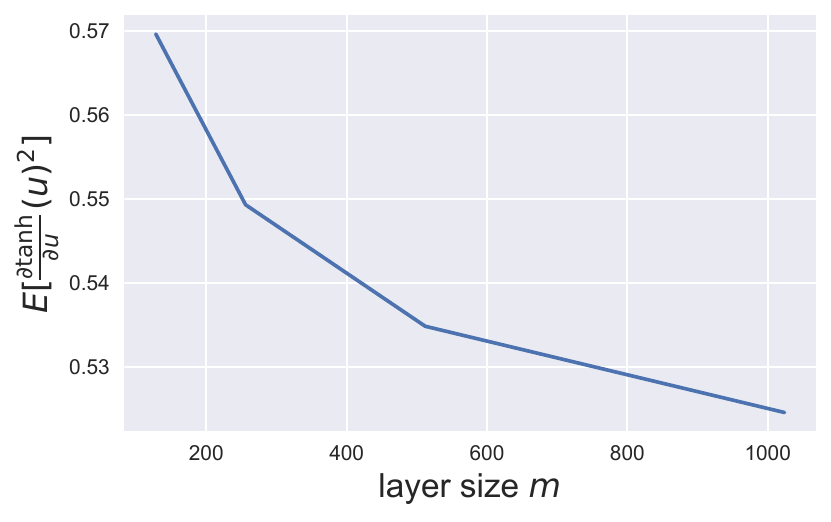}
    \caption{Expected value of $\tanh$ derivative with integer values as input, for several output sizes $m$.}
    \label{fig:tanhfactor}
\end{figure}

The above figure suggests that for reasonable layer sizes $m$, $\E{\frac{\partial \tanh}{\partial u}(u= \alpha W^l \cdot x^l)^2} \simeq \frac{1}{2}$. As a consequence we can make use of \eqref{eq:variance}, and approximate the variance of the backpropagated signal as:
\begin{equation}
    \var(Z^{l-1}) = \frac{m}{2} \var(Z^l).
\end{equation}

\begin{remark}
    The backpropagation inside a convolutional layer behaves in a similar fashion except that the tensor dot product is replace by a \textbf{full} convolution with the 180-rotated kernel matrix. For a given stride $v$ and kernel sizes $k_x, k_y$, the variance of the backpropagated signal is affected as follow:
    \begin{equation}
        \var(Z^{l-1}) = \frac{m k_x k_y}{2v} \var(Z^l).
    \end{equation}
\end{remark}

Let denote by MP the maxpooling operator (we assume its size to be $(2, 2)$). In the backward pass, one should not forget the impact of the $\tanh'(\alpha\Delta)$, \textbf{and} the MP operator so that:
\begin{equation}
    Z^{l-1} = \Conv_{\full}\pren{W_{\textrm{rot180}}^l, Z^l} \cdot \frac{\partial \tanh}{\partial u}(u=\MP[\Conv(\alpha W^l, X^l)]) \cdot \frac{\partial \MP}{\partial u}(u=\Conv(\alpha W^l, X^l)).
\end{equation}

Let us focus on the last term: $\frac{\partial \MP}{\partial u}(u=\Conv(\alpha W^l, X^l)) = \one{u=\max(\Conv(\alpha W^l, X^l))}$. Hence,
\begin{align}
    \E{\frac{\partial \MP}{\partial u}(u=\Conv(\alpha W^l, X^l))^2} & = \E{\frac{\partial \MP}{\partial u}(u=\Conv(\alpha W^l, X^l))} \\
                                                                    & = \frac{1}{4} \times 1 + 0.
\end{align}

As a consequence, for a given stride $v$ and kernel sizes $k_x, k_y$, the variance of the backpropagated signal is affected as follow:
\begin{equation}
    \var(Z^{l-1}) = \frac{1}{4} \frac{m k_x k_y}{2v} \var(Z^l).
\end{equation}

\subsection{BackPropagation Through Boolean Activation Function}\label{app:scaling_activation}

Due to the binary activation, the effect on the loss function by an action on weight $w$ diminishes with the distance $\Delta := |s - \tau|$ from threshold $\tau$ to pre-activation $s$ to which $w$ contributes. Throughout the step activation function, the backpropagation signal can be optionally re-weighted by a function which is inversely proportional to $\Delta$, for instance, $\tanh'(\Delta)$, $(1  + \Delta)^{-2}$, $\exp(-\Delta)$, or any other having this property. In our study, $\tanh'(\alpha \Delta)$ turns out to be a good candidate in which $\alpha$ is used to match the spreading in terms of standard deviation of this function and that of the pre-activation distribution. 

We start by computing the variance of the pre-activation signal $S^{l}$ with respect to the number of neurons, without considering the influence of backward $\tanh'$:
\begin{align}
    \forall j<n, \var(S^{l}_j) & = \sum_i^m \var(W^l_{ij} X^l_i), \quad \text{($W$, $X$ are self mutually independent)} \\
                               & = \sum_i^m \E{{W^l_{ij}}^2  {X^l_i}^2} - \E{W^l_{ij}  X^l_i]^2} 	\\
                               & = \sum_i^m \E{{W^l_{ij}}^2} \E{{X^l_i}^2} - \E{W^l_{ij}}^2 \E{X^l_i}^2, \quad \text{($W$, $X$ are independent)} \\
                               & = \sum_i^m \E{{X^l_i}^2} - (2p_{ij}-1)^2 \E{X^l_i}^2 \\
                               & \simeq m \E{{X^l}^2}, \qquad \text{since $\mu \ll \sigma$} \\
                               & = m.
\end{align}
This leads to
\begin{equation}\label{eq:ActivationScaling}
	\alpha = \frac{\pi}{2\sqrt{3m}}
\end{equation}
in order to have $\var(\alpha S)= \frac{\pi^2}{12}$, where $m$ is the range of the pre-activation, e.g., $m = c_{\textrm{in}} \times k_x \times k_y$ for a 2D convolution layer of filter dimensions $[c_{\textrm{in}}, k_x, k_y, c_{\textrm{out}}]$. 

\section{Experimental Details}\label{exp_setup}

\subsection{Image Classification}

\subsubsection{Training Setup}

The presented methodology and the architecture of the described Boolean \glspl{NN} were implemented in PyTorch \citep{Paszke2019} and trained on 8 Nvidia Tesla V100 GPUs. The networks thought predominantly Boolean, also contain a fraction of \gls{FP} parameters that were optimized using the Adam optimizer \citep{Kingma2014} with learning rate $10^{-3}$. For learning the Boolean parameters we used the Boolean optimizer (see \Cref{alg:code_optim}).
Training the Boolean networks for image classification was conducted with learning rates $\eta=150$ and $\eta=12$ (see \Cref{eq:Accum}), for architectures with and without batch normalization, respectively.
The hyper-parameters were chosen by grid search using the validation data.
During the experiments, both optimizers used the cosine scheduler iterating over 300 epochs.

We employ data augmentation techniques when training low bitwidth models which otherwise would overfit with standard techniques. In addition to techniques like random resize crop or random horizontal flip, we used RandAugment, lighting \cite{Liu2020} and Mixup \cite{Zhang2018a}. Following \cite{Touvron2019}, we used different resolutions for the training and validation sets. For \textsc{imagenet}, the training images were 192$\times$192 px and 224$\times$224 px for validation images. The batch size was 300 for both sets and the cross-entropy loss was used during training.

\subsubsection{\textsc{cifar10}}

\textsc{vgg-small} is found in the literature with different fully-connected \gls{FC} layers. Several works take inspiration from the classic work of \citep{Courbariaux2015}, which uses 3 \gls{FC} layers. Since other \gls{BNN} methodologies only use a single \gls{FC} layer, \Cref{TabBlocks} presents the results with the modified \textsc{vgg-small}.
\begin{table}[htp!]
    \centering
    \caption{Top-1 accuracy for different binary methodologies using the modified \textsc{vgg-small} (ending with 1 \gls{FC} layer) on the \textsc{cifar10} dataset.}
    \resizebox{0.5\columnwidth}{!}{
    \begin{tabular}{ l  c  c  c} \toprule
        \multirow{2}{*}{\textbf{Method}}    & \textbf{Forward}		 &  \textbf{Training} 		& \textbf{Acc.} \\
         & \textbf{Bit-width (W/A)}	 & \textbf{Bit-width (W/G)}& \textbf{(\%)} \\ \midrule\midrule
        \gls{FP}                                & 	$32/32$	&	$32/32$	&	93.8		\\ %
        \textsc{xnor-net}	\cite{Rastegari2016}	     &	$1/1$	&	$32/32$	&	87.4		\\ %
        \textsc{lab} \cite{Hou2016}                &	$1/1$	&	$32/32$	&	87.7		\\ %
        \textsc{rad} \cite{Ding2019}               &	$1/1$	&	$32/32$	&	90.0		\\ %
        \textsc{ir-net} \cite{Qin2020a}            &	$1/1$	&	$32/32$	&	90.4		\\ %
        \textsc{rbnn} \cite{lin2020rotated}       &	$1/1$	&	$32/32$	&	91.3		\\ %
        \textsc{slb} \cite{Yang2020}               &	$1/1$	&	$32/32$	&	92.0		\\ %
        \ours [Ours]                              &	$1/1$	&	$1/16$	&	90.8		\\ \bottomrule
    \end{tabular}}    
    \label{TabBlocks}
\end{table}

\subsubsection{Ablation Study on Image Classification}

The final block design for image classification was established after iterating over two models. The Boolean blocks examined were evaluated using the \textsc{resnet18} baseline architecture and  adjusting the training settings to improve performance. \Cref{fig:designs} presents the preliminary designs.

The Boolean Block I, Figure \ref{fig:block2}, is similar to the original \textsc{resnet18} block in that \gls{BN} operations are removed and ReLUs are replaced by the Boolean activation. This design always includes a convolution in the shortcut with spatial resolution being handled by the stride. Notice that for this block we add a Boolean activation after the Maxpool module in the baseline (also for the final baseline architecture). The Boolean Block II, Figure \ref{fig:block4}, is composed by two stacked residual modules. For downsampling blocks we use the reshaping operation to reduce the spatial resolution and enlarge the channel dimensions both by a factor of 2. The shortcut is modified accordingly with different operations in order to guarantee similar spatial dimensions before the summation.

\begin{figure}[!h]
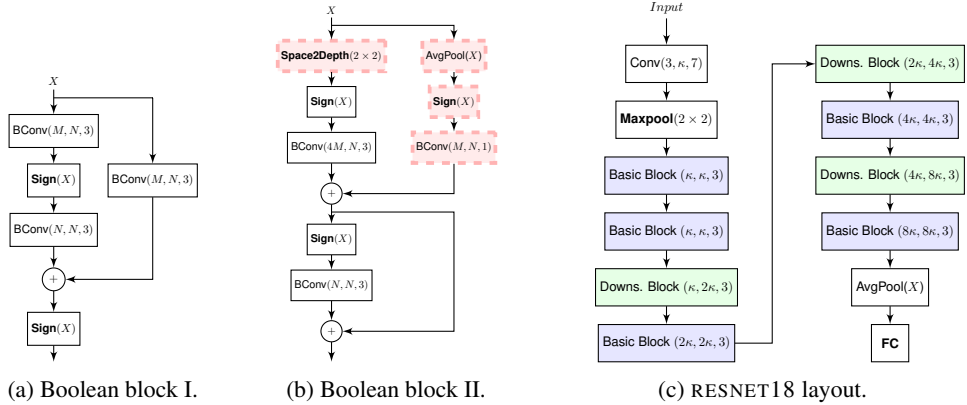

\tikzexternaldisable
	\begin{subfigure}[t]{.3\columnwidth}
		\centering		
		\includegraphics[width=.6\textwidth]{figures/classification_supp/BlockI}		
		\caption{Boolean block I.}
		\label{fig:block2}
	\end{subfigure}
	\hspace*{-6mm}
	\begin{subfigure}[t]{.3\columnwidth}
		\centering
		\includegraphics[width=.7\textwidth]{figures/classification_supp/BlockII}
		\caption{Boolean block II.}
		\label{fig:block4}
	\end{subfigure}
	\hspace*{5mm}
	\begin{subfigure}[t]{.32\columnwidth}
		\centering
		\includegraphics[width=1.1\textwidth]{figures/classification_supp/resnet_layout}
		\caption{\textsc{resnet18} layout.}
		\label{fig:resnet}
	\end{subfigure}
\tikzexternalenable
	\caption{Preliminary designs for the baseline architecture and the Boolean basic blocks. The dashed and red-shaded operations in the Boolean block II are introduced for downsampling blocks.}
	\label{fig:designs}
\end{figure}

Table \ref{TabBlocksAb} summarizes the results obtained with the proposed designs on \textsc{imagenet}. During our experimentation, we validated the hypothesis that increasing network capacity on the convolutional layers yielded higher accuracy values. However, similar to \gls{FP} \glspl{CNN}, we confirmed there is a limit by which the hypothesis ceases to be true, leading to overfitting. Incorporating a more severe training strategy had a sustained positive impact. Even so, for larger configurations, the compromise between accuracy and size can be cumbersome.

Among the strategies to reduce overfitting during training we included: mixup data-augmentation \cite{Zhang2018a}, image illumination tweaking, rand-augment and smaller input resolution for training than for validation \cite{Touvron2019}. All combined, increased the accuracy by $\sim$3 points (check results for Block II + base channel 230 with and w/o additional data augmentation).

\begin{table}[b!]
    \centering    
    \caption{Evaluation of the proposed blocks in \textsc{imagenet} and their respective configurations during training.}
    \resizebox{0.7\columnwidth}{!}{
    \begin{tabular}{ c  c  c  c  c  c} \toprule
    \textbf{Block}   & \textbf{Base}    & \textbf{$1^{st}$ Conv.} & \textbf{Shortcut} & \textbf{Data} & \textbf{Acc.} \\
    \textbf{Design}  & \textbf{Channel} & \textbf{Bit-width} & \textbf{Fil. Size}  & \textbf{Augmentation} & \textbf{(\%)}\\ \midrule \midrule
     \multirow{5}{*}{Block I} & 128 & 32 & $1\times1$ & Random Crop, Random Flip &  53.35 \\ %
     & 192 & 32 & $1\times1$ & Random Crop, Random Flip & 56.79 \\ %
     & 192 & 32 & $1\times1$ & Lighting, Mixup, RandAugment and \cite{Touvron2019} & 61.90 \\ %
     & 256 & 32 & $1\times1$ & Lighting, Mixup, RandAugment and \cite{Touvron2019} & 64.32 \\
     & 256 & 32 & $3\times3$ & Lighting, Mixup, RandAugment and \cite{Touvron2019} & \textbf{66.89} \\ \midrule
    \multirow{6}{*}{Block II} & 128 & 1 & $1\times1$ & Random Crop, Random Flip & 56.05 \\ %
     & 128 & 32 & $1\times1$ & Random Crop, Random Flip & 58.38 \\ %
     & 192 & 32 & $1\times1$ & Random Crop, Random Flip & 61.10 \\ %
     & 192 & 32 & $1\times1$ & Lighting, Mixup, RandAugment and \cite{Touvron2019} & 63.21 \\ %
     & 230 & 32 & $1\times1$ & Random Crop, Random Flip & 61.22 \\ %
     & 230 & 32 & $1\times1$ & Lighting, Mixup, RandAugment and \cite{Touvron2019} & \textbf{64.41}\\ \bottomrule
\end{tabular}}    
    \label{TabBlocksAb}
\end{table}

Compared to Block II, notice that the data streams in Block I are predominantly Boolean throughout the design. This is because it makes use of lightweight data types such as integer (after convolutions) and binary (after activations). In addition, it avoids the need of using a spatial transformation that may affect the data type and data distribution. In that regard, Block II requires 4 times more parameters for the convolution after reshaping, than the corresponding operation in Block I. This is exacerbated in upper layer convolutions, where the feature maps are deeper. Therefore, it makes sense to use Block I, as it is lighter and less prone to overfitting when the network capacity is expanded.

\subsubsection{Neural Gradient Quantization}\label{nn_quant}
For completeness, we also implemented neural gradient quantization to quantize it by using \textsc{int4} quantization with logarithmic round-to-nearest approach \cite{Chmiel2021} and statistics aware weight binning \cite{choi2018bridging}. Statistics aware weight binning is a method that seeks for the optimal scaling factor, per layer, that minimizes the quantization error based on the statistical characteristics of neural gradients. It involves per layer additional computational computations, but stays negligible with respect to other (convolution) operations. On \textsc{imagenet}, we recover the findings from \cite{Chmiel2021}: 4 bits quantization is enough to recover standard backpropagation performances.

\subsubsection{Basic Blocks SOTA BNNs for Classification}

Recent \gls{BNN} methodologies have proposed different mechanisms to improve performance. Most of them exploit full-precision operations to adjust datastreams within the network, like shift and scaling factors before binary activations \cite{Liu2020} or channel scaling through Squeeze-and-Excitation modules \cite{Martinez2020,Guo2022}. \Cref{fig:designsSota} shows the basic blocks of three methodologies that perform particularly well in \textsc{imagenet}. Together with \gls{BN} and regular activations, those techniques not only add an additional level of complexity but also lead to heavier use of computational resources and latency delays.

For comparison we also show the proposed block (\Cref{fig:blocka}) used in our experiments for Image Classification, Image Segmentation and Image Super-Resolution. Our block is compact in the sense that it only includes Boolean convolutions and Boolean activations, strategically placed to keep the input and output datastreams Boolean.
\begin{figure}[!h]
\tikzexternaldisable
	\centering
	\begin{subfigure}[t]{.2\columnwidth}
		\centering
		\includegraphics[width=\columnwidth]{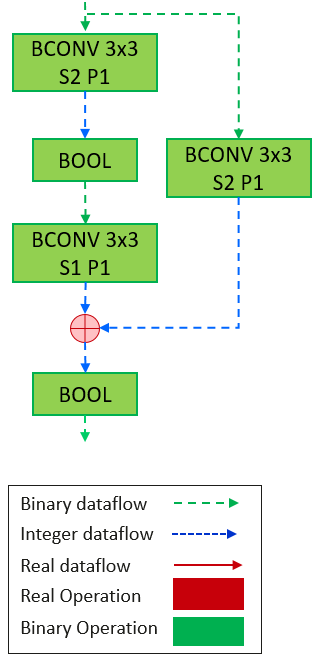}
		\caption{Boolean Block w/o \gls{BN}.}
		\label{fig:blocka}
	\end{subfigure}
	\hfill
	\begin{subfigure}[t]{.28\columnwidth}
		\centering
		\includegraphics[width=\columnwidth]{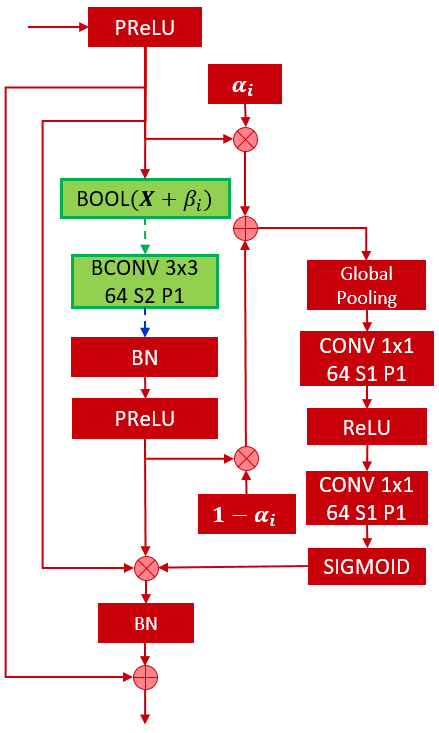}
		\caption{\textsc{bnext} sub-block \cite{Guo2022}.}
		\label{fig:blockb}
	\end{subfigure}
	\hfill
	\begin{subfigure}[t]{.23\columnwidth}
		\centering
		\includegraphics[width=\columnwidth]{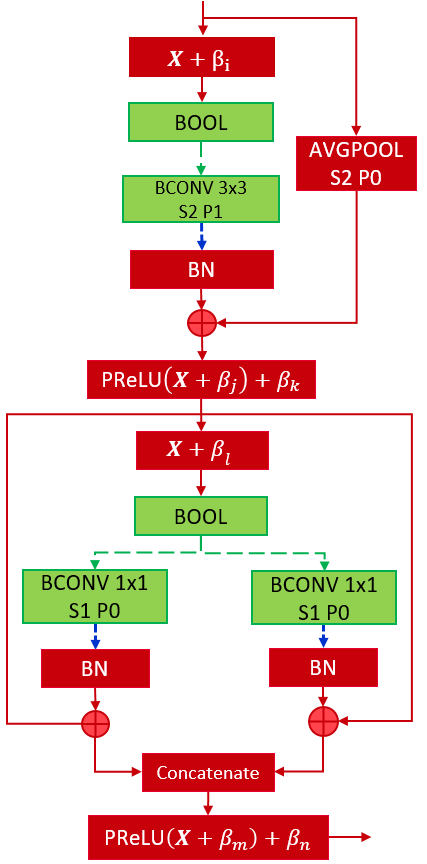}
		\caption{\textsc{reactnet} \cite{Liu2020}.}
		\label{fig:blockc}
	\end{subfigure}
	\hfill
	\begin{subfigure}[t]{.2\columnwidth}
		\centering
		\includegraphics[width=\columnwidth]{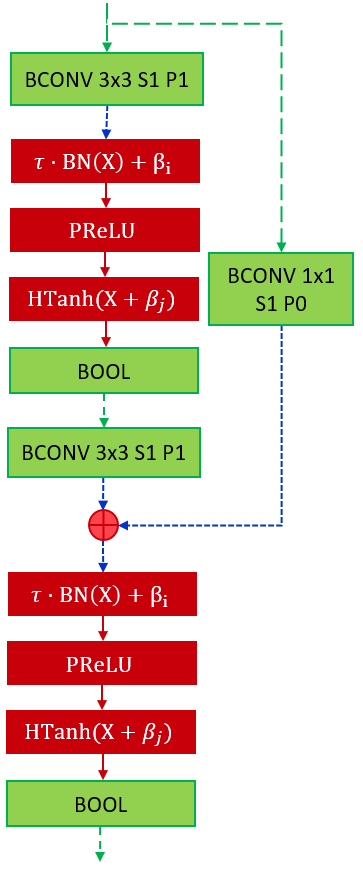}
		\caption{\textsc{bineal-net} \cite{Nie2022}.}
		\label{fig:blockd}
	\end{subfigure}
\tikzexternalenable
	\caption{Comparative graph of popular \gls{BNN} techniques and our Boolean module. Notice how multiple full-precision operations like \gls{BN}, PReLU, or Squeeze-and-Excitation are overly used on each \gls{BNN} block.}
	\label{fig:designsSota}
\end{figure}

\subsection{Image Super-resolution}

The seminal \textsc{edsr} \cite{Lim2017} method for super-resolution was used together with our Boolean methodology. In particular, the residual blocks are directly replaced by our Boolean basic block, see \Cref{fig:edsr_arch}. For all three tasks in super-resolution, (i.e. $\times2$, $\times3$, $\times4$), training was carried out with small patches of 96$\times$96 px (40 of them extracted randomly from each single image in the \textsc{div2k} dataset) and validated with the original full-resolution images. The learning rate for real and boolean parameters were $10^{-4}$ and $\eta=36$, respectively. The networks were trained by minimizing the $L_1$-norm between the ground-truth and the predicted upsampled image while using the Adam optimizer and Boolean optimizer (see \Cref{alg:code_optim}). In our experiments the batch size was 20. Some example images generated by our methodology are showed in \Cref{fig:sr1,fig:sr2}.
\begin{figure}[!h]
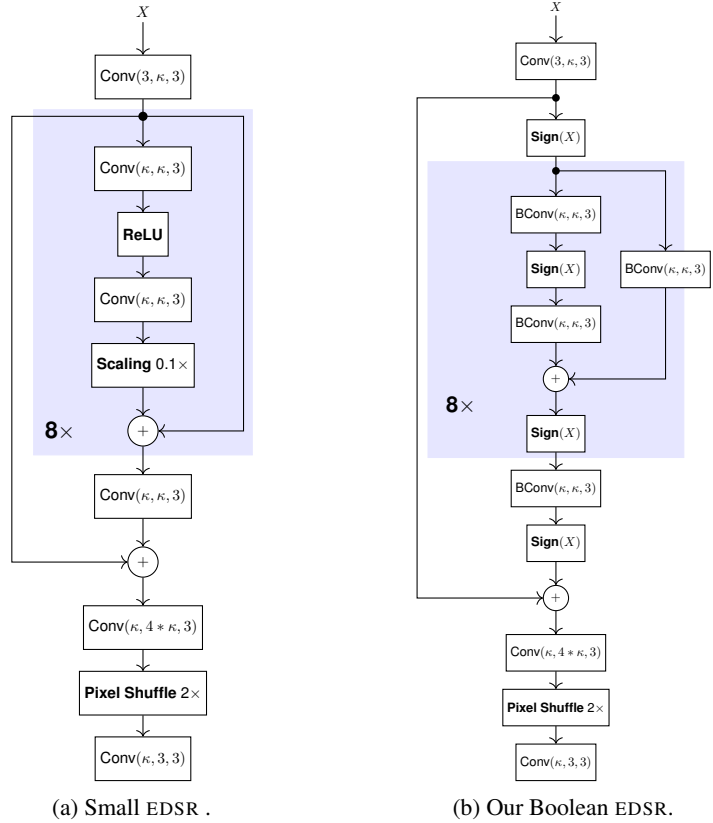

\tikzexternaldisable
	\centering	
	\begin{subfigure}[b]{.23\columnwidth}
		\centering
		\includegraphics[width=\columnwidth]{figures/super_resolution_supp/edsr_fp}
		\caption{Small \textsc{edsr} .}
		\label{fig:edsr}
	\end{subfigure}
	\hspace{2cm}
	\begin{subfigure}[b]{.28\columnwidth}
		\centering
		\includegraphics[width=\columnwidth]{figures/super_resolution_supp/edsr_bool}
		\caption{Our Boolean \textsc{edsr}.}
		\label{fig:booledsr}
	\end{subfigure}
\tikzexternalenable
	\caption{Small \textsc{edsr} for single scale $\times 2$ super-resolution and our Boolean version with Boolean residual blocks. In both architectures the channels dimensions are $\kappa=256$ and the shaded blocks are repeated $8\times$.}
	\label{fig:edsr_arch}
\end{figure}

\begin{figure}[!h]
\tikzexternaldisable
	\centering
	\begin{subfigure}[b]{.295\columnwidth}
		\centering
		\includegraphics[width=\columnwidth]{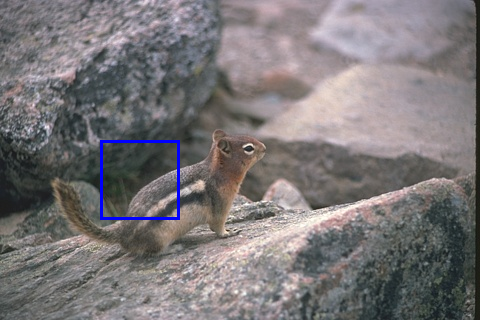}
		\label{fig:super_squirrela}
	\end{subfigure}
	\begin{subfigure}[b]{.195\columnwidth}
		\centering
		\includegraphics[width=\columnwidth]{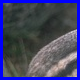}
		\label{fig:super_squirrelpa}
	\end{subfigure}
	\begin{subfigure}[b]{.295\columnwidth}
		\centering
		\includegraphics[width=\columnwidth]{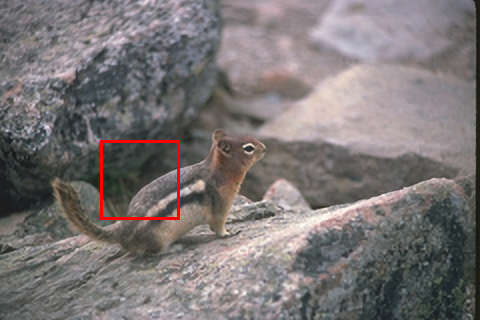}
		\label{fig:super_squirrelb}
	\end{subfigure}
	\begin{subfigure}[b]{.195\columnwidth}
		\centering
		\includegraphics[width=\columnwidth]{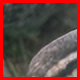}
		\label{fig:super_squirrelpb}
	\end{subfigure}
	\hfill
	\begin{subfigure}[b]{.295\columnwidth}
		\centering
		\includegraphics[width=\columnwidth]{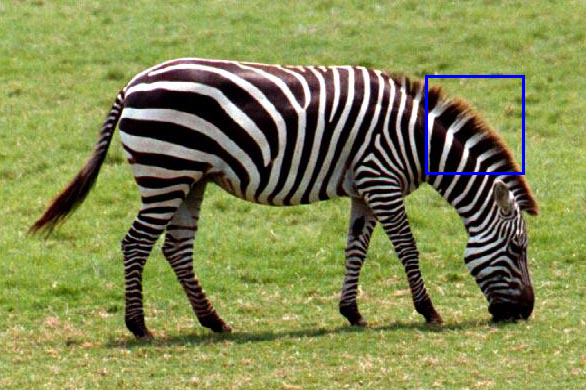}
		\caption{Ground-truth targets.}
		\label{fig:super_zebraa}
	\end{subfigure}
	\begin{subfigure}[b]{.195\columnwidth}
		\centering
		\includegraphics[width=\columnwidth]{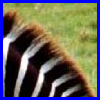}
		\caption{Enlarged crops.}
		\label{fig:super_zebrapa}
	\end{subfigure}
	\begin{subfigure}[b]{.295\columnwidth}
		\centering
		\includegraphics[width=\columnwidth]{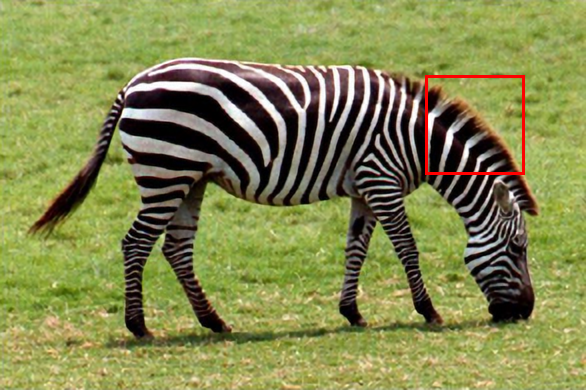}
		\caption{Predicted images.}
		\label{fig:super_zebrab}
	\end{subfigure}
	\begin{subfigure}[b]{.195\columnwidth}
		\centering
		\includegraphics[width=\columnwidth]{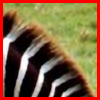}
		\caption{Enlarged crops.}
		\label{fig:super_zebrapb}
	\end{subfigure}
\tikzexternalenable
	\caption{Ground-truth high resolution images and the output of our Boolean super-resolution methodology. First row: image ``013'' from \textsc{bsd100}, with PSNR: 35.54 dB. Second row: image ``014'' from Set14, with PSNR: 33.92 dB.}
	\label{fig:sr1}
\end{figure}
\begin{figure}[!h]
\tikzexternaldisable
	\centering
	\begin{subfigure}[b]{.7\columnwidth}
		\centering
		\includegraphics[width=\columnwidth]{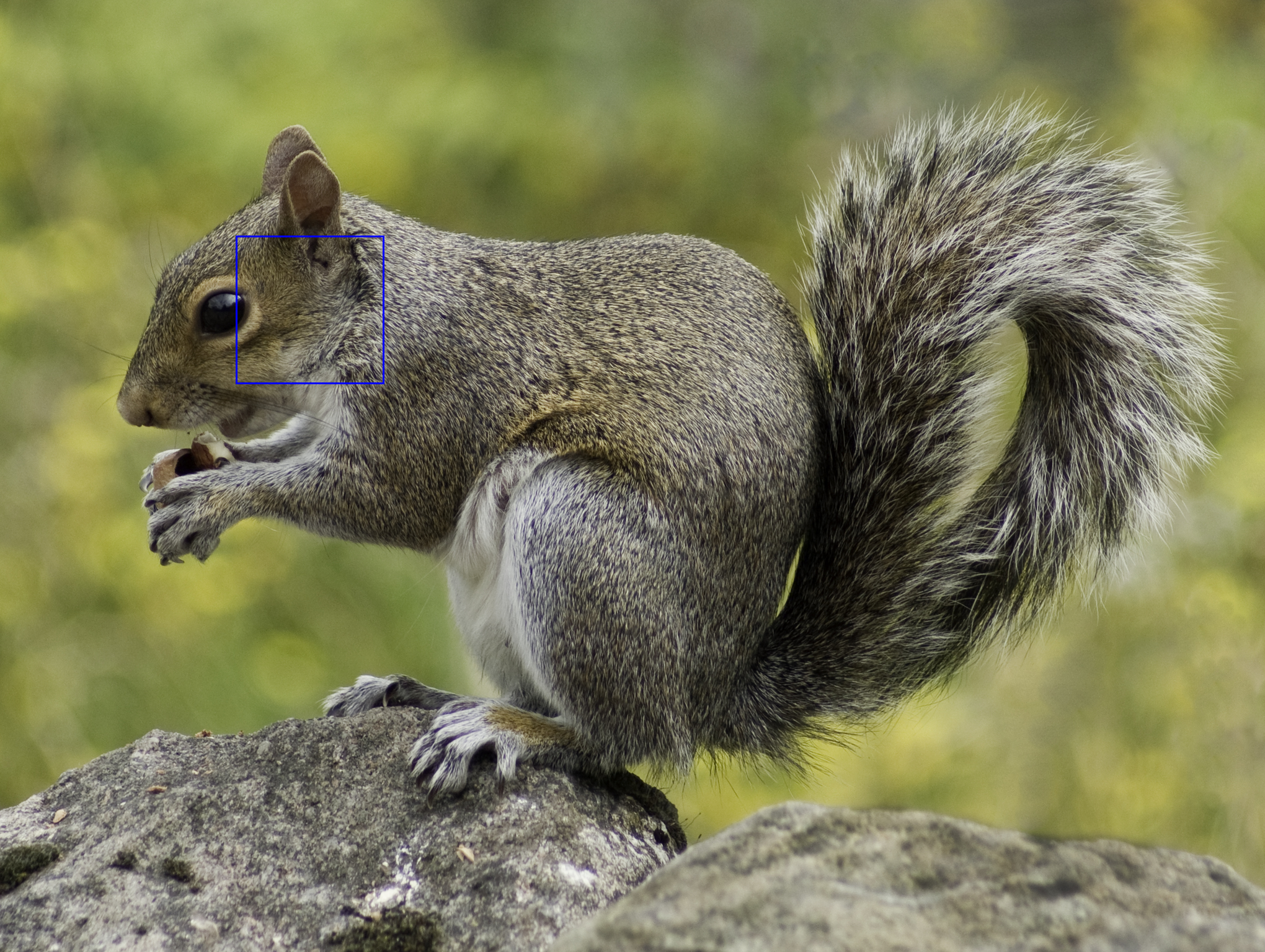}
		\caption{Ground-truth target.}
	\end{subfigure}
	\begin{subfigure}[b]{.25\columnwidth}
		\centering
		\includegraphics[width=\columnwidth]{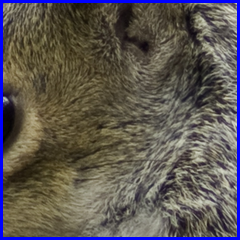}
		\caption{Enlarged crop.}
	\end{subfigure}
	\begin{subfigure}[b]{.7\columnwidth}
		\centering
		\includegraphics[width=\columnwidth]{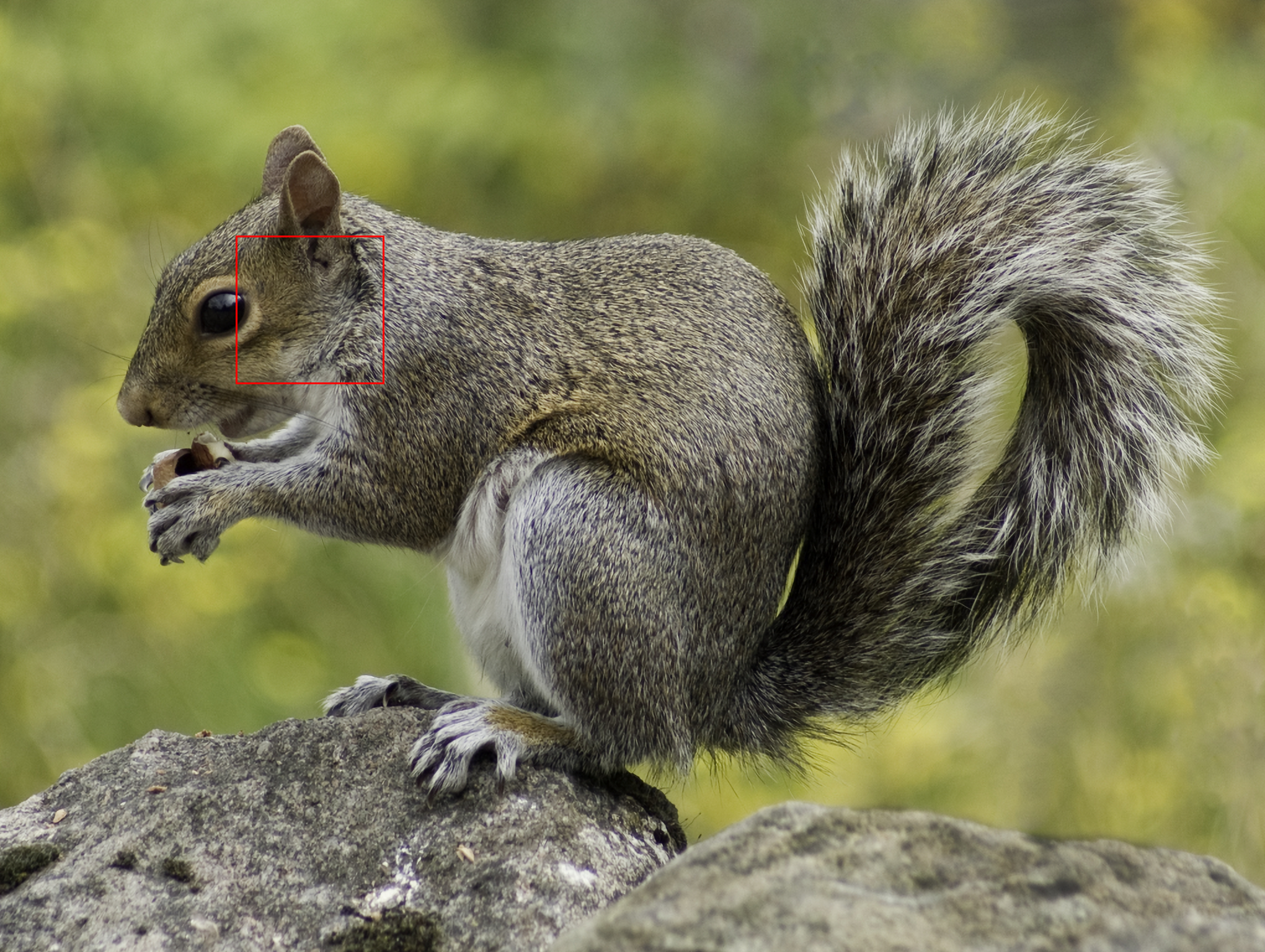}
		\caption{Predicted image.}
	\end{subfigure}
	\begin{subfigure}[b]{.25\columnwidth}
		\centering
		\includegraphics[width=\columnwidth]{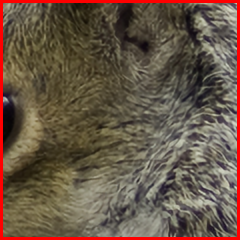}
		\caption{Enlarged crop.}
	\end{subfigure}
\tikzexternalenable
	\caption{Ground-truth high resolution target image (top) and the output of our Boolean super-resolution methodology (bottom). Image ``0810'' from the validation set of \textsc{div2k}, with PSNR: 34.90 dB}
	\label{fig:sr2}
\end{figure}

\subsection{Semantic Segmentation}\label{supp_seg}

\subsubsection{Network architecture}
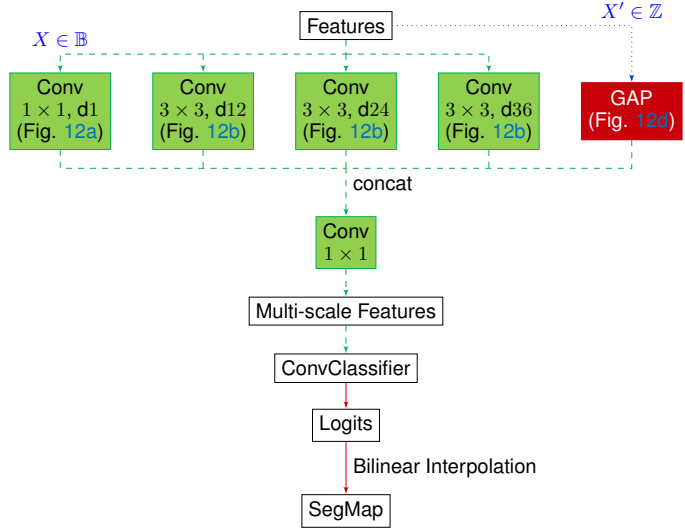
\begin{figure}[t]
\tikzexternaldisable
    \centering
	\resizebox{.65\textwidth}{!}{
    \begin{tikzpicture}[auto, node distance=2cm,>=latex']

        \tikzstyle{arrow} = [thick,->,>=stealth, line width=1.5pt]
        \tikzstyle{line} = [thick,-,>=stealth, line width=1.5pt]

        \node[rectangle, draw] (Input) at (0, -0/5) {Features};
        \node[rectangle, draw, b_df_c, fill=b_block_c, text=black, align=center, minimum width=1.5cm, minimum height=1cm] (D0) at (-5, -1.5) {Conv \\ $1\times 1$, d$1$ \\ (Fig. \ref{fig:binaspp_1x1})};
        \node[rectangle, draw, b_df_c, fill=b_block_c, text=black, align=center, minimum width=1.5cm, minimum height=1cm] (D12) at (-2.5, -1.5) {Conv \\ $3\times 3$, d$12$ \\ (Fig. \ref{fig:binaspp_3x3})};
        \node[rectangle, draw, b_df_c, fill=b_block_c, text=black, align=center, minimum width=1.5cm, minimum height=1cm] (D24) at (0, -1.5) {Conv \\ $3\times 3$, d$24$ \\ (Fig. \ref{fig:binaspp_3x3})};
        \node[rectangle, draw, b_df_c, fill=b_block_c, text=black, align=center, minimum width=1.5cm, minimum height=1cm] (D36) at (2.5, -1.5) {Conv \\ $3\times 3$, d$36$ \\ (Fig. \ref{fig:binaspp_3x3})};
        \node[rectangle, draw, fp_df_c, fill=fp_block_c, text=white, align=center, minimum width=1.5cm, minimum height=1cm] (GAP) at (5, -1.5) {GAP \\ (Fig. \ref{fig:binaspp_gap})};
        \node[rectangle, draw, b_df_c, fill=b_block_c, text=black, align=center] (Conv) at (0, -3.8) {Conv \\ $1\times 1$};
        \node[rectangle, draw] (Output) at (0, -5) {Multi-scale Features};

        \node[rectangle, draw, align=center] (Classifier) at (0, -6) {ConvClassifier};
        \node[rectangle, draw, align=center] (Logits) at (0, -7) {Logits};
        \node[rectangle, draw, align=center] (SegMap) at (0, -8.5) {SegMap};

        \draw [->, b_df_c, dashed] (Input.south) -- (0, -0.5) -| node[above, blue]{$X\in \mathbb{B}$} (D0.north);
        \draw [->, b_df_c, dashed] (-2.5, -0.5) -- (D12.north);
        \draw [->, b_df_c, dashed] (0, -0.5) -- (D24.north);
        \draw [->, b_df_c, dashed] (0, -0.5) -| (D36.north);
        \draw [->, int_df_c, dotted] (Input.east) -| node[above, blue] {$X'\in \mathbb{Z}$} (GAP.north);

        \draw [-, b_df_c, dashed] (D0.south) |- (0, -2.5); %
        \draw [-, b_df_c, dashed] (D12.south) -- (-2.5, -2.5); %
        \draw [-, b_df_c, dashed] (D24.south) -- (-0, -2.5); %
        \draw [-, b_df_c, dashed] (D36.south) -- (2.5, -2.5); %
        \draw [->, b_df_c, dashed] (GAP.south) -- (5, -2.5) -| node[right, black, align=center] {\\\\concat} (Conv.north);

        \draw [->, b_df_c, dashed] (Conv) -- (Output);
        \draw [->, b_df_c, dashed] (Output) -- (Classifier);
        \draw [->, fp_df_c] (Classifier) -- (Logits);
        \path (Logits) -- node[right, align=center]{Bilinear Interpolation} (SegMap);
        \draw [->, fp_df_c] (Logits) -- (SegMap);

    \end{tikzpicture}
	}
\tikzexternalenable
    \caption{Boolean segmentation architecture.}
    \label{fig:aspp_bool}
\end{figure}
Our Boolean architecture is based on \textsc{deeplabv3} \cite{chen2017rethinking}, which has shown great success in semantic segmentation. It is proven that using dilated or atrous convolutions, which preserve the large feature maps, instead of strided convolutions is prominent for this task. In our Boolean model with \textsc{resnet18} layout, we replace the strided convolutions in the last two \textsc{resnet18} layers with the non-strided version, and the dilated convolutions are employed to compensate for the reduced receptive field. Thus, the images are $8\times$ downsampled instead of $32\times$, preserving small object features and allowing more information flow through the Boolean network. As shown in \Cref{fig:block2}, in the Boolean basic block, a $3\times 3$ convolution instead of $1\times 1$ convolution is used to ensure the comparable dynamic range of pre-activations between the main path and the shortcut. Keeping these Boolean convolutional layers non-dilated naturally allows the backbone to extract multi-scale features without introducing additional computational cost.

\begin{figure}[t]
\tikzexternaldisable
	\centering
	\begin{subfigure}[b]{0.24\textwidth}
    \centering    
    \resizebox{.5\textwidth}{!}{
        \begin{tikzpicture}[auto, node distance=1.5cm,>=latex']
            \node (in){$X \in \mathbb{B}$};
            \node (mod1)[binary, b_df_c, fill=b_block_c, text=black, below of=in]{BConv$(M,N,1)$};
            \node (actv1)[binary, b_df_c, fill=b_block_c, text=black, below of=mod1]{\textbf{Sign}$(X)$};

            \draw[->, b_df_c, dashed] (in) -- (mod1);
            \draw[->, int_df_c, dotted] (mod1) -- (actv1);
            \draw[->, b_df_c, dashed] (actv1) -- ([yshift=-0.5cm]actv1.south);
        \end{tikzpicture}
    }
    \caption{ }%
    \label{fig:binaspp_1x1}
	\end{subfigure}
\hfil
\begin{subfigure}[b]{0.24\textwidth}
    \centering
    
    \resizebox{.5\textwidth}{!}{
        \begin{tikzpicture}[auto, node distance=1.5cm,>=latex']
            \node (in) {$X \in \mathbb{B}$};
            \node (mod1)[binary, b_df_c, fill=b_block_c, text=black, align=center, below of=in]{BConv$_{d}(M,N,3)$ \\ $d \in \{12, 24, 36\}$};
            \node (actv1)[binary, b_df_c, fill=b_block_c, text=black, below of=mod1]{\textbf{Sign}$(X)$};

            \draw[->, b_df_c, dashed] (in) -- (mod1);
            \draw[->, int_df_c, dotted] (mod1) -- (actv1);
            \draw[->, b_df_c, dashed] (actv1) -- ([yshift=-0.5cm]actv1.south);
        \end{tikzpicture}
    }
    \caption{ }%
    \label{fig:binaspp_3x3}
\end{subfigure}
\hfil
\begin{subfigure}[b]{0.24\textwidth}
    \centering    
    \resizebox{.5\textwidth}{!}{
        \begin{tikzpicture}[auto, node distance=1.5cm,>=latex']
            \node (in){$X \in \mathbb{B}$};
            \node (gap) [binary, fp_df_c, fill=fp_block_c, text=white, below of=in] {Global Avg Pool};
            \node (mod1)[binary, b_df_c, fill=b_block_c, text=black, below of=gap]{BConv$(M,N,1)$};
            \node (actv1)[binary, b_df_c, fill=b_block_c, text=black, below of=mod1]{\textbf{Sign}$(X)$};

            \draw[->, b_df_c, dashed] (in) -- (gap);
            \draw[->, fp_df_c] (gap) -- (mod1);
            \draw[->, fp_df_c] (mod1) -- (actv1);
            \draw[->, b_df_c, dashed] (actv1) -- ([yshift=-0.5cm]actv1.south);
        \end{tikzpicture}
    }
    \caption{ }%
    \label{fig:nv_binaspp_gap}
\end{subfigure}
\hfil
\begin{subfigure}[b]{0.24\textwidth}
    \centering
    \resizebox{.5\textwidth}{!}{
        \begin{tikzpicture}[auto, node distance=1.5cm,>=latex']
            \node (in){$X' \in \mathbb{Z}$};
            \node (bn1) [binary, fp_df_c, fill=fp_block_c, text=white, below of=in] {ReLU - BN};
            \node (gap) [binary, fp_df_c, fill=fp_block_c, text=white, below of=bn1] {Global Avg Pool};
            \node (mod1)[binary, b_df_c, fill=b_block_c, text=black, below of=gap]{BConv$(M,N,1)$};
            \node (actv1)[binary, fp_df_c, fill=fp_block_c, text=white, below of=mod1]{BN - \textbf{Sign}$(X)$};

            \draw[->, int_df_c, dotted] (in) -- (bn1);
            \draw[->, fp_df_c] (bn1) -- (gap);
            \draw[->, fp_df_c] (gap) -- (mod1);
            \draw[->, fp_df_c] (mod1) -- (actv1);
            \draw[->, b_df_c, dashed] (actv1) -- ([yshift=-0.5cm]actv1.south);
        \end{tikzpicture}
    }
\tikzexternalenable
\caption{ }%
\label{fig:binaspp_gap}
\end{subfigure}
	\vspace{.2cm}
	\caption{Boolean Atrous Spatial Pyramid Pooling (\textsc{bool-aspp}) architecture. (a) $1\times 1$ Conv branch. (b) $3\times 3$ dilated Conv branch with dilation rate of $d$. (c) Naive global average pooling branch. (d) Global average pooling branch.}
\end{figure}
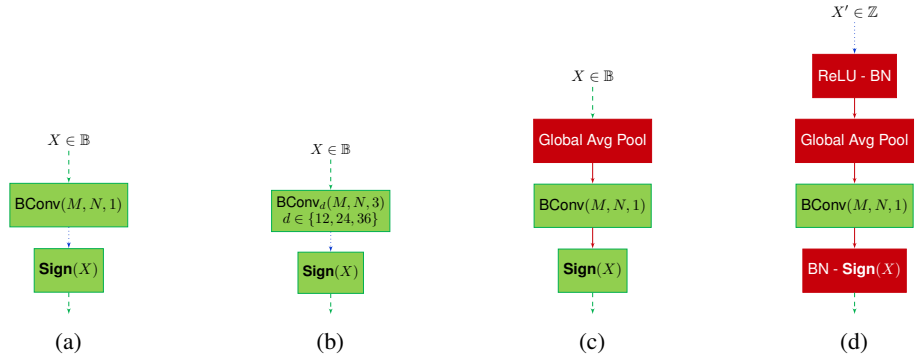

The Atrous Spatial Pyramid Pooling (\textsc{aspp}) consists of multiple dilated convolution layers with different dilation rates and global average pooling in parallel, which effectively captures multi-scale information. In the Boolean \textsc{aspp} (\textsc{bool-aspp}), we use one $1\times 1$ Boolean convolution and three $3\times 3$ Boolean dilated convolution with dilation rates of $\{12, 24, 36\}$ following by Boolean activation functions. The global average pooling (\textsc{gap}) branch in \textsc{aspp} captures image-level features, which is crucial for global image understanding as well as large object segmenting accuracy. However, in \textsc{bool-aspp}, as shown in \Cref{fig:nv_binaspp_gap}, the Boolean input $X$ leads to significant information loss before the global average pooling may cause performance degradation on large objects. Therefore, we keep the inputs integer for the \textsc{gap} branch as demonstrated in \Cref{fig:binaspp_gap}. To prevent numerical instability, batch normalization is used in the \textsc{gap} branch before each activation function. Using \textsc{bool-aspp} enhances the multi-scale feature extraction and avoids parameterized upsampling layers, e.g. transposed convolution.

\subsubsection{Training setup}
The model was trained on the \textsc{cityscapes} dataset for $400$ epochs with a batch size of $8$. The AdamW optimizer \cite{loshchilov2017decoupled} with an initial learning rate of $5\times 10^{-4}$ and the Boolean logic optimizer (see \Cref{alg:code_optim}) with a learning rate of $\eta=12$ were used respectively for real and Boolean parameters.
At the early training stage, parameters could easily be flipped due to the large backward signal; thus, to better benefit from the \textsc{imagenet}-pretrained backbone, we reduce the learning rate for Boolean parameters in the backbone to $\eta=6$. We employed the polynomial learning rate policy with $p = 0.9$ for all parameters. The cross-entropy loss was used for optimization. We did not employ auxiliary loss or knowledge distillation as these training techniques require additional computational cost, which is not in line with our efficient on-device training objective.

\subsubsection{Data sampling and augmentation}
\begin{table}[h!]
	\centering	
	\caption{Class per image and performance gap occurrence rates in \textsc{cityscapes} training set with naive \textsc{bool-aspp} design. \good{Class with low performance gap} and \bad{class with high performance gap}.}
	\vspace{.2cm}
	\resizebox{.45\textwidth}{!}{
	\begin{tabular}{@{}cSS[table-format=2.1]@{}}
		\toprule
		                  & {\textbf{Image ratio} (\%)} & {$\Delta$ \textbf{mIoU} (\%)} \\ \midrule
		\good{Road}       & 98.62              & 0.0                  \\
		\good{Sideway}    & 94.49              & 0.7                  \\
		\good{Building}   & 98.62              & 0.6                  \\
		Wall              & 32.61              & 7.4                  \\
		Fence             & 43.56              & 3.8                  \\
		Pole              & 99.13              & 1.8                  \\
		Light             & 55.73              & 6.5                  \\
		\good{Sign}       & 94.39              & 2.8                  \\
		\good{Vegetation} & 97.18              & 0.1                  \\
		\good{Terrain}    & 55.60              & 0.8                  \\
		\good{Sky}        & 90.29              & -0.2                 \\
		Person            & 78.76              & 1.5                  \\
		\bad{Rider}       & 34.39              & 7.4                  \\
		\good{Car}        & 95.19              & 0.3                  \\
		\bad{Truck}       & 12.07              & 6.9                  \\
		\bad{Bus}         & 9.21               & 12.8                 \\
		\bad{Train}       & 4.77               & 17.0                 \\
		\bad{Motor}       & 17.24              & 15.3                 \\
		bike              & 55.33              & 2.0                  \\ \bottomrule
	\end{tabular}
	}
	\label{tab:class_ratio_cs}
\end{table}

\begin{figure}[htb!]
	\centering
	\includegraphics[width=0.8\textwidth]{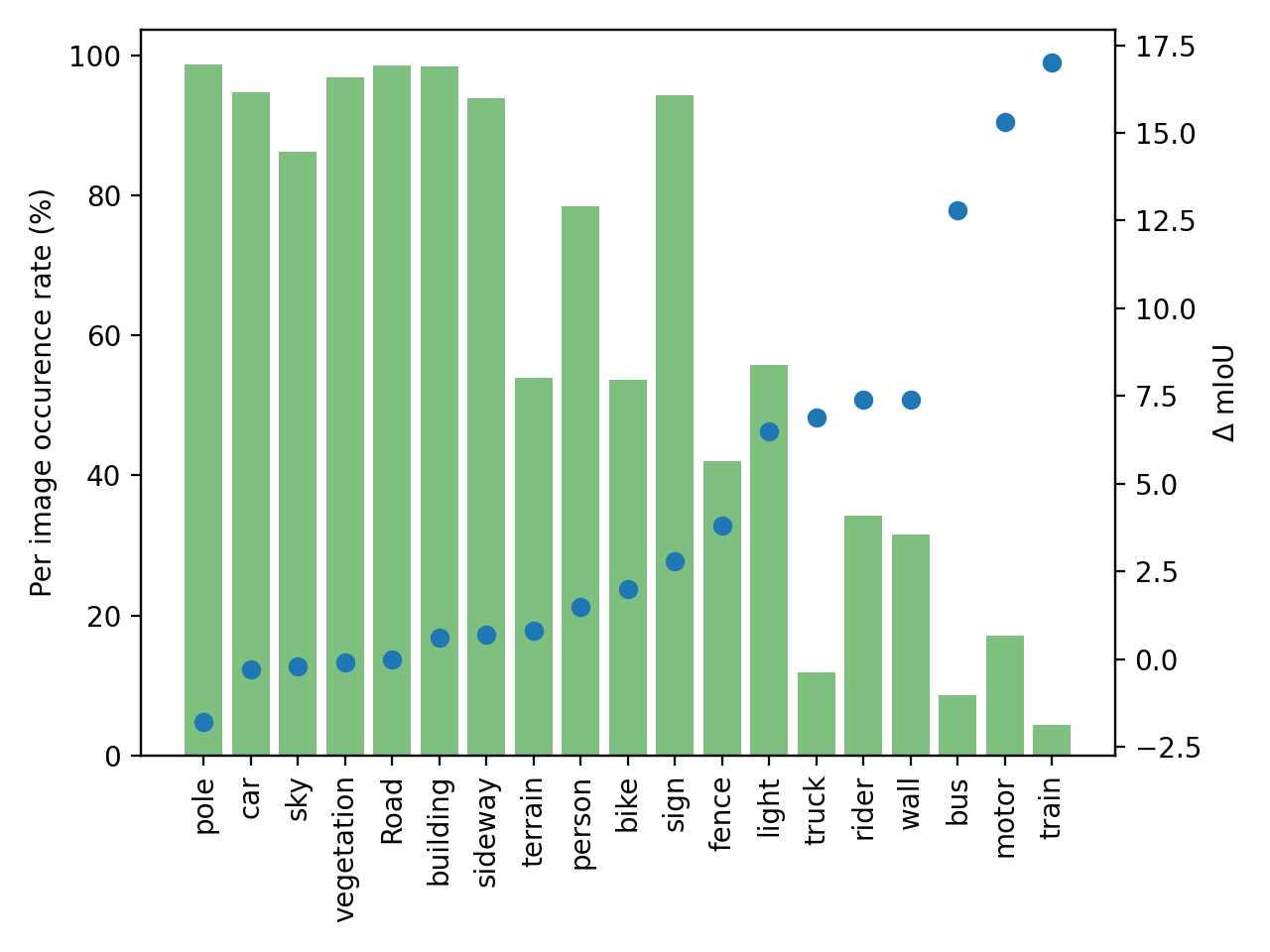}
	\vspace{.2cm}
	\caption{Class per image occurrence ratio and performance gap with naive \textsc{bool-aspp} design.}
	\label{fig:ratio_gap_cs}
\end{figure}

We aim to reproduce closely full-precision model performance in the semantic segmentation task with Boolean architecture and Boolean logic training. Due to the nature of the Boolean network, the common regularization method, e.g., weight decay, is not applicable. Moreover, with more trainable parameters,
the Boolean network can suffer from over-fitting. In particular, as shown in \Cref{tab:class_ratio_cs}, the imbalanced dataset for semantic segmentation aggravates the situation. There is a significant performance gap for several classes which has low occurrence rate, including \textit{rider (9.5\%), motor (11.2\%), bus (9.5\%), truck (6.9\%), train (17.0\%)}. We argue that the performance gap is due to the similarity between classes and the dataset's low occurrence rate, which is confirmed as shown in \Cref{fig:ratio_gap_cs}.

Data augmentation and sampling are thus critical for Boolean model training. Regarding data augmentation, we employed multi-scale scaling with a random scaling factor ranging from $0.5$ to $2$. We adopted a random horizontal flip with probability $p=0.5$ and color jittering. In addition, we used rare class sampling (RCS) \cite{hoyer2022daformer} to avoid the model over-fitting to frequent classes. For class $c$, the occurrence frequency in image $f_c$ is given by:
\begin{equation}
	f_c = \frac{\sum_{i=1}^N \one{c \in y_i}}{N},
\end{equation}
where $N$ is the number of samples and $y_i$ is the set of classes existing in sample $i$. The sampling probability of class $c$ is thus defined as:
\begin{equation}
	p_c = \frac{\exp\pren{\frac{1-f_c}{T}}}{\sum_{c'=1}^K\exp\pren{\frac{1-f_{c'}}{T}}},
\end{equation}
where $K$ is the number of classes, and $T$ is a hyper-parameter for sampling rate balancing. In particular, for the \textsc{cityscapes} dataset, we selected $T = 0.5$.

\begin{figure}[h!]
	\centering
	\begin{minipage}{0.24\textwidth}
		\centering
		\includegraphics[width=\textwidth]{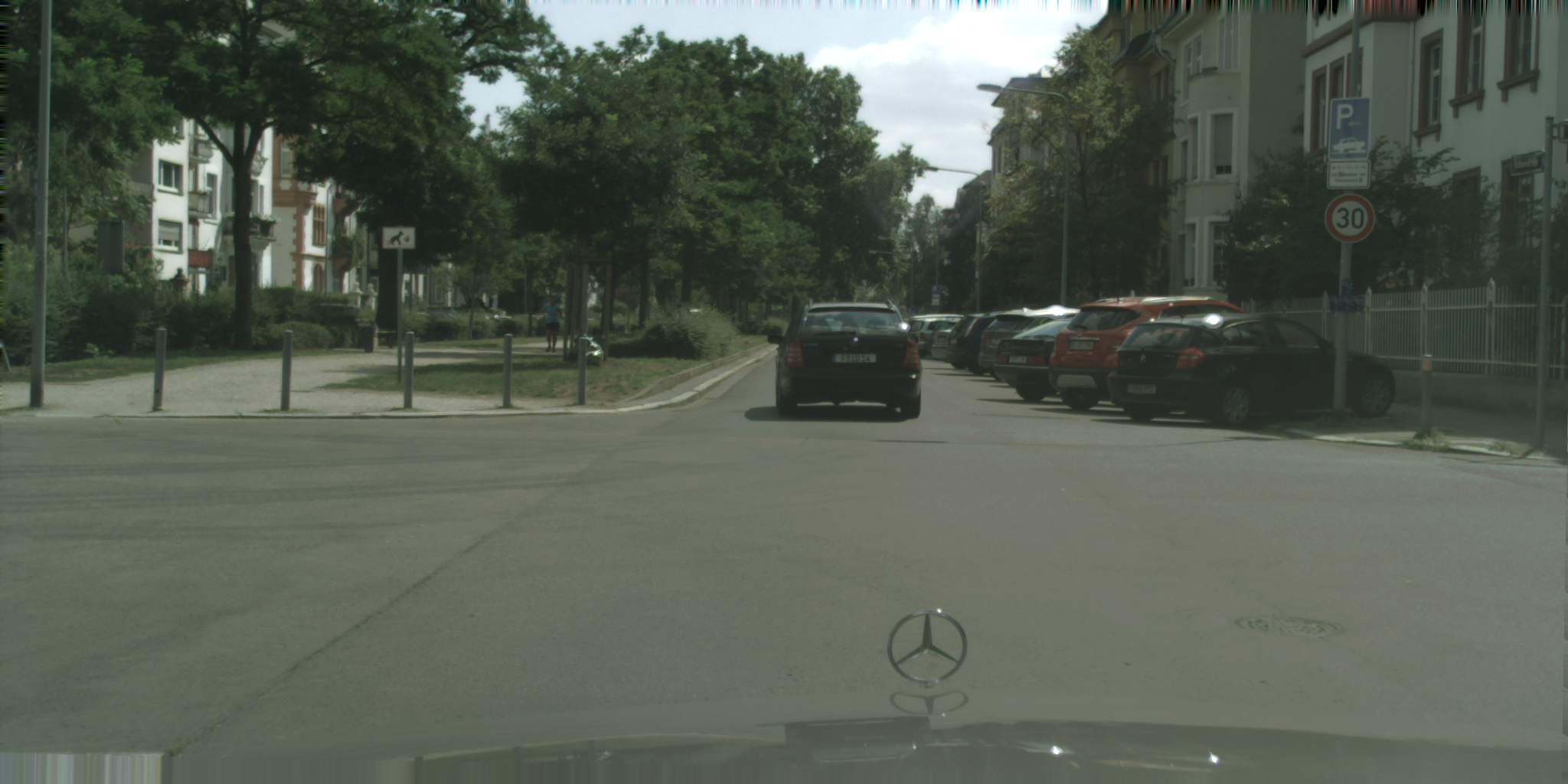}
	\end{minipage}
	\hfil
	\begin{minipage}{0.24\textwidth}
		\centering
		\includegraphics[width=\textwidth]{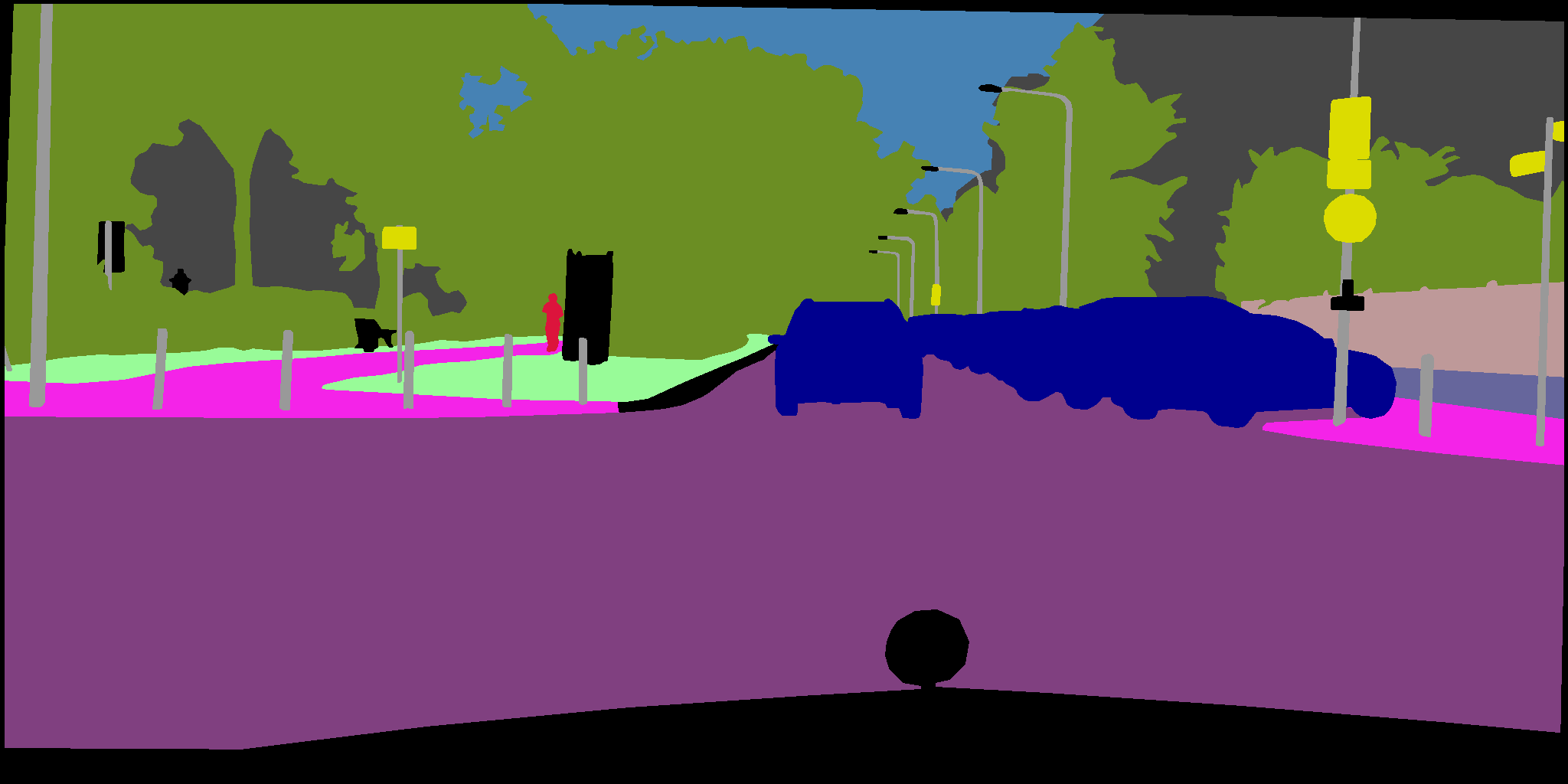}
	\end{minipage}
	\hfil
	\begin{minipage}{0.24\textwidth}
		\centering
		\includegraphics[width=\textwidth]{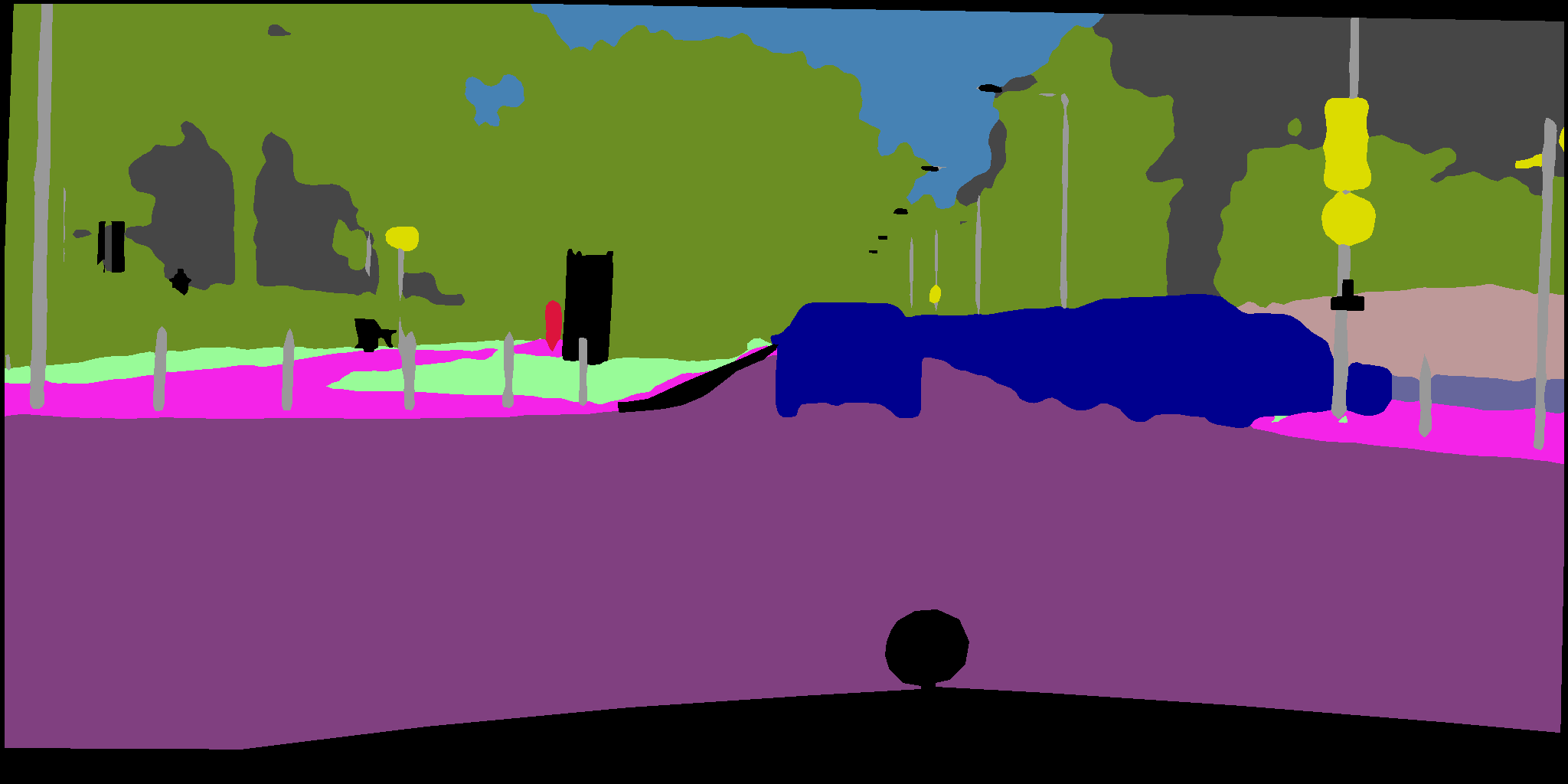}
	\end{minipage}
	\hfil
	\begin{minipage}{0.24\textwidth}
		\centering
		\includegraphics[width=\textwidth]{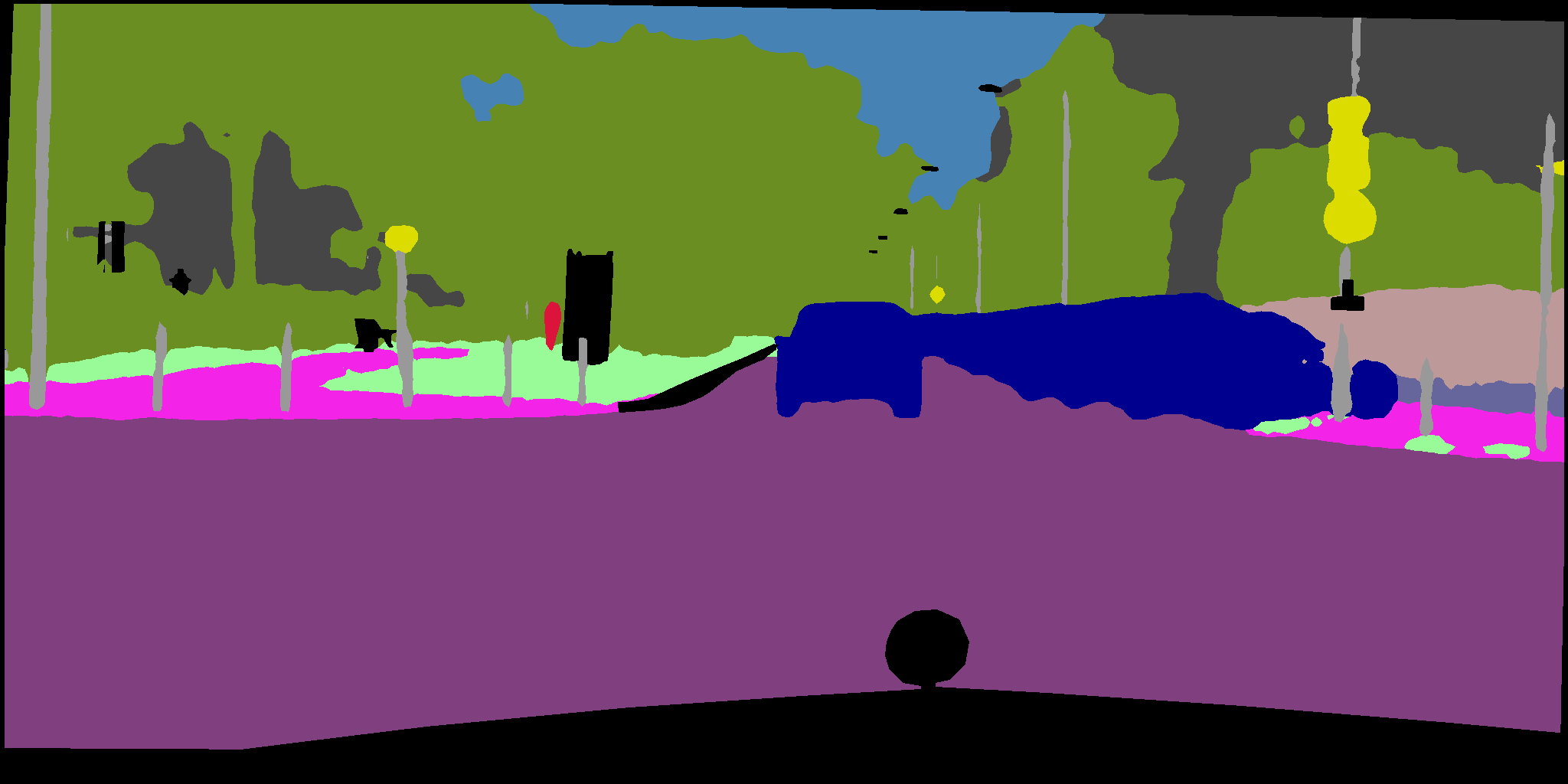}
	\end{minipage}

	\begin{minipage}{0.24\textwidth}
		\centering
		\includegraphics[width=\textwidth]{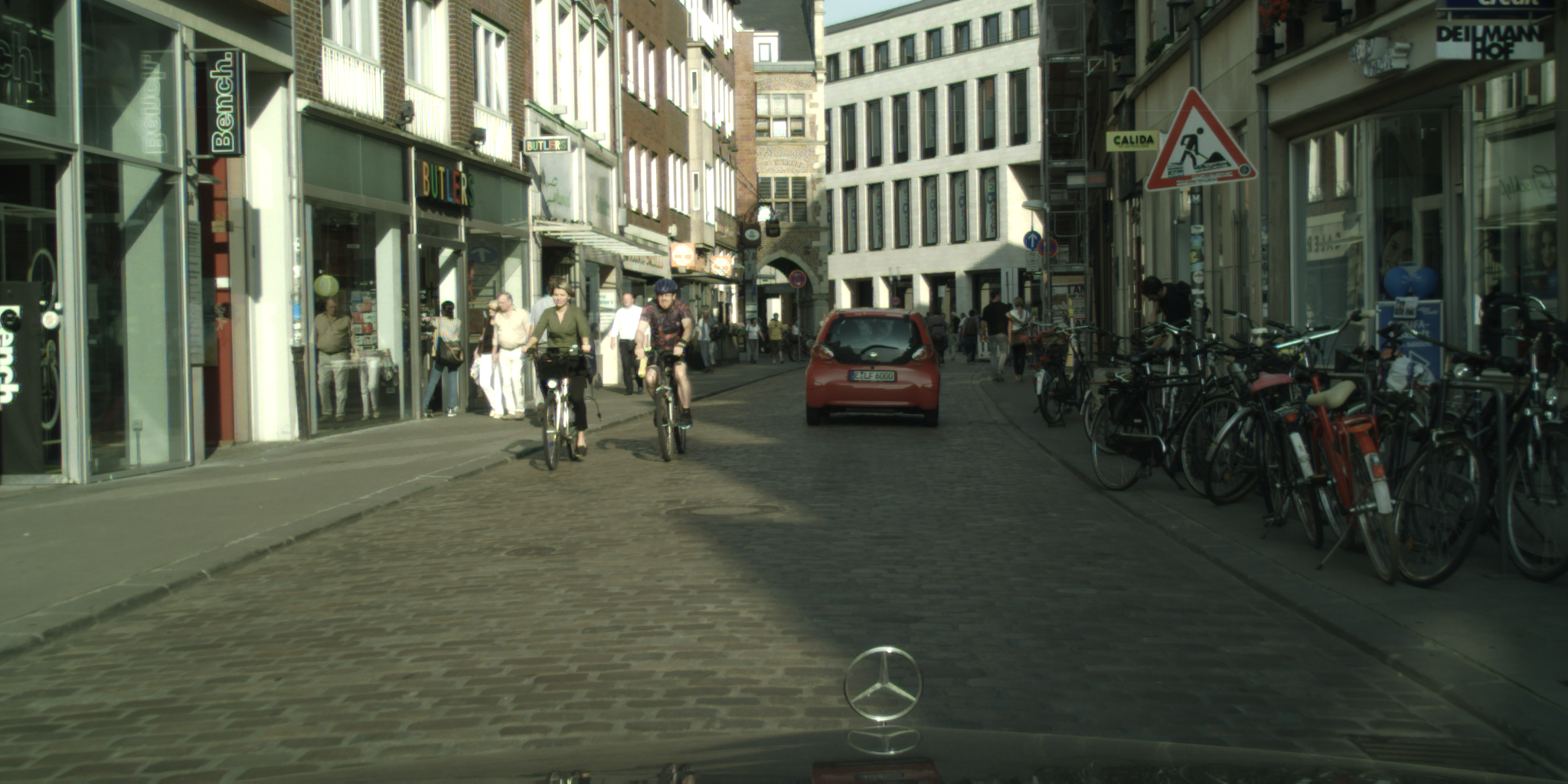}
	\end{minipage}
	\hfil
	\begin{minipage}{0.24\textwidth}
		\centering
		\includegraphics[width=\textwidth]{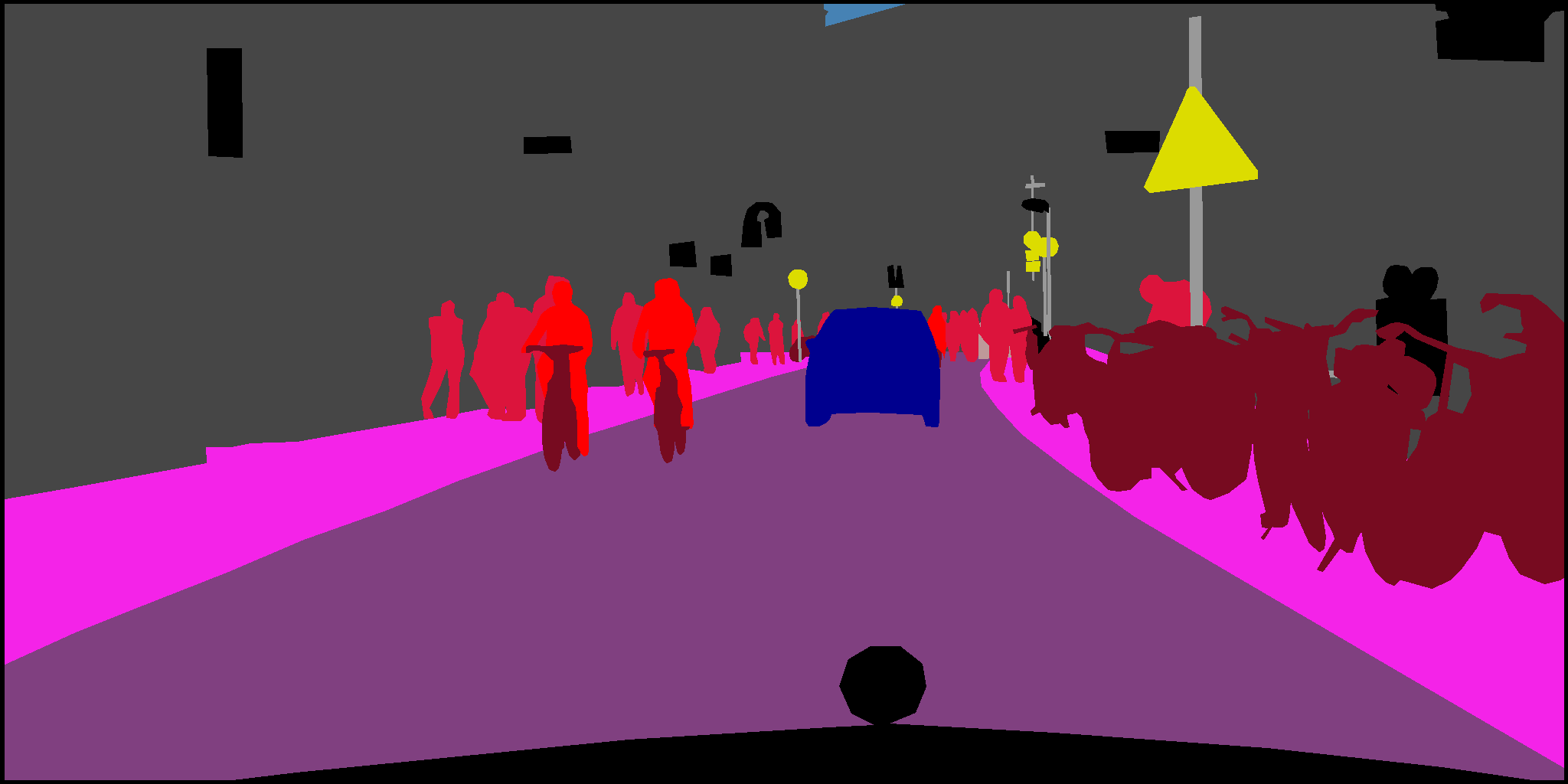}
	\end{minipage}
	\hfil
	\begin{minipage}{0.24\textwidth}
		\centering
		\includegraphics[width=\textwidth]{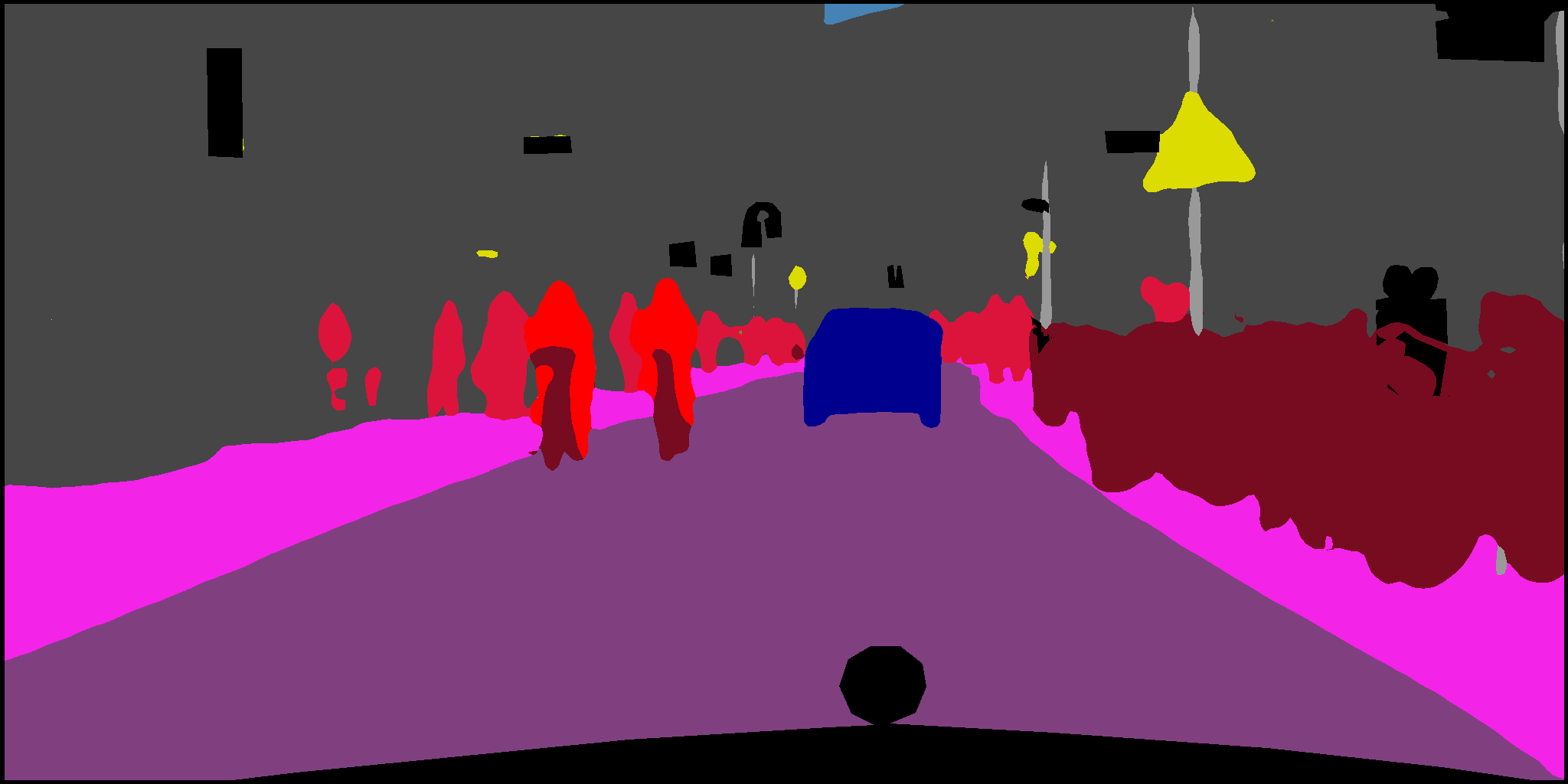}
	\end{minipage}
	\hfil
	\begin{minipage}{0.24\textwidth}
		\centering
		\includegraphics[width=\textwidth]{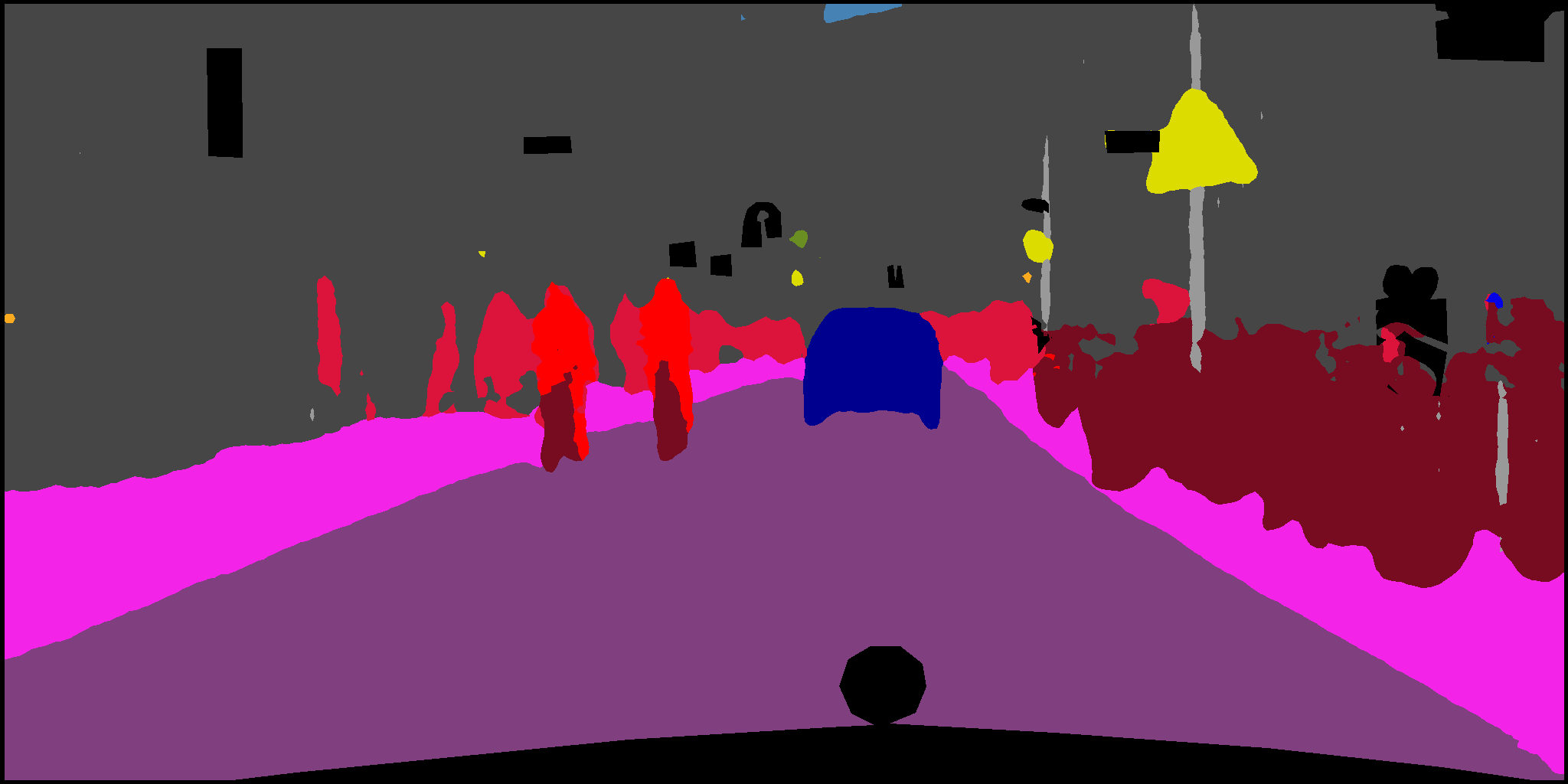}
	\end{minipage}

	\begin{minipage}{0.24\textwidth}
		\centering
		\includegraphics[width=\textwidth]{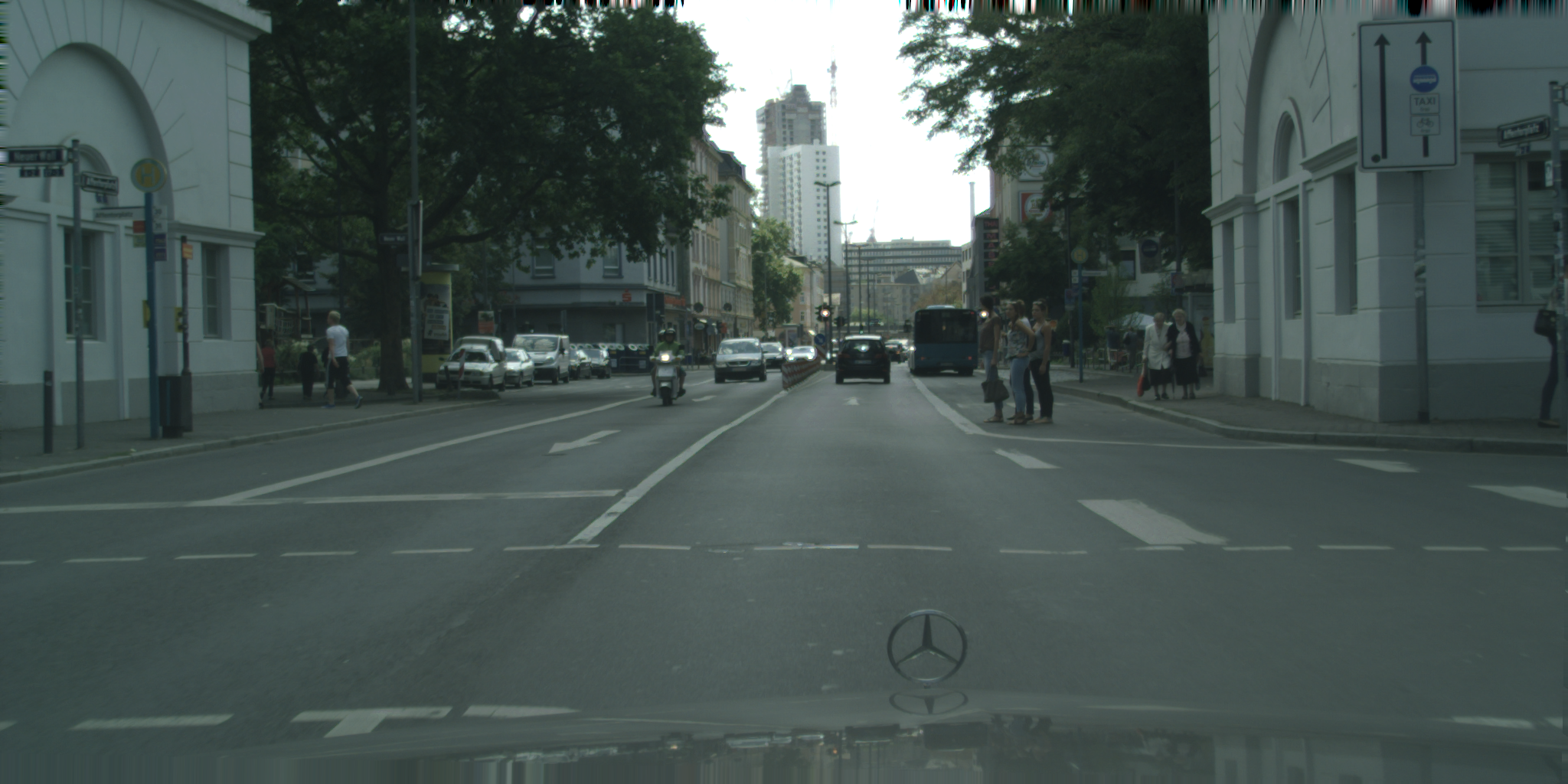}
	\end{minipage}
	\hfil
	\begin{minipage}{0.24\textwidth}
		\centering
		\includegraphics[width=\textwidth]{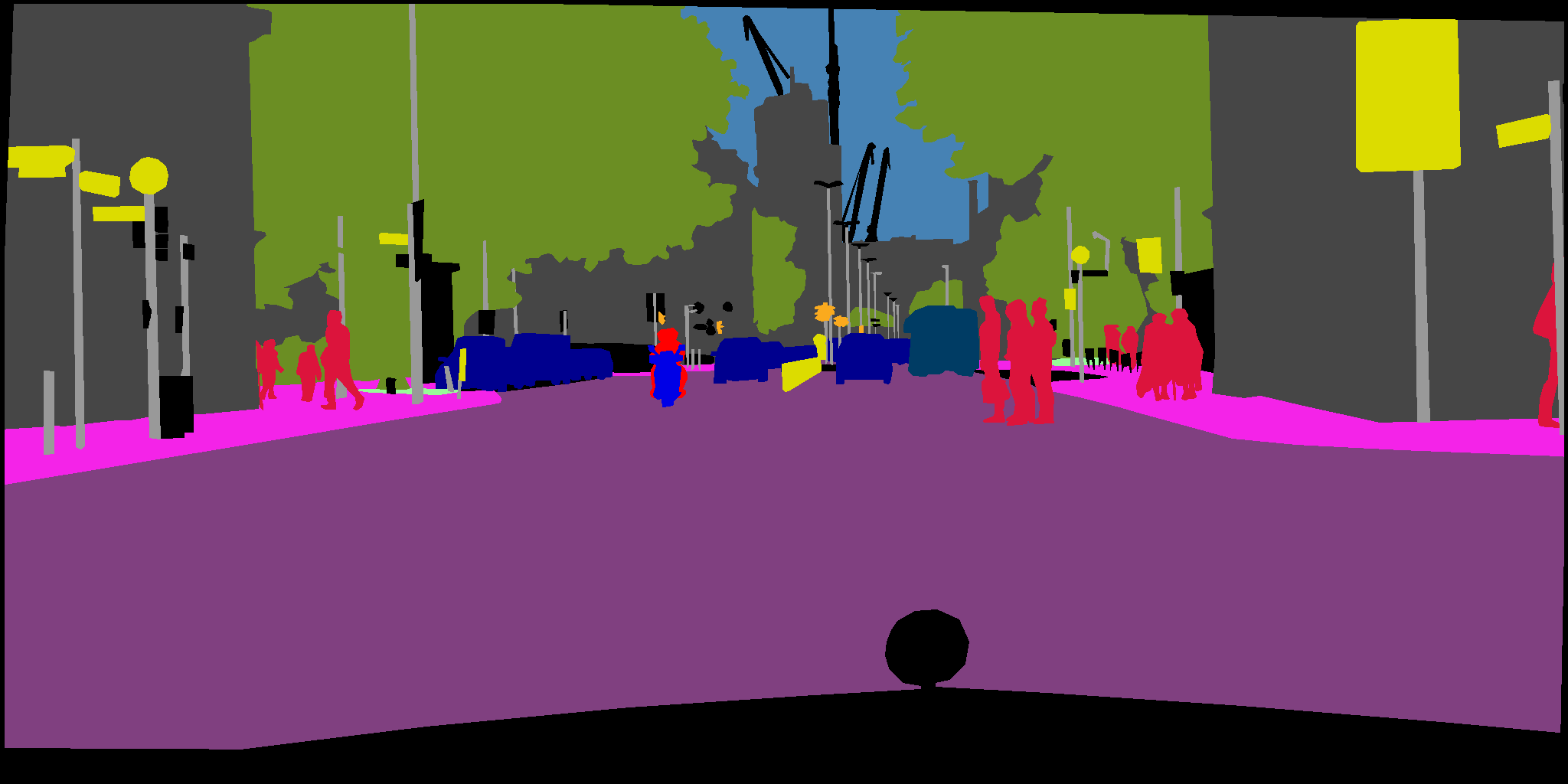}
	\end{minipage}
	\hfil
	\begin{minipage}{0.24\textwidth}
		\centering
		\includegraphics[width=\textwidth]{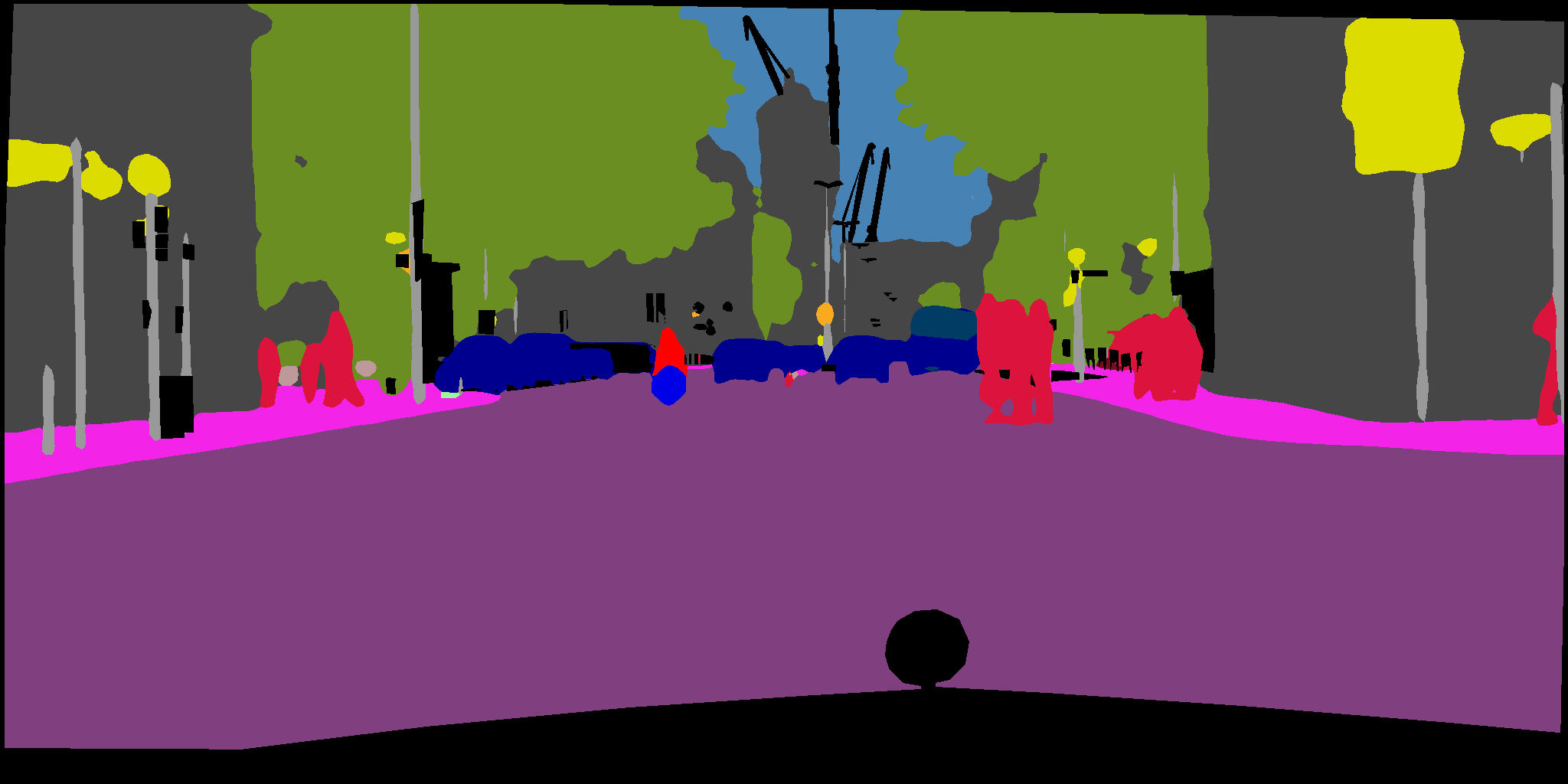}
	\end{minipage}
	\hfil
	\begin{minipage}{0.24\textwidth}
		\centering
		\includegraphics[width=\textwidth]{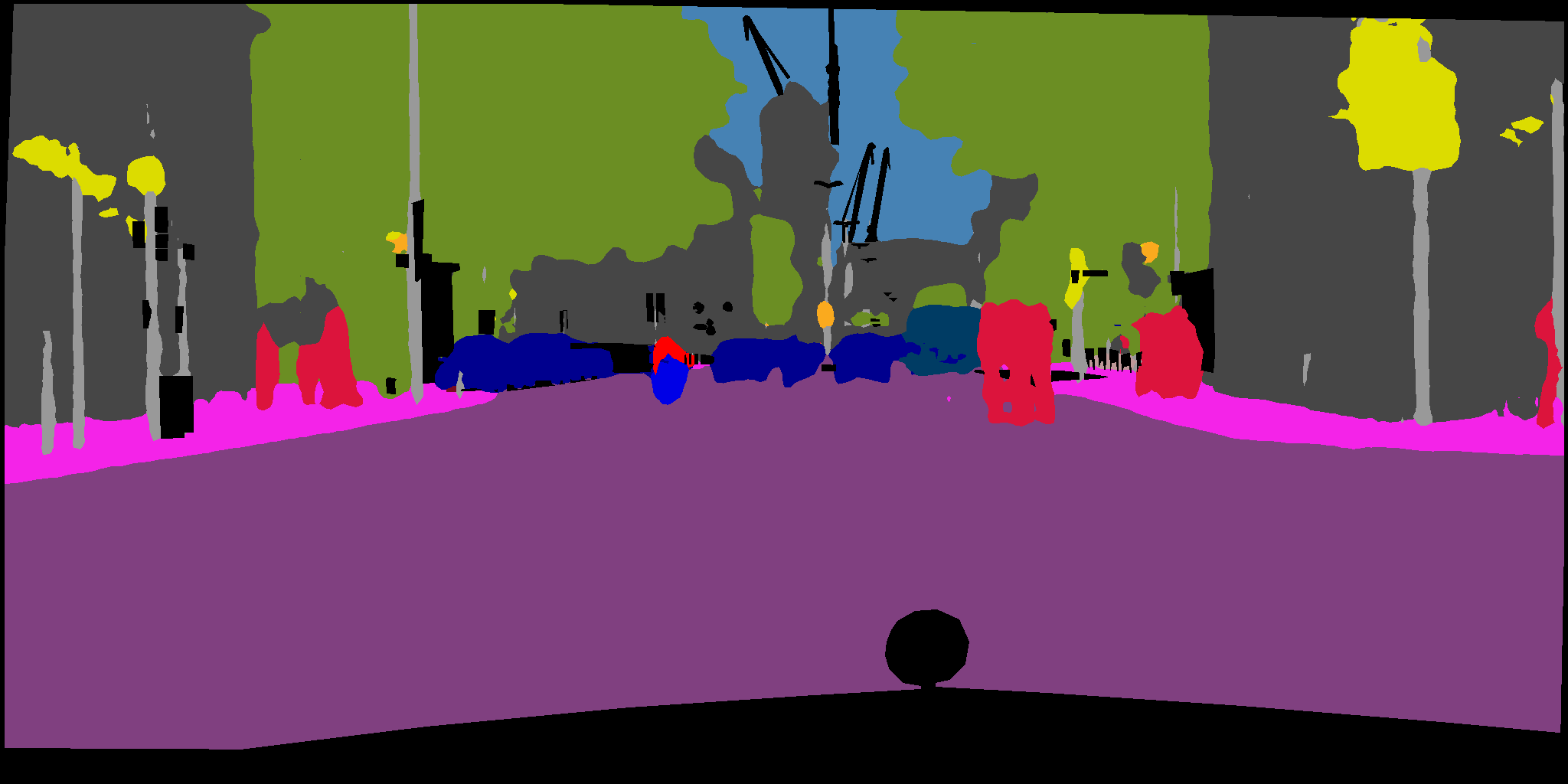}
	\end{minipage}

	\begin{minipage}{0.24\textwidth}
		\centering
		Input image
	\end{minipage}
	\hfil
	\begin{minipage}{0.24\textwidth}
		\centering
		Ground truth
	\end{minipage}
	\hfil
	\begin{minipage}{0.24\textwidth}
		\centering
		Full-precision
	\end{minipage}
	\hfil
	\begin{minipage}{0.24\textwidth}

		\centering
		\ours (Ours)
	\end{minipage}
	\vspace{.2cm}
	\caption{Qualitative comparison of Boolean model on \textsc{cityscapes} validation set.}
	\label{fig:quali_cs}

\end{figure}

\subsubsection{Qualitative analysis on cityscapes validation set}

The qualitative results of our Boolean network and the full-precision based are demonstrated in \Cref{fig:quali_cs}. Despite the loss of model capacity, the proposed Boolean network trained with Boolean logic optimizer has comparable performance with large objects in the frequent classes, even in the complicated scene.

\subsubsection{More experiments on semantic segmentation}

\begin{table}[b!]
	\centering	
	\caption{Class-wise IoU performance on \textsc{cityscapes} validation set. }
	\vspace{.2cm}
	\resizebox{\textwidth}{!}{%
		\begin{tabular}{@{}ccccccccccccccccccccc@{}}
			\toprule
			Methods         & \rotatebox{90}{road} & \rotatebox{90}{sideway} & \rotatebox{90}{building} & \rotatebox{90}{wall} & \rotatebox{90}{fence} & \rotatebox{90}{pole} & \rotatebox{90}{light} & \rotatebox{90}{sign} & \rotatebox{90}{vegetation} & \rotatebox{90}{terrain} & \rotatebox{90}{sky} & \rotatebox{90}{person} & \rotatebox{90}{rider} & \rotatebox{90}{car} & \rotatebox{90}{truck} & \rotatebox{90}{bus} & \rotatebox{90}{train} & \rotatebox{90}{motor} & \rotatebox{90}{bike} & mIoU         \\ \midrule\midrule
			FP baseline     & 97.3                 & 79.8                    & 90.1                     & 48.5                 & 55.0                  & 49.4                 & 59.2                  & 69.                  & 90.0                       & 57.5                    & 92.4                & 74.3                   & 54.6                  & 91.2                & 61.4                  & 78.3                & 66.6                  & 58.0                  & 70.8                 & 70.7         \\ \midrule
			Naive Bool ASPP & 97.3                 & 79.1                    & 89.5                     & 41.1                 & 51.2                  & 51.2                 & 52.7                  & 66.2                 & 90.1                       & 56.7                    & 92.6                & 72.8                   & 47.2                  & 91.5                & 54.5                  & 65.5                & 49.6                  & 42.7                  & 68.8                 & 66.3         \\
			$\Delta$        & \textbf{0.0}         & \textbf{0.7}            & 0.6                      & 7.4                  & 3.8                   & -1.8                 & 6.5                   & 2.8                  & -0.1                       & 0.8                     & -0.2                & \textbf{1.5}           & \textbf{7.4}          & -0.3                & 6.9                   & 12.8                & \textbf{17.0}         & 15.3                  & 2.0                  & 4.4          \\ \midrule
			\ours [Ours]            & 97.1                 & 78.                     & 89.8                     & 46.2                 & 51.3                  & 52.7                 & 53.3                  & 66.5                 & 90.2                       & 58.                     & 92.7                & 72.6                   & 45.1                  & 91.9                & 61.1                  & 68.8                & 48.7                  & 46.8                  & 69.1                 & 67.4         \\
			$\Delta$        & 0.2                  & 1.8                     & \textbf{0.3}             & \textbf{2.3}         & \textbf{3.7}          & \textbf{-3.3}        & \textbf{5.9}          & \textbf{2.5}         & \textbf{-0.2}              & \textbf{-0.5}           & \textbf{-0.3}       & 1.7                    & 9.5                   & \textbf{-0.7}       & \textbf{0.3}          & \textbf{9.5}        & 17.9                  & \textbf{11.2}         & \textbf{1.7}         & \textbf{3.3} \\ \bottomrule
		\end{tabular}%
	}
	\label{tab:comp_fp_bool}
\end{table}

We evaluated the effectiveness of \textsc{bool-aspp} by investigating the per-class performance gap to the full-precision model. As demonstrated in \Cref{tab:comp_fp_bool}, a significant gap exists between the Boolean architecture with naive \textsc{bool-aspp} design; i.e., using Boolean activations for \textsc{aspp} module as illustrated in \Cref{fig:nv_binaspp_gap}. However, the gap could be reduced by using \textsc{bool-aspp} and RCS. In particular, the \textsc{bool-aspp} improves the IoU of \textit{truck} from $54.5\%$ to $64.1\%$ and \textit{bus} from $65.5\%$ to $68.8\%$, \textit{bike} from $68.8\%$ to $69.1\%$ and \textit{motor} from $42.8\%$ to $46.8\%$. This indicates that combining proposed \textsc{bool-aspp} and RCS improves the model performance on low occurrence classes as well as similar classes with which are easy to be confused.

\subsubsection{Validation on \textsc{pascal voc 2012} dataset}

We also evaluated our Boolean model on the $21$-class \textsc{pascal voc 2012} dataset with augmented additional annotated data containing $10,582$, $1,449$, and $1,456$ images in training, validation, and test set, respectively. The same setting is used as in the experiments on the \textsc{cityscapes} dataset, except the model was trained for $60$ epochs.

As shown in \Cref{tab:seg_voc_val}, our model with fully Boolean logic training paradigm, i.e., without any additional intermediate latent weight, achieved comparable performance as the state-of-the-art latent-weight-based method. Our Boolean model improved performance by incorporating multi-resolution feature extraction modules to $67.3\%$ mIoU.

\begin{table}[htb!]
	\centering	
	\caption{Performance on \textsc{pascal voc 2012} val set.}
	\resizebox{.57\textwidth}{!}{
	\begin{tabular}{@{}cccccc@{}}
		\toprule
		\textbf{Seg. head}                                            & \textbf{Model}                                 & \textbf{mIoU} (\%)     & $\Delta$ \\ \midrule\midrule
		\multirow{3}{*}{\textsc{fcn-32s} \cite{long2015fully}}        & \textsc{fp} baseline                           & 64.9          & -        \\
		                                                     & \textsc{group-net} \cite{zhuang2019structured} & 60.5          & 4.4      \\
		                                                     & \ours [Ours]                                  & 60.1          & 4.8      \\ \midrule
		\multirow{2}{*}{\textsc{deeplabv3} \cite{chen2017rethinking}} & \textsc{fp} baseline                           & 72.1          & -        \\
		                                                     & \ours [Ours]                                  & \textbf{67.3} & 4.8      \\ \bottomrule
	\end{tabular}
	}
	\label{tab:seg_voc_val}
\end{table}

\subsection{Boolean BERT Fine-tuning} \label{supp:bert}
We conducted experiments on the \textsc{bert} model \cite{DevlinCLT19} drawing on the experimental framework proposed in \textsc{bit} \cite{Liu2022d}. For this experiment, our goal was to fine-tune the original pre-trained \textsc{bert} model using Boolean precision. To validate our method we used the \textsc{glue} benchmark \cite{Wang2019d} with 8 datasets. For our experiments, we modified the baseline \textsc{fp} architecture using the proposed methodology \Cref{alg:code_optim}. That is, core operations within the transformer are substituted by native Boolean components, e.g. Boolean activations and Boolean linear layers. The \textsc{fp} parameters were optimized using the Adam optimizer with the learning rates proposed in \cite{Liu2022d}. Correspondingly for Boolean weights, we used our Boolean optimizer with learning rate $\eta=100$.

\section{Energy Estimation}
\label{appendix:complexity}

Energy consumption is a fundamental metric for measuring hardware complexity.
However, it requires specific knowledge of computing systems and makes it hard to estimate.
Few results are available, though experimental-based and limited to specific tested models \citep[see, e.g.,][]{Gao2020, Shao2013,Mei2014,Bianco2018,Canziani2016,GarciaMartin2019}. Although experimental evaluation is precise, it requires considerable implementation efforts while not generalizing.
In addition, most relevant works are only limited to inference and not training \citep{Chen2016a, kwon2019understanding, yang2020interstellar}.

\subsection{Hardware Specification}

\paragraph{Ascend architecture.}
We intend to estimate the training energy consumption on Ascend chip architecture introduced in \cite{Liao2021} and dedicated to DNN computing.
The core design of Ascend is described in \cite{Liao2021}. 
Essentially, it introduces a 3D (cube) computing unit, providing the bulk of high-intensity computation and increasing data reuse. 
On the other hand, it provides multiple levels of on-chip memory.
In particular, memory L0, which is nearest to the computing cube, is tripled to boost further near-memory computing capability, namely L0-A dedicated to the left-hand-side (LHS) input data, L0-B dedicated to RHS input data, and L0-C for the output. For instance, in a convolution, L0-A, L0-B, and L0-C correspond to the input feature maps (\textsc{ifmaps}), \textsc{filters}, and output feature maps (\textsc{ofmaps}), respectively. 
In addition, the output results going through L0-C can be processed by a Vector Unit for in-place operations such as normalization and activation. 
\Cref{tab:energy_cost} shows energy efficiency and capacity of the memory hierarchy of a commercial Ascend architecture \cite{Liao2021}. 

\begin{table}[H]
	\centering
    \caption{Memory hierarchy and energy efficiency (EE) of an Ascend core \citep{Liao2021} used in our evaluation.}	
    \label{tab:energy_cost}

	\vspace{1ex}
	
    \resizebox{.7\columnwidth}{!}{	
		\begin{tabular}{ccccccc } 
			\toprule
			& \textbf{L3 (DRAM)} & \textbf{L2} & \textbf{L1} & \textbf{L0-A} & \textbf{L0-B} & \textbf{L0-C} \\ 
			\midrule\midrule
			EE [GBPS/mW] & 0.02 & 0.2 & 0.4 & 4.9 & 3.5 & 5.4 \\
			\midrule
			Capacity [KB] & $-$ & 8192 & 1024 & 64 & 64 & 256 \\				
			\bottomrule
		\end{tabular}
	}
\end{table}

\paragraph{Nvidia architecture.}
We also have estimated the energy consumption for	 the Nvidia GPU (Tesla V100). Similarly, this architecture utilizes different memory levels with varying read and write energy characteristics. For instance, the L2 cache size is 6 MB per GPU. The L1 cache is 64 KB per Streaming Multiprocessor (SM). Each Tesla V100 GPU has 80 SMs, so the total L1 cache is 5120 KB (5 MB). However, specific details on the read and write energy consumption for each memory level of the Tesla V100 GPU are proprietary and not publicly disclosed by Nvidia. Thus, we show the normalized energy consumption relative to the computation of a MAC at the arithmetic logic unit (ALU).
\cref{tab:nvidia_cost} shows the numbers for each storage level, which are extracted from a commercial 65 nm process \citep{Chen2016}.

\begin{table}[H]
    \centering
    \caption{Normalized energy cost relative to the computation of one MAC operation at ALU.}
    \label{tab:nvidia_cost}

    \vspace{1ex}

    \resizebox{.48\columnwidth}{!}{
        \begin{tabular}{ccccc} \toprule
            \textbf{DRAM} & \textbf{L2} & \textbf{L1} & \textbf{RF} & \textbf{1 MAC at ALU} \\ \midrule\midrule
            $200\times$   & $6\times$   & $2\times$   & $1\times$   & $1\times$             \\ \bottomrule
        \end{tabular}
    }
\end{table}

\subsection{Compute Energy}
Energy consumption is the sum of compute and memory energies.
\emph{Compute energy} is simply given by the number of arithmetic operations multiplied by their unit cost.
The number of arithmetic operations is directly determined from the layer's parameters. For the Ascend architecture, their unit cost is obtained by considering the compute efficiency at 1.7 TOPS/W \cite{Liao2021}. 
For Boolean logic operations, we follow the usual estimation that ADD INT-$n$ costs $(2n-1)$ logic operations where $n$ stands for bitwidth.

\subsection{Memory Energy}
On the other hand, \emph{memory energy} is all consumed for moving data between their storage through memory levels and the computing unit during the entire lifetime of the process.
Since energy consumed at each memory level is given by the number of data accesses to that level times per-access energy cost, it consists in determining the number of accesses to each level of all data streams (i.e., LHS, RHS, Output).
Besides taking into account the hardware architecture and memory hierarchy of chip, our approach to quantifying memory energy is based on existing methods \cite{Chen2016a, Sze2017, kwon2019understanding, yang2020interstellar, Horowitz2014,Yang2017a} for dataflow and energy evaluation.
Given the layer parameters and memory hierarchy, it amounts to:
\begin{enumerate}
	\item \emph{Tiling}: determining the tiling strategy for allocating data streams on each memory level.
	\item \emph{Movement}: specifying how data streams are reused or kept stationary to determine their access numbers.
\end{enumerate}
In the following, we present our method for the forward and backward passes by taking the example of a convolution layer, as convolutions are the main components of \glspl{CNN} and the primary source of complexity due to their high data reuse.
The parameters of 2D convolution layer are summarized in  \Cref{tab:shape_parameters}. 
Here, we denote \textsc{ifmaps}, \textsc{filters}, and \textsc{ofmaps} by $I$, $F$, and $O$, respectively. 

\begin{table}[H]
	\centering
	\caption{Shape parameters of a convolution layer.}
	\label{tab:shape_parameters}

    \vspace{1ex}

	\resizebox{0.47\columnwidth}{!}{
		{\setlength{\tabcolsep}{10pt}
			{
				\begin{tabular}{ cc } 
					\toprule
					\textbf{Parameter} & \textbf{Description} \\ 
					\midrule\midrule
					$N$ & batch size\\
					\midrule
					$M$ & number of \textsc{\textsc{ofmaps}} channels\\
					\midrule
					$C$ & number of \textsc{ifmaps} channels\\
					\midrule
					$H^I/W^I$ & \textsc{ifmaps} plane height/width\\
					\midrule
					
					$H^F/W^F$ & \textsc{filters} plane height/width\\
					\midrule
					$H^O/W^O$ & \textsc{\textsc{ofmaps}} plane height/width\\
					\bottomrule
		\end{tabular}}}
	}
	
\end{table}

\subsubsection{Tiling}
Since the \textsc{ifmaps} and \textsc{filters} are usually too large to be stored in buffers, the tiling strategy is aimed at efficiently transferring them to the computing unit. Determining tiling parameters, which are summarized in \Cref{tab:tile_parameters}, is an NP-Hard problem \cite{yang2020interstellar}.
\begin{table}[H]
	\centering
    \caption{Tiling parameters of a convolution layer.}
    \label{tab:tile_parameters}

    \vspace{1ex}

    \resizebox{0.58\columnwidth}{!}{
	{
		{
			\begin{tabular}{ cc } 
				\toprule
				\textbf{Parameter} & \textbf{Description} \\
				\midrule\midrule
				$M_2$ & number of tiling weights in L2 buffer \\
				\midrule
				$M_1$ & number of tiling weights in L1 buffer \\
				\midrule
				$M_0$ & number of tiling weights in L0-B buffer \\
				\midrule
				$N_2$ & number of tiling \textsc{ifmaps} in L2 buffer \\
				\midrule
				$N_1$ & number of tiling \textsc{ifmaps} in L1 buffer \\
				\midrule
				$N_0$ & number of tiling \textsc{ifmaps} in L0-A buffer \\
				\midrule
				$H^{I}_2/W^{I}_2$ & height/width of tiling \textsc{ifmaps} in L2 buffer \\
				\midrule
				$H^{I}_1/W^{I}_1$ & height/width of tiling \textsc{ifmaps} in L2 buffer \\
				\midrule
				$H^{I}_0/W^{I}_0$ & height/width of tiling \textsc{ifmaps} in L0-A buffer\\
				\bottomrule
	\end{tabular}}}
    }
    
\end{table}

An iterative search over possibilities subject to memory capacity constraint provides tiling combinations of \textsc{ifmaps} and \textsc{filters} on each memory level. Different approaches can be used and  \Cref{algo:TilingStrategy} shows an example that explores the best tiling parameters subjected to maximizing the buffer utilization and near compute stationary (\ie, as much reuse as possible to reduce the number of accesses to higher levels).
Therein, the amount of data stored in level $\mathit{i}$ is calculated as:
\begin{equation}
	\begin{aligned}
		Q^{I}_i & = N_i \times C_i \times H^{I}_i \times W^{I}_i \times b^I, \\
		Q^{F}_i & = M_i \times C_i \times H^{F} \times W^{F} \times b^F,
	\end{aligned}
\end{equation}
where $Q^{I}_i/Q^{F}_i$ and $b^{I}/b^{F}$ represent the memory and bitwidth of \textsc{ifmaps}/\textsc{filters}, respectively.

\begin{algorithm}[H]
	\caption{Loop tiling strategy in the $i$th level}
	\label{algo:TilingStrategy}
	\SetKwBlock{Initialize}{Initialize}{end}
	\SetAlgoLined
	\SetNoFillComment
	\KwIn{tiling parameters of \textsc{ifmaps} and \textsc{filters} at level $i+1$, and buffer capacity of level $i$.}
	\KwOut{tiling parameters of \textsc{ifmaps} and \textsc{filters} at level $i$.}
	\Initialize{$\Ec^{\textrm{min}} := \infty$;}
	\For{$M_i \gets M_{i+1}$ to $1$}
	{
		\For{$N_i \gets N_{i+1}$ to $1$}
		{
			\For{$H^{I}_i \gets H^{I}_{i+1}$ to $H^{F}$}
			{
				\For{$W^{I}_i \gets W^{I}_{i+1}$ to $W^{F}$}
				{
					Calculate $Q_i$, the required amount of \textsc{ifmaps} and \textsc{filters} to be stored in the $i$th level of capacity $Q^{\textrm{max}}_i$;\\
					Calculate $\Ec_i$, the energy cost of moving \textsc{ifmaps} and \textsc{filters} from the $i$th level;\\
					\If{($Q_i \leq Q^{\textrm{max}}_i$) and ($\Ec_i < \Ec^{\textrm{min}}$)}
					{
						Retain tiling parameters as best;\\
						$\Ec^{\textrm{min}} \gets \Ec_i$;			
					}
				}
			}
		}
	}
	\KwRet{Best tiling parameters}
\end{algorithm}

\subsubsection{Data movement}
For data movement, at level L0, several data stationary strategies, called \emph{dataflows}, have been proposed in the literature, notably weight, input, output, and row stationary \cite{Chen2016a}. Since Ascend chip provides tripled L0 buffers, partial sums can be directly stationary in the computing cube, hence equivalent to output stationary whose implementation is described in \cite{du2015shidiannao}. 
For the remaining levels, our question of interest is how to move \textsc{ifmaps} block $[N_{i+1}, C_{i+1}, H^I_{i+1}, W^I_{i+1}]$ and \textsc{filters} block $[M_{i+1}, C_{i+1}, H^F, W^F]$ from level $i+1$ to level $i$ efficiently. Considering that:
\begin{itemize}
	\item \textsc{ifmaps} are reused by the \textsc{filters} over output channels,
	\item \textsc{filters} are reused over the \textsc{ifmaps} spatial dimensions,
	\item \textsc{filters} are reused over the batch dimension,
	\item \textsc{ifmaps} are usually very large whereas \textsc{filters} are small,
\end{itemize}
the strategy that we follow is to keep \textsc{filters} stationary on level $i$ and cycle through \textsc{ifmaps} when fetching them from level $i+1$ as shown in \Cref{algo:DataMovement}. Therein, \textsc{filters} and \textsc{ifmaps} are read block-by-block of their tiling sizes, \ie, \textsc{filters} block $[M_i, C_i, H^F, W^F]$ and \textsc{ifmaps} block  $[N_i, C_i, H^I_i, W^I_i]$. Hence, the number of filter accesses to level $i+1$ is 1 whereas the number of \textsc{ifmaps} accesses to level $i+1$ equals the number of level-$i$ \textsc{filters} blocks contained in level $i+1$. Following this method, the number of accesses to memory levels of each data stream can be determined. Hence, denote by: 
\begin{itemize}
	\item $n_i^d$: number of accesses to level $i$ of data $d$,
	\item $\varepsilon_i$: energy cost of accessing level $i$, given as the inverse of energy efficiency from \Cref{tab:energy_cost}.
\end{itemize} 
Following \cite{Chen2016a}, the energy cost of moving data $d$ from DRAM (L3) into the cube is given as:
\begin{equation}\label{eq:data_reuse_energy}
	\Ec^d = n^d_3 \varepsilon_3 + n^d_3 n^d_2 \varepsilon_2 + n^d_3 n^d_2 n^d_1 \varepsilon_1 + n^d_3 n^d_2 n^d_1 n^d_0 \varepsilon_0.
\end{equation}
Regarding the output partial sums, the number of accumulations at each level is defined as the number of times each data goes in and out of its lower-cost levels during its lifetime. Its data movement energy is then given as: 
\begin{multline}\label{eq:psum_energy}
	\Ec^{O} = (2n^O_3-1) \varepsilon_3 + 2 n^O_3 (n^O_2-1) \varepsilon_2 + 2n^O_3 n^O_2 (n^O_1-1) \varepsilon_1 + 2n^O_3 n^O_2 n^O_1 (n^O_0-1) \varepsilon_0,
\end{multline}
where factor of 2 accounts for both reads and writes and the subtraction of 1 is because we have only one write in the beginning \cite{Chen2016a}.

\begin{algorithm}[H]
	\caption{Data movement from $i+1$ to $i$ levels}
	\label{algo:DataMovement}
	\SetAlgoLined
	\SetNoFillComment
	\KwIn{tiling parameters of \textsc{ifmaps} and \textsc{filters} at levels $i+1$ and $i$.}
	\Repeat{all \textsc{filters} are read into level $i$}{
		read next \textsc{filters} block of size $[M_i, C_i, H^F, W^F]$ from levels $i+1$ to $i$\;
		\Repeat{all \textsc{ifmaps} are read into level $i$}{
			read next \textsc{ifmaps} block of size $[N_i, C_i, H^I_i, W^I_i]$ from levels $i+1$ to $i$\;
			let the data loaded to $i$ be processed\;
		}
	}
\end{algorithm}

\subsubsection{Forward}
In the forward pass, there are three types of input data reuse:
\begin{itemize}
	\item For an $H^I \times W^I$ \textsc{ifmap}, there are $H^O \times W^O$ convolutions performed with a single $H^F \times W^F$ filter to generate a partial sum. The filter is reused $H^O \times W^O$ times, and this type of reuse is defined as \emph{filter convolutional reuse}. Also, each feature in the \textsc{ifmaps} is reused $H^F \times W^F$ times, and this is called \emph{feature convolutional reuse}.
	\item Each \textsc{ifmap} is further reused across $M$ filters to generate $M$ output channels. This is called \emph{\textsc{ifmaps} reuse}.
	\item Each \textsc{filter} is further reused across the batch of $N$ \textsc{ifmaps}. This type of reuse is called \emph{filter reuse}.
\end{itemize}

From the obtained tiling parameters, the number of accesses that is used for \eref{eq:data_reuse_energy} and  \eref{eq:psum_energy} is determined by taking into account the data movement strategy as shown in \Cref{algo:DataMovement}. As a result, \Cref{tab:nb_reuse_forward} summarizes the number of accesses to memory levels for each data type in the forward pass. Therein, $\alpha^{v}=H^{O}/H^{I}$, $\alpha^{h}=W^{O}/W^{I}$, $H^{O}_\mathit{i}/W^{O}_\mathit{i}$ define the height/width of tiling \textsc{ofmaps} in L$\mathit{i}$ buffers, $\alpha^{v}_\mathit{i}=H^{O}_\mathit{i}/H^{I}_\mathit{i}$, and $\alpha^{h}_\mathit{i}=W^{O}_\mathit{i}/W^{I}_\mathit{i}$ for $i=2,1,$ and 0.

\begin{table}[H]
	\centering
	\caption{Numbers of accesses at different memory levels of forward convolution.}
	\label{tab:nb_reuse_forward}

	\vspace{1ex}

	\resizebox{\columnwidth}{!}{
	{
		{
			\begin{tabular}{ ccccc } 
				\toprule
				\textbf{Data} & \textbf{DRAM (L3)} & \textbf{L2} & \textbf{L1} & \textbf{L0} \\
				\midrule\midrule
				$I$ ($n_i^I$) & $\big \lceil \frac{M}{M_2} \big \rceil \times \frac{\alpha^{v}}{\alpha^{v}_2} \times  \frac{\alpha^{h}}{\alpha^{h}_2}  $ & $\big \lceil \frac{M_2}{M_1} \big \rceil \times  \frac{\alpha^{v}_2}{\alpha^{v}_1} \ \times  \frac{\alpha^{h}_2}{\alpha^{h}_1}  $ & $\big \lceil \frac{M_1}{M_0} \big \rceil \times \frac{\alpha^{v}_1}{\alpha^{v}_0}  \times  \frac{\alpha^{h}_1}{\alpha^{h}_0} $ & $H^{F} \times W^{F} \times \alpha^{v}_0 \times \alpha^{h}_0$ \\
				\midrule
				$F$ ($n_i^F$) & 1 & $\big \lceil \frac{N}{N_2} \big \rceil \times \big \lceil \frac{H^{O}}{H^{O}_2} \big \rceil \times \big \lceil \frac{W^{O}}{W^{O}_2} \big \rceil$ & $\big \lceil \frac{N_2}{N_1} \big \rceil \times \big \lceil \frac{H^{O}_2}{H^{O}_1} \big \rceil \times \big \lceil \frac{W^{O}_2}{W^{O}_1} \big \rceil$ & $\big \lceil \frac{N_1}{N_0} \big \rceil \times \big \lceil \frac{H^{O}_1}{H^{O}_0} \big \rceil \times \big \lceil \frac{W^{O}_1}{W^{O}_0} \big \rceil$\\
				\midrule
				$O$ ($n_i^O$) & 1 & 1 & 1 & 1\\
				\bottomrule
	\end{tabular}}}
	}
\end{table}

\subsubsection{Backward}
For the backward pass, given that $\partial\Loss/\partial O$ is backpropagated from the downstream, it consists in computing  $\partial\Loss/\partial F$ and $\partial\Loss/\partial I$. Following the derivation of backpropagation in CNNs by \cite{zhang2016derivation}, it is given that:
\begin{align}
	\partial\Loss/\partial F & = \textrm{Conv}(I, \partial\Loss/\partial O), \label{eq:local_gradient_filter} \\
	\partial\Loss/\partial I & = \textrm{Conv}(\textrm{Rot}_\pi(F), \partial\Loss/\partial O), \label{eq:local_gradient_input}
\end{align}
where $\textrm{Rot}_\pi(F)$ is the filter rotated by $180$-degree. As a result, the backward computation structure is also convolution operations, hence follows the same process as detailed above for the forward pass. For instance, \Cref{tab:nb_reuse_backward} summarizes the number of accesses at each memory level in the backward pass when calculating the gradient $G^I = \partial\Loss/\partial I$. Therein, $C_\mathit{i}$ defines the number of tiling \textsc{ifmaps} in L$\mathit{i}$ buffer, $\beta^{v}=H^{I}/H^{O}$, $\beta^{h}=W^{I}/W^{O}$, $\beta^{v}_\mathit{i}=H^{I}_\mathit{i}/H^{O}_\mathit{i}$, and $\beta^{h}_\mathit{i}=W^{I}_\mathit{i}/W^{O}_\mathit{i}$ for $i=2,1,$ and 0.

\begin{table}[H]
	\centering
    \caption{Numbers of accesses at different memory levels for $\partial\Loss/\partial I$.}
    \label{tab:nb_reuse_backward}

    \vspace{1ex}

    \resizebox{\columnwidth}{!}{
	{
            \begin{tabular}{ ccccc } 
				\toprule
				\textbf{Data} & \textbf{DRAM (L3)} & \textbf{L2} & \textbf{L1} & \textbf{L0} \\
				\midrule\midrule
				$O$ ($n_i^O$) & $\big \lceil \frac{C}{C_2} \big \rceil \times \frac{\beta^{v}}{\beta^{v}_2} \times  \frac{\beta^{h}}{\beta^{h}_2}  $ & $\big \lceil \frac{C_2}{C_1} \big \rceil \times  \frac{\beta^{v}_2}{\beta^{v}_1} \ \times  \frac{\beta^{h}_2}{\beta^{h}_1}  $ & $\big \lceil \frac{C_1}{C_0} \big \rceil \times \frac{\beta^{v}_1}{\beta^{v}_0}  \times  \frac{\beta^{h}_1}{\beta^{h}_0} $ & $H^{F} \times W^{F} \times \beta^{v}_0 \times \beta^{h}_0$ \\
				\hline
				$F$ ($n_i^F$) & 1 & $\big \lceil \frac{N}{N_2} \big \rceil \times \big \lceil \frac{H^{I}}{H^{I}_2} \big \rceil \times \big \lceil \frac{W^{I}}{W^{I}_2} \big \rceil$ & $\big \lceil \frac{N_2}{N_1} \big \rceil \times \big \lceil \frac{H^{I}_2}{H^{I}_1} \big \rceil \times \big \lceil \frac{W^{I}_2}{W^{I}_1} \big \rceil$ & $\big \lceil \frac{N_1}{N_0} \big \rceil \times \big \lceil \frac{H^{I}_1}{H^{I}_0} \big \rceil \times \big \lceil \frac{W^{I}_1}{W^{I}_0} \big \rceil$\\
				\midrule
				$G^I$ ($n_i^{G^I}$) & 1 & 1 & 1 & 1\\
				\bottomrule
	\end{tabular}}}
\end{table}

\section{Broader Impacts}
\label{appendix:broader_impacts}

Our multidomain comprehensive examination confirms that it is possible to create high performing binary deep neural networks thanks to the proposed method. 
Our findings suggest the positive impacts in many domains, making deep learning more environmentally friendly, in particular reduce the complexity of huge models like \glspl{LLM}, and enabling new applications like online, incremental, on-device training, and user-centric AI models.
Given the prevalence of \glspl{LLM}, our approach can facilitate faster predictions on more affordable devices, contributing to the democratization of deep learning. 
On the other hand, computing architectures have been so far pushed far from its native logic by high-precision arithmetic applications, e.g., 16-bit floating-point is currently a most popular AI computing architecture. Boolean logic deep learning would motivate new software optimization and hardware accelerator architectures in the direction of bringing them back to the native Boolean logic computing. 
The proposed mathematical notion and its calculus could also benefit other fields such as circuit theory, binary optimization, etc.

\end{document}